\documentclass{article}

\usepackage{iclr2027_preprint,times}

\usepackage{hyperref}
\usepackage{url}
\usepackage{graphicx}
\usepackage{xcolor}
\usepackage{amsmath,amssymb,amsthm}
\usepackage{mathtools}
\usepackage{booktabs}
\usepackage{subcaption}
\usepackage{microtype}
\usepackage{bm}
\usepackage{dsfont}
\usepackage{float}
\usepackage{accents}
\usepackage{standalone}
\usepackage{tikz}
\usepackage{pgfplots}
\usepackage{standalone}
\usepackage{siunitx}
\usepackage[frozencache]{minted}
\usepackage[capitalize]{cleveref}
\usepackage{csquotes}
\usepackage{IEEEtrantools}
\usepackage{algorithm}
\usepackage{pythonhighlight}
\usepackage[noend]{algpseudocode}

\theoremstyle{plain}
\newtheorem{theorem}{Theorem}
\newtheorem{lemma}{Lemma}
\newtheorem{corollary}{Corollary}
\newtheorem{proposition}{Proposition}
\theoremstyle{remark}
\newtheorem*{remark}{Remark}
\theoremstyle{definition}
\newtheorem{definition}{Definition}

\pgfplotsset{compat=newest}
\usepgfplotslibrary{fillbetween,groupplots}
\usetikzlibrary{positioning,calc,arrows.meta,fit}

\title{Scalable Discrete-to-Continuous Channel Simulation for Compression and Privacy}

\author{Joseph Rowan, Buu Phan and Ashish J.\ Khisti \\
Department of Electrical and Computer Engineering, University of Toronto\\
\small\texttt{\{joseph.rowan,truong.phan\}@mail.utoronto.ca, akhisti@ece.utoronto.ca}
}

\newcommand{\<}{\negmedspace{}}
\algnewcommand{\LComment}[1]{\State \textcolor{blue}{\(\triangleright\) \textit{{#1}}}}
\renewcommand{\algorithmicrequire}{\textbf{Input:}}
\renewcommand{\algorithmicensure}{\textbf{Output:}}
\newcommand{\todo}[1]{}
\renewcommand{\todo}[1]{{\color{red} TODO: {#1}}}

\newcommand{\mc}{\mathcal}

\renewcommand{\Pr}{\operatorname{Pr}}
\DeclareMathOperator{\E}{\mathbb{E}}
\DeclareMathOperator{\Var}{Var}
\DeclareMathOperator{\1}{\mathds{1}}

\DeclarePairedDelimiter{\abs}{\lvert}{\rvert}
\DeclarePairedDelimiter{\norm}{\lVert}{\rVert}

\DeclareMathOperator*{\argmin}{arg\,min}

\definecolor{c0}{HTML}{4C78A8}
\definecolor{c1}{HTML}{8F63B8}
\definecolor{c2}{HTML}{59A14F}
\definecolor{c3}{HTML}{E15759}
\definecolor{c4}{HTML}{F28E2B}
\definecolor{c5}{HTML}{653700}

\begin{document}

\maketitle

\begin{abstract}
Channel simulation has recently emerged as a useful component in machine learning systems where samples from a prescribed probability distribution are to be compressed.
Yet, general channel simulation algorithms often suffer from high computational costs, random stopping times or, in the worst case, can require generating an infinite number of shared random samples.
We introduce a scheme for both exact and approximate simulation of discrete-to-continuous channels which conversely uses a fixed number of random samples, and therefore has a runtime independent of the channel and the input.
Unlike existing channel simulation schemes which generate a sequence of independent samples from a proposal distribution, our approach generates one sample, or alternatively a fixed number of samples, from each potential target distribution. We then apply a latent permutation to the samples before performing sample selection using an exponential race.
Our scheme provides a flexible tradeoff between the number of  generated samples and the compression rate.
Using polar and multilevel coding, we scale our approach to handle long blocklengths in $O(n \log n)$ time in order to benefit from reduced per-symbol overhead.
We conclude by demonstrating applications to variable-rate compression with stochastic VQ-VAEs and communication-efficient differentially private distributed mean estimation via exact simulation of the Gaussian mechanism.
\end{abstract}

\section{Introduction}\label{sec:introduction}
In channel simulation, one party, termed the \emph{encoder}, observes a source $X \in \mc{X}$, $X \sim P_X$, and sends a message $M$ to another party, the \emph{decoder}, which uses the message to output a random variable $Y$.
Here, $Y$ must follow a prescribed conditional distribution $P_{Y \mid X}(\cdot \mid X)$, called the \emph{channel}, while the expected bit-length of $M$, $\E[\ell(M)]$, should be made as small as possible.
To assist in reducing the communication cost, we assume that shared randomness $W$ is available at both the encoder and decoder.
Viewed in this way, channel simulation provides a general method to compress a noisy version of $X$ into a binary representation, where the amount and type of noise added to $X$ is controlled by the choice of $P_{Y \mid X}$.
It may also be regarded as a stochastic generalization of lossy compression methods such as quantization, which similarly perturb a source and produce discrete representations.
As such, it has emerged as a key building block in neural compression architectures, where the deterministic nature of quantization can pose difficulties during training \citep{flamich2026}, and has also been applied to privacy \citep{liu2024}.

The nature of $X$ and $Y$ has a direct bearing on the difficulty of the channel simulation task.
While numerous constructions have been proposed for arbitrary probability spaces, notably the Poisson functional representation (PFR, \citealp{li2018}), rejection-sampling based schemes \citep{flamich2023c,phan2025} and approximate methods leveraging importance sampling \citep{havasi2019,phan2024}, these may either be slow without additional restrictions on $P_{Y \mid X}$, suffer from highly variable running times, or produce biased outputs at smaller sample counts.
Most critically, known exact general methods require generating infinitely many samples in the worst case, making them less attractive for real-world applications.
In contrast, when $Y$ is discrete, the exponential functional representation (EFR, \citealp{li2018}) can be used to simulate any channel exactly in $O(\abs{\mc{Y}})$ time with the shared randomness consisting of $\abs{\mc{Y}}$ exponential random variables.
We are interested in the complementary case where $\mc{X}$ is discrete but $\mc{Y}$ is continuous, which naturally arises in situations where continuous noise perturbs discrete input data, and ask whether this second problem may also be solved with a \emph{fixed number} of shared randomness samples.
Our first goal in this paper is to answer in the affirmative by introducing an approach we term the \emph{permuted scheme}, which uses $2 \abs{\mc{X}}$ random samples to achieve communication-efficient simulation of any discrete-input, continuous-output channel.

In machine learning applications, a further challenge is that the channels to be simulated are often high-dimensional, involving for example vectors of latents, gradients or network weights.
On one hand, increasing the channel dimension, also known as the blocklength, has the benefit of amortizing communication overheads.
For example, PFR achieves an expected message length at most $I(X; Y) + \log(I(X;Y) + 2) + 3$ bits \citep{li2024a}; extending to $n$ independent uses of the channel yields $I(X; Y) + \log(n I(X;Y) + 2) / n + 3 / n$ bits per sample, which approaches the lower bound of $I(X;Y)$ in the limit as $n \to \infty$.
Unfortunately, the time complexity of PFR and similar general channel simulation schemes increases exponentially with $n$, rendering their use on high-dimensional channels impractical.
Recent work has turned to techniques from modern coding theory to accelerate channel simulation across many independent dimensions \citep{sriramu2024,zhao2026,ozyilkan2026}, yet their applicability remains limited to binary or other discrete outputs.
Our second objective is therefore to demonstrate how a polar coding framework can similarly be incorporated into the permuted scheme, allowing for scalable discrete-to-continuous channel simulation at long blocklengths.
In summary, our contributions are:
\begin{enumerate}
    \item We introduce a method for the simulation of general discrete-input, continuous-output channels which uses a fixed number of shared randomness samples, with the runtime being independent of the input and choice of $P_{Y \mid X}$.
    We further show how an approximation based on the Sinkhorn-Knopp algorithm \citep{sinkhorn1964,knight2008} can be used to scale our approach to large input alphabets in applications such as neural compression where an exact output distribution may not be necessary. 
    \item By adopting a construction based on polar codes \citep{arikan2009} and multilevel coding (MLC, \citealp{wachsmann1999}), we scale the dimension of channels which can be simulated with the permuted scheme.
    The complexity of our algorithm scales with the channel dimension as $O(n \log n)$, making it practical even when $n$ is large and allowing operation at long blocklengths where per-symbol overhead is reduced.
    \item We deploy our scheme to two applications: first, we demonstrate how our discrete-to-continuous channel simulation can be used to achieve variable-rate lossy compression from VQ-VAEs \citep{vandenoord2017} without retraining; second, we present a scheme for communication-efficient distributed mean estimation (DME) with central differential privacy (CDP) via exact simulation of the Gaussian mechanism \citep{dwork2014}.
\end{enumerate}

Throughout, logarithms are base 2 unless otherwise stated.
Entropies, mutual information and KL divergence are to the same base.
We write $x^n$ for a sequence $\{ x_i \}_{i=1}^{n}$ and let $[n] = \{ 1, \ldots, n \}$. 


\section{Background and Related Work}\label{sec:related_work}
\paragraph{One-Shot Channel Simulation.}
In the \emph{one-shot} setting, the encoder observes a single symbol $X$ and the decoder outputs a corresponding $Y$.
Since $X$ and $Y$ are not required to be scalars, but may instead be random sequences, vectors, or any other object, the one-shot setting is the most general.
\citet{bennett2002} and \citet{winter2002} established theoretical performance bounds for channel simulation under unlimited common randomness, with the lower bound being $I(X;Y)$. 
\citet{harsha2010} introduced an early simulation algorithm for discrete channels known as greedy rejection sampling (GRS) which achieves a near-optimal coding cost, which was later extended to general probability spaces \citep{flamich2023c,flamich2023b}.
Adopting a shared Poisson point process as the common randomness, \citet{li2018} introduced PFR, a popular exact channel simulation algorithm whose rate also attains the mutual information with logarithmic redundancy.
More recently, schemes based on standard rejection sampling have also emerged as alternatives with similar theoretical guarantees \citep{phan2025,hill2026}.

\paragraph{Multi-Shot Channel Simulation.}
By the multi-shot setting, we refer to the task of simulating $n$ iid copies of a channel, i.e.\ the product channel $P_{X \mid Y}^{\times n}$.
In principle, any one-shot simulation algorithm can be scaled to the multi-shot case by treating the input $X^n$ as one high-dimensional symbol.
However, while several algorithms have been proposed which improve the performance of one-shot simulation in special cases \citep{flamich2022,flamich2023a,hegazy2022}, existing general sampling-based schemes do not scale well to large blocklengths as their sample complexity is exponential in the joint mutual information \citep{li2024a}.
By focusing on binary-output channels, \citet{sriramu2024} devised a multi-shot algorithm based on polar coding \citep{arikan2009} which runs in $O(n \log n)$ time and can hence take advantage of decreased per-symbol overhead at long blocklengths.
Subsequently, \citet{zhao2026} introduced a family of algorithms based on polar and other linear codes extending multi-shot simulation to additive exchangeable noise channels over finite fields, while \citet{ozyilkan2026} generalized to a non-stationary setting.
However, these approaches are currently limited to discrete or, in the case of long-blocklength polar code constructions, binary outputs, and cannot be used when a richer output space is desired.

\paragraph{Applications to Neural Compression and Privacy.}
The connection between channel simulation and lossy compression was established in \citet{winter2002}; any channel simulation algorithm can be converted into a lossy compressor \citep{li2018}.
On the practical side, \citet{havasi2019} used approximate simulation to compress network parameters, while \citet{flamich2020} applied a similar approach to images.
\citet{theis2022a} defined a channel from a diffusion model and simulated it with PFR to achieve image compression under a realism constraint, which \citet{vonderfecht2025} accelerated to practical speeds using an approximate construction.
\citet{phan2024} adopted channel simulation as the basis for image compression with side information at the decoder.
Continuous latent spaces derived from $\beta$-VAEs or Gaussian diffusion have typically been used in these approaches, which consequently suffer from the high computational cost of simulating general channels and are practically limited to short blocklengths.
By instead using our discrete-to-continuous scheme and the discrete latent space of a VQ-VAE, we retain the flexibility of channel simulation-based compression while reducing coding overhead by jointly processing blocks of up to $8192$ latent variables.

Channel simulation has also found applications in differential privacy (DP, \citealp{dwork2006,dwork2014}), which offers a principled way to protect users' data released to an untrusted party.
A DP mechanism applies a randomized privacy-preserving perturbation to the data $X$, canonical examples being additive  Gaussian or Laplacian noise, and can therefore be construed as a channel \citep{li2024a}.
\cite{feldman2021} used approximate simulation via rejection sampling and pseudorandomness to reduce the communication cost of locally differentially private (LDP) mechanisms.
\citet{shah2022} developed an approach based on importance sampling (IS).
However, these methods are approximate in the sense that they do not preserve the target noise distribution.
Under a trusted aggregator, \citet{hasircioglu2024} and \citet{hegazy2024} used dithered quantization to achieve exact realizations of 1D additive noise mechanisms.
\citet{liu2024} proposed a general method extending PFR to exactly simulate any LDP or CDP mechanism, at the expense of inheriting PFR's cost and non-deterministic stopping time.
We offer a complementary approach that can be used to exactly and quickly simulate CDP mechanisms acting on discrete inputs up to 16-ary alphabets.

\section{One-Shot Discrete-to-Continuous Channel Simulation}\label{sec:one_shot_simulation}
We are interested in simulating a channel $P_{Y \mid X}$ from $X \in \mc{X}$ to $Y \in \mc{Y}$, where $\mc{X}$ is a finite set, while requiring that the shared randomness contain a fixed number of samples.
Suppose without loss of generality that $\mc{X} = [N]$; we assume for simplicity that for each realization $X = x$, the density $p_{Y \mid X}(\cdot \mid x)$ exists.
A naive scheme that satisfies our requirements consists of the following:

\begin{enumerate}
    \item Generate shared randomness $\bar{U}^N$, where $\bar{U}_i \sim P_{Y \mid X}(\cdot \mid i)$.
    \item At the encoder, observe $X = x$ and select the index $K = x$.
    \item Send $K$ to the decoder, which will output $Y = \bar{U}_K$.
\end{enumerate}

However, since this scheme directly transmits $X$ to the decoder independently of the common randomness, the communication cost is at least $H(X)$, which may in general be much larger than $I(X;Y)$.
An intuitive idea to reduce the amount of communication might be to introduce randomness into the index $K$ while retaining the same sample pool $\bar{U}^N$.
Unfortunately, it turns out that this approach does not allow an arbitrary target distribution to be reproduced for any rate less than $H(X)$.
To see this, suppose $N = 2$ and let the index follow a conditional distribution $P_{K \mid X, \bar{U}^2}(\cdot \mid x, u^2)$.
One can then show, see \cref{sec:biawgn_example}, that the only solution preserving the output distribution is $P_{K \mid X, \bar{U}^2}(k \mid x, u^2) = \delta(k - x)$, bringing us back to the naive scheme; the same argument holds for larger values of $N$.
To reduce the rate below $H(X)$ while achieving our goal of simulating $P_{Y \mid X}$ exactly with a fixed sample count, a successful strategy must also apply additional randomization to $\bar{U}^N$.
This forms the basis of our permuted scheme, which we will consider in the next section.

\subsection{Permuted Scheme for Channel Simulation with Finitely Many Samples}\label{sec:oneshot_permuted}

To overcome the naive scheme's failure to drive the communication cost closer to $I(X;Y)$, our key insight is to break the deterministic mapping between the index of each sample $\bar{U}_i$ and the distribution which generated it by introducing a random permutation.
Importantly, the permutation is ignored by both the encoder and decoder after being used to shuffle the indices, the samples being treated thereafter as marginally identically distributed.
The key steps are as below:

\begin{enumerate}
    \item Generate a uniform random permutation $\Pi$ of $[N]$ and, given $\Pi = \pi$, generate shared randomness $\bar{U}^N$, where $\bar{U}_{\pi(i)} \sim P_{Y \mid X}(\cdot \mid i)$.
    \item At the encoder, observe $X = x$ and select $K$ according to the target distribution $P_{K \mid X, \bar{U}^N}(k \mid x, u^N) = \Pr[\Pi(x) = k \mid \bar{U}^N = u^N]$.
    Note that $K$ is conditionally independent of $\Pi$ given $\bar{U}^N$ by construction.
    \item Losslessly encode $K$ into the message $M$ and send it to the decoder to output $Y = \bar{U}_K$.
\end{enumerate}

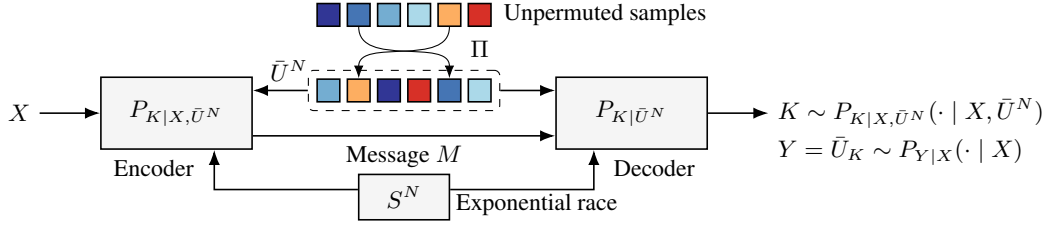
\begin{figure}[tb]
    \centering
    \begin{tikzpicture}[
    >=Latex,
    font=\footnotesize,
    box/.style={
        draw,
        semithick,
        minimum width=2cm,
        minimum height=0.95cm,
        fill=gray!8,
        align=center
    },
    smallbox/.style={
        draw,
        semithick,
        minimum width=1.2cm,
        minimum height=0.6cm,
        fill=gray!8,
        align=center
    },
    sample/.style={
        draw,
        semithick,
        minimum width=3mm,
        minimum height=3mm,
        inner sep=0pt
    }
]

\definecolor{s0}{RGB}{49,54,149}
\definecolor{s1}{RGB}{69,117,180}
\definecolor{s2}{RGB}{116,173,209}
\definecolor{s3}{RGB}{171,217,233}
\definecolor{s4}{RGB}{253,174,97}
\definecolor{s5}{RGB}{215,48,39}

\node[box] (enc) {$P_{K \mid X,\bar{U}^N}$};
\node[box, right=40mm of enc] (dec) {$P_{K \mid \bar{U}^N}$};
\coordinate (mid) at ($(enc)!0.5!(dec)$);

\node[left=8mm of enc] (source) {$X$};
\draw[->, semithick] (source) -- (enc);

\draw[->, semithick] ([yshift=-3mm]enc.east) -- node[below] {Message $M$} ([yshift=-3mm]dec.west);

\node[right=8mm of dec] (output) {$K \sim P_{K \mid X,\bar{U}^N}(\cdot \mid X,\bar{U}^N)$};
\node[below=0.5cm of output.west, anchor=west] {$Y = \bar{U}_K \sim P_{Y \mid X}(\cdot \mid X)$};
\draw[->, semithick] (dec) -- (output.west);

\node[below=0.0mm of enc, xshift=-0.3cm] {Encoder};
\node[below=0.0mm of dec, xshift=0.3cm] {Decoder};

\node[smallbox, below=8mm of mid] (exp) {$S^N$};

\draw[->, semithick] ([yshift=1mm]exp.west) -| ([xshift=0.5cm]enc.south);
\draw[->, semithick] ([yshift=1mm]exp.east) node[right, yshift=-2mm, xshift=-0.5mm] {Exponential race} -| ([xshift=-0.5cm]dec.south);

\begin{scope}[shift={([yshift=3mm]mid.center)}]
    \node[sample, fill=s0] (t1) at (-10mm, 10mm) {};
    \node[sample, fill=s1] (t2) at (-6mm, 10mm) {};
    \node[sample, fill=s2] (t3) at (-2mm, 10mm) {};
    \node[sample, fill=s3] (t4) at (2mm, 10mm) {};
    \node[sample, fill=s4] (t5) at (6mm, 10mm) {};
    \node[sample, fill=s5] (t6) at (10mm, 10mm) {};
    
    \node[sample, fill=s2] (b1) at (-10mm, 0mm) {};
    \node[sample, fill=s4] (b2) at (-6mm, 0mm) {};
    \node[sample, fill=s0] (b3) at (-2mm, 0mm) {};
    \node[sample, fill=s5] (b4) at (2mm, 0mm) {};
    \node[sample, fill=s1] (b5) at (6mm, 0mm) {};
    \node[sample, fill=s3] (b6) at (10mm, 0mm) {};

    \draw[->] (t2) to[out=-90,in=90] node [xshift=1cm, yshift=0.5mm] {$\Pi$}(b5);
    \draw[->] (t5) to[out=-90,in=90] (b2);

    \node[
        draw,
        dashed,
        rounded corners=2pt,
        inner sep=3pt,
        fit=(b1)(b6)
    ] (pool) {};
\end{scope}

\draw[->, semithick] (pool.west) node[above, xshift=-2.5mm] (samples label) {$\bar{U}^N$} -- (enc.east |- pool.west);
\draw[->, semithick] (pool.east) -- (dec.west |- pool.east);
\node[right=1pt of t6] {Unpermuted samples};

\end{tikzpicture}
    \caption{Permuted scheme for one-shot channel simulation. After being used to scramble the random samples, the permutation $\Pi$ does not play any further role in the encoder's target distribution $P_{K \mid X, \bar{U}^N}(\cdot \mid X, \bar{U}^N)$ or the decoder's prior $P_{K \mid \bar{U}^N}(\cdot \mid \bar{U}^N)$.}
    \label{fig:permuted_diagram}
    \vspace{-1em}
\end{figure}

The problem is now reduced to communicating a sample from $P_{K \mid X, \bar{U}^N}(\cdot \mid x, u^N)$.
We highlight that unlike most prior channel simulation schemes which generate a shared sequence of independent samples and afterwards select an output from among them, our method instead generates one sample from each potential target distribution.
Crucially, we then permute the samples before applying sample selection to reduce the communication rate.
This latent permutation increases randomness in the output given $X$ and $\bar{U}^N$ to facilitate compression.
A concrete example on a BI-AWGN channel, motivating the need for the random permutation, is provided in \cref{sec:biawgn_example}.
Next, and before proceeding to the rate analysis, we provide a guarantee that the permuted scheme performs exact channel simulation.
The proof is found in \cref{sec:correctness_proof}.

\begin{proposition}\label{prop:correctness}
The permuted scheme exactly simulates $P_{Y \mid X}$.
\end{proposition}

To efficiently send the index $K \sim P_{K \mid X, \bar{U}^N}(\cdot \mid x, u^N)$ to the decoder and establish the communication cost, we can use the well-established strategy of coordinated sampling with a shared exponential race, also known as EFR \citep{li2018}.
In particular, a sequence $S^N$ of iid $\operatorname{Exp}(1)$ random variables is generated and shared between the encoder and the decoder.
After sorting $S^N$ to take into account the decoder's knowledge of the prior $P_{K \mid \bar{U}^N}(\cdot \mid u^N)$, $K$ is chosen to be the winning index of the exponential race, and its position $\bar{K}$ in the \emph{sorted} list is communicated using a suitable lossless entropy code.
This approach is detailed in \cref{alg:permuted_scheme} and visualized in \cref{fig:permuted_diagram}.
The upper bound on the expected length achieved by EFR \citep[Theorem~14]{li2024a} then applies directly to the permuted scheme, giving us the following one-shot achievability result.
A proof is given in \cref{sec:upper_bound_proof}.

\begin{theorem}\label{thm:cost_upper_bound}
The index transmitted to the decoder as part of the permuted scheme satisfies
\begin{IEEEeqnarray}{c}
    \E[\log \bar{K}] \leq I(X;Y \mid \bar{U}^N) + 1
\end{IEEEeqnarray}
and achieves an expected message length in bits of
\begin{IEEEeqnarray}{c}
    \E[\ell(M)] \leq I(X;Y \mid \bar{U}^N) + \log(I(X;Y \mid \bar{U}^N) + 2) + 3. \label{eqn:coding_length_bound}
\end{IEEEeqnarray}
\end{theorem}

In \cref{sec:lower_bound}, we provide context on the differences between the upper bound in \cref{thm:cost_upper_bound} and similar achievability results for other channel simulation algorithms such as PFR and GRS, especially the role of the conditional mutual information $I(X;Y \mid \bar{U}^N)$, and prove a related lower bound.

\subsection{Computing the Posterior Distribution}\label{sec:permuted_scheme_impl}

To run \cref{alg:permuted_scheme} in practice, the encoder needs to compute the posterior $P_{K \mid X, \bar{U}^N}(k \mid x, u^N)$ for all $k \in [N]$, while the decoder needs to calculate $P_{K \mid \bar{U}^N}(k \mid u^N)$.
From Bayes' rule, we have
\begin{IEEEeqnarray}{c}
    P_{K \mid X, \bar{U}^N}(k \mid x, u^N) = \frac{p_{\bar{U}^N}(u^N \mid \Pi(x) = k)}{N p_{\bar{U}^N}(u^N)}. \label{eqn:posterior_distribution}
\end{IEEEeqnarray}
The prior $P_{K \mid \bar{U}^N}(k \mid u^N)$ can in turn be found from \eqref{eqn:posterior_distribution}, and evaluates to a constant if $X$ is uniform.
However, computing the terms on the right-hand side is difficult due to the introduction of the latent permutation $\Pi$ and the dependence between the shared random samples, distinct from existing channel simulation schemes where these are independent.
To derive an exact expression for the posterior, we introduce the $N \times N$ matrix of likelihoods $L(u^N)$, where $L_{ij} = p_{Y \mid X}(u_i \mid j)$.
Then,
\begin{IEEEeqnarray}{c}
    P_{K \mid X, \bar{U}^N}(k \mid x, u^N) = \frac{p_{Y \mid X}(u_k \mid x) \operatorname{perm}(L(u^N)_{k,x})}{\operatorname{perm}(L(u^N))} \label{eqn:exact_posterior}
\end{IEEEeqnarray}
where $\operatorname{perm}$ denotes the matrix permanent and $L(u^N)_{k,x}$ is formed by removing row $k$ and column $x$ from $L(u^N)$.
While computing the permanent is \#P-complete, it remains fast for small $N$ up to around 16.
This is the method we adopt for the CDP experiment in \cref{sec:experiment_cdp}, which requires exact simulation to maintain privacy.
For large $N$, as is the case for our VQ-VAE image compression experiment in \cref{sec:experiment_vqvae} which uses $N=256$, we turn to an approximation based on applying Sinkhorn-Knopp iterations to the likelihood matrix, which has previously been adopted in machine learning for approximating posteriors over latent permutations \citep{mena2018,mena2020}.
Although this sacrifices the exactness of the simulation, we find it sufficient to achieve competitive end-to-end rate-distortion performance with our stochastic VQ-VAE framework.
Additional details concerning the approximation and its empirical performance are provided in \cref{sec:approx_posterior}.
We emphasize that while large \emph{alphabets} require an approximation in our scheme, exact simulation remains practical for large \emph{blocklengths} if $N$ is small enough, as we will demonstrate in \cref{sec:experiment_cdp}. 

\begin{algorithm}[bt]
\caption{One-Shot Discrete to Continuous Channel Simulation}\label{alg:permuted_scheme}
\begin{algorithmic}\small
\Function{Encode}{$x, \mathfrak{G}, \mathtt{enc}$}
    \State \textbf{Input:} Channel input $x$, PRNG $\mathfrak{G}$, lossless encoder $\mathtt{enc}$. \textbf{Output:} Message $M$.
    \State Generate $\Pi$, a uniform random permutation of $[N]$, using $\mathfrak{G}$.
    \State Given $\Pi = \pi$, sample $\bar{U}_{\pi(i)} \sim P_{Y \mid X}(\cdot \mid i)$ using $\mathfrak{G}$, along with $S^N \sim \operatorname{Exp}(1)$ iid.
    \State Given $\bar{U}^N = u^N$, $p_i \gets P_{K \mid X, \bar{U}^N}(i \mid x, u^N)$, $q_i \gets P_{K \mid \bar{U}^N}(i \mid u^N)$.
    \Comment See \eqref{eqn:exact_posterior}.
    \State Compute a permutation $\sigma$ such that $S_{\sigma(1)} / q_{\sigma(1)}  \leq S_{\sigma(2)} / q_{\sigma(2)} \leq \cdots \leq S_{\sigma(N)} / q_{\sigma(N)}$.
    \State $K \gets \argmin_{1 \leq i \leq N} S_i / p_i$, $\bar{K} \gets \sigma^{-1}(K)$, $M \gets \mathtt{enc}(\bar{K})$
    \State \Return $M$ 
\EndFunction
\vspace{1ex}
\Function{Decode}{$M, \mathfrak{G}, \mathtt{dec}$}
    \State \textbf{Input:} Message $M$, PRNG $\mathfrak{G}$, lossless decoder $\mathtt{dec}$. \textbf{Output:} Channel output $Y$.
    \State Generate $\Pi$, a uniform random permutation of $[N]$, using $\mathfrak{G}$.
    \State Given $\Pi = \pi$, sample $\bar{U}_{\pi(i)} \sim P_{Y \mid X}(\cdot \mid i)$ using $\mathfrak{G}$, along with $S^N \sim \operatorname{Exp}(1)$ iid.
    \State Given $\bar{U}^N = u^N$, $q_i \gets P_{K \mid \bar{U}^N}(i \mid u^N)$.
    \Comment See \eqref{eqn:exact_posterior}.
    \State Compute a permutation $\sigma$ such that $S_{\sigma(1)} / q_{\sigma(1)}  \leq S_{\sigma(2)} / q_{\sigma(2)} \leq \cdots \leq S_{\sigma(N)} / q_{\sigma(N)}$.
    \State $\bar{K} \gets \mathtt{dec}(M)$, $K \gets \sigma(\bar{K})$, $Y \gets \bar{U}_K$
    \State \Return $Y$
\EndFunction
\end{algorithmic}
\end{algorithm}

\subsection{Communication Cost with Multiple Samples per Mode}\label{sec:asymptotic_result}

A natural question following from the theoretical discussion in \cref{sec:oneshot_permuted} is whether the structure of the shared randomness in the permuted scheme can be modified to recover the mutual information $I(X;Y)$ in place of $I(X;Y \mid \bar{U}^N)$.
Since
\begin{IEEEeqnarray}{c}
    I(X;Y \mid \bar{U}^N) = I(X;Y) + I(X; \bar{U}^N \mid Y),
\end{IEEEeqnarray}
approaching $I(X;Y)$ implies reducing the dependence between $X$ and the set of rejected samples given $Y = \bar{U}_K$.
As this dependence arises due to the finite size of the sample pool, it is reasonable to expect that increasing the total number of samples beyond $N$ may serve to reduce the size of the penalty.
Following this idea, we now consider a generalized setting with $m > 1$ iid samples per mode; as $m$ becomes large, knowing $X$ and $Y$ reveals relatively less information about the remaining $Nm - 1$ samples.
The target distribution for the index $K_m \in [Nm]$ becomes
\begin{IEEEeqnarray}{c}
    P_{K_m \mid X, \bar{U}^{Nm}}(k \mid x, u^{Nm}) = \frac{1}{m} \Pr[k \in \Pi(x) \mid \bar{U}^{Nm} = u^{Nm}] \label{eqn:multisample_target_distribution}
\end{IEEEeqnarray}
where $\Pi$ is now a uniform random mapping from each mode $x$ to a set $\Pi(x)$ containing the $m$ unique indices assigned to it.
First, we show that this scheme also performs exact channel simulation, with the proof following in \cref{sec:multi_sample_correctness_proof}.

\begin{proposition}\label{prop:multi_sample_correctness}
The multi-sample permuted scheme exactly simulates $P_{Y \mid X}$.
\end{proposition}

If we again use EFR sampling to communicate $K_m$, \cref{thm:cost_upper_bound} also applies to the multi-sample extension and shows that the communication cost is determined by $I(X; Y \mid \bar{U}^{Nm})$.
Summarized below and in \cref{sec:asymptotic_proof}, we show that this approaches $I(X;Y)$ as $m$ increases.
For technical reasons we restrict our analysis to the binary-input case where $N=2$ and leave the general case to future work, though we expect the result to extend to larger $N$.
Note also that in the polar coding extension and experiments to follow in \cref{sec:polar,sec:experiments}, we focus on $m=1$ for simplicity, which we find already sufficient to obtain competitive compression rates.

\begin{theorem}\label{thm:asymptotic}
Let $\E[\ell_m(M)]$ be the expected message length when $m$ samples are taken per mode.
When $N=2$, the expected message length using the multi-sample permuted scheme satisfies
\begin{IEEEeqnarray}{c}
    \limsup_{m \to \infty} \E[\ell_m(M)] \leq I(X;Y) + \log(I(X;Y) + 2) + 3.
\end{IEEEeqnarray}
\end{theorem}

\section{Scaling the Permuted Scheme with Polar Coding}\label{sec:polar}
As mentioned in \cref{sec:introduction}, in certain applications such as the transmission of a $n$-dimensional vector of independent latent variables, it may be necessary to simulate multiple instances of one channel.
Similar to existing channel simulation protocols, the one-shot permuted scheme benefits from the resulting amortization of the excess coding cost over $n$ samples; when considering a source vector $X^n \sim P_X$ iid and $n$ independent copies of the channel $P_{Y \mid X}$, \eqref{eqn:coding_length_bound} gives
\begin{IEEEeqnarray}{c}
    \E[\ell(M)] / n \leq I(X; Y \mid \bar{U}^N) + \log(n I(X;Y \mid \bar{U}^N) + 2) / n + 3 / n
\end{IEEEeqnarray}
which approaches $I(X;Y \mid \bar{U}^N)$ as $n$ becomes large.
However, directly scaling the one-shot algorithm in this way quickly runs into computational bottlenecks due to its unstructured nature; if $X^n$ is a binary source, the size of the likelihood matrix is $2^n$, and even the Sinkhorn-Knopp approximation rapidly becomes infeasible as $n$ increases.
Therefore, the one-shot permuted scheme suffers from the rapid growth in complexity with the source dimension common to existing protocols like PFR and GRS.
This motivates us to adapt the PolarSim algorithm \citep{sriramu2024}, which natively operates on iid copies of binary-output channels in $O(n\log n)$ time, to devise an efficient multi-shot version of the permuted scheme.

Owing to PolarSim's binary-output requirement, which is also inherited by recent extensions \citep{ozyilkan2026}, it is not immediately obvious how to incorporate its framework into our discrete-to-continuous setting.
While we have shown in \cref{sec:oneshot_permuted} that the permuted scheme reduces to communicating the $N$-ary index $K$ after conditioning on the shared random samples, even if $K$ were binary, constructing a new polar code for each instance of the shared randomness would be impractical.
Moreover, the resulting channel instances would not be identically distributed.
In the remainder of this section, we show how to overcome these challenges.
First, the permuted scheme with $N = 2$ can be  reformulated as simulating a binary-symmetric channel (BSC) with side information, which can be solved efficiently by a simple generalization of PolarSim.
Second, we generalize to larger alphabets via multilevel coding (MLC, \citealp{wachsmann1999}).

\subsection{Applying Polar Coding for a Binary Alphabet}\label{sec:polar_coding_n_2}

To begin, we consider $N = 2$ as a special case, where we can evaluate \eqref{eqn:posterior_distribution} directly to get
\begin{IEEEeqnarray}{rCl}
    P_{K \mid X, \bar{U}^2}(1 \mid 1, u^2) & = & P_{K \mid X, \bar{U}^2}(2 \mid 2, u^2) = 1 - \delta(u^2) \\
    P_{K \mid X, \bar{U}^2}(1 \mid 2, u^2) & = & P_{K \mid X, \bar{U}^2}(2 \mid 1, u^2) = \delta(u^2)
\end{IEEEeqnarray}
where we have introduced the state-dependent BSC crossover probability
\begin{IEEEeqnarray}{c}
    \delta(u^2) = \frac{p_{Y \mid X}(u_1 \mid 2) p_{Y \mid X}(u_2 \mid 1)}{p_{Y \mid X}(u_1 \mid 1) p_{Y \mid X}(u_2 \mid 2) + p_{Y \mid X}(u_1 \mid 2) p_{Y \mid X}(u_2 \mid 1)}.
\end{IEEEeqnarray}
For $n$ independent channel uses, 
\begin{IEEEeqnarray}{c}
    P_{K^n \mid X^n, \bar{U}^{2n}}(k^n \mid x^n, u^{2n}) = \prod_{i=1}^{n} \operatorname{BSC}(k_i \mid x_i, \delta(u^2_i)) = \delta(k_i - x_i)(1 - 2 \delta(u_i^2)) + \delta(u_i^2) \IEEEeqnarraynumspace
\end{IEEEeqnarray}
which represents a nonstationary BSC where the crossover probability changes depending on the shared randomness realization at each timestep $i$.
While \citet{ozyilkan2026} recently proposed a generalization of PolarSim to simulate such a time-varying channel, their approach relies on interleaving custom permutations inside the polar transform, introducing additional overhead and requiring a more complicated theoretical analysis.

Instead, we propose viewing the channel as a stationary channel with side information, which is possible due to the specific structure of the shared randomness in the permuted scheme.
Let $w_i = u_i^2$ be the side information at timestep $i$, which is known to both the encoder and the decoder.
We have that $(x_i, w_i)$ are iid draws from the joint distribution $P_X \times P_{\bar{U}^2}$.
Then, following the discussion in \cref{sec:polar_sim}, and in particular \cref{thm:polar_sim}, the modified PolarSim framework which we present there in \cref{alg:polar_sim} can be used to simulate the channel with a communication cost approaching $I(X; Y \mid \bar{U}^2)$ bits per channel use as $n$ grows.

\subsection{Increasing the Alphabet Size via Multilevel Coding}\label{sec:polar_mlc}

For real-world applications like those we will pursue experimentally in \cref{sec:experiments}, it is important to be able to use alphabets larger than $N = 2$.
Fortunately, the binary scheme provides a foundation for efficient simulation at long blocklengths when $N$ is increased, so long as the posterior distribution in \eqref{eqn:posterior_distribution} or its approximation remains tractable.
We base our method on MLC \citep{wachsmann1999}, which is a classical approach for constructing capacity-achieving codes for $2^m$-ary transmission, $m \geq 1$, from existing binary codes.
Beginning with the channel $P_{K \mid X, \bar{U}^N}$, we start by representing $K$ as a binary vector of length $m = \log N$, written as the sequence $B^m$.
Then, $K = \sum_{l=1}^{m} 2^{l-1} B_l$.
Applying the chain rule, $P_{K \mid X, \bar{U}^N}(k \mid x, u^N) = \prod_{l=1}^{m} P_{B_l \mid B^{l-1}, X, \bar{U}^N}(b_l \mid b^{l-1}, x, u^N)$, which decomposes the task of sampling $K$ into $m$ sequential binary sampling steps at each level $1 \leq l \leq m$.

To find the bit-wise conditional distributions, we take the posterior distribution in \eqref{eqn:posterior_distribution} and let $\delta_{jk}(u^N) = P_{K \mid X, \bar{U}^N}(k \mid j, u^N)$ be the crossover probability from $X = j$ to $K = k$.
Next, we divide the set of binary numbers $A = \{ 0, 1 \}^m$ into subsets $A(b^l) = \{ a^m \in A : a^l = b^l \}$.
Then,
\begin{IEEEeqnarray}{c}
    P_{B_l \mid B^{l-1}, X, \bar{U}^N}(b_l \mid b^{l-1}, x, u^N) = \frac{\sum_{k \in A(b^l)} \delta_{x(k+1)}(u^N)}{\sum_{k \in A(b^{l-1})} \delta_{x(k+1)}(u^N)}.
\end{IEEEeqnarray}
We now have a binary-output channel at each level, which can be efficiently simulated for $n$ independent uses using the polar coding method introduced in \cref{sec:polar_coding_n_2}.
For level $l$, the side information at timestep $i$ is $w_i = (u_i^N, b^{l-1}_i)$, and $(x_i, w_i)$ are iid draws from $(P_X \times P_{\bar{U}^N}) P_{B^{l-1} \mid X, \bar{U}^N}$.
Again referring to \cref{sec:polar_sim} and \cref{thm:polar_sim}, for each level and as the blocklength grows, the extended PolarSim scheme in \cref{alg:polar_sim} achieves
\begin{IEEEeqnarray}{c}
    \lim_{n \to \infty} \E[\ell(M_l)] / n = I(X; B_l \mid B^{l-1}, \bar{U}^N)
\end{IEEEeqnarray}
where $M_l$ is the transmitted message at level $l$.
The total cost to communicate $K$ in the limit is then
\begin{IEEEeqnarray}{c}
    \lim_{n \to \infty} \frac{1}{n} \sum_{l=1}^{m} \E[\ell(M_l)] = \sum_{l=1}^{m} I(X; B_l \mid B^{l-1}, \bar{U}^N) = I(X; B^{l} \mid \bar{U}^N) = I(X; Y \mid \bar{U}^N).
\end{IEEEeqnarray}

\section{Experiments}\label{sec:experiments}
\subsection{Variable-Rate Compression with Stochastic VQ-VAEs}\label{sec:experiment_vqvae}
VQ-VAEs \citep{vandenoord2017} are autoencoders which compress high-dimensional data such as images into discrete tokens, and have emerged in recent years as  fundamental components in autoregressive image generation models \citep{esser2021,yu2024}.
Letting $W$ be the input image, a standard VQ-VAE first consists of an encoder network producing a sequence of $n$ latent vectors $z^n = f_{\phi_e}(W) \in \mathbb{R}^{n \times D}$, typically arranged in a grid, where $D$ is the dimension of the latent space.
A vector quantizer then snaps each latent to the nearest codeword from a learned codebook $c^N$, producing
\begin{IEEEeqnarray}{c}
    k_i^\ast = \argmin_{1 \leq k \leq N} \norm{f_{\phi_e}(W)_i - c_k}_2, \ \hat{z}_i = c_{k_i^\ast}. \label{eqn:vq_vae_assignment}
\end{IEEEeqnarray}
Finally, a decoder network takes the quantized latents and outputs a reconstruction $\hat{W} = g_{\phi_d}(\hat{z}^n)$.

Unlike $\beta$-VAEs \citep{higgins2017} which offer inherent rate-distortion flexibility, the rate of a basic VQ-VAE is fixed at approximately $\log N$ bits per latent, as the one-hot posterior over the indices induced by \eqref{eqn:vq_vae_assignment} does not provide a differentiable rate objective.
However, if the deterministic output $\hat{z}_i$ is replaced with a learned Gaussian posterior, thus creating a \emph{stochastic} variant of the VQ-VAE, our permuted scheme can take the discrete indices returned by the argmin and reduce the communication cost below $\log N$ by adding noise to the latents.
To this end, we learn an isotropic Gaussian with mean and variance vectors $\mu_k, \sigma_k^2 \in \mathbb{R}^D$ in place of the fixed codeword $c_k$ for each codebook index.
We note that related stochastic formulations have previously been explored to enhance the training of VQ-VAEs \citep{takida2022,xu2026}, but have not to the best our knowledge been adopted at inference time to reduce the rate.
By adopting the loss-conditional training technique of \citet{dosovitskiy2020} to scale the variances and condition the decoder at inference time based on a scalar input $\lambda$, our approach can train a single VQ-VAE checkpoint to sweep a continuous range of rate-distortion operating points without retraining.

We demonstrate our approach using image compression on CIFAR-10 upscaled to $64 \times 64$, with all models sharing the architecture from \citet{vandenoord2017}.
The input symbols fed to the channel simulator are $X_i = k^\ast_i$ as in \eqref{eqn:vq_vae_assignment} and the alphabet size is $N = 256$.
Each batch of 32 images yields 8192 discrete latents which are jointly compressed by simulating the channel $P_{Y\mid X}(\cdot \mid x) = \mathcal{N}(\mu_x, \operatorname{diag}(\sigma_x^2))$ using the polar-MLC algorithm from \cref{sec:polar}, where the Sinkhorn-Knopp approximation is used for the discrete posterior.
In \cref{fig:rd_plot} we compare against individually trained deterministic VQ-VAEs with smaller codebooks.
By way of ablation, we also evaluate the same variable-rate stochastic models replacing our polar-MLC scheme with coupled IS \citep{phan2024}, closely related to  ordered random coding \citep{theis2022b}.
Our method smoothly achieves rates between $1.6$ and $7.8$ bits/latent while outperforming the IS baseline with varying chunk sizes and incurring only a small distortion penalty over individually trained fixed-rate models.
The experimental setup is discussed in further detail in \cref{sec:img_compression_extra}.

\begin{figure}[tb]
    \centering
    \begin{subfigure}[b]{0.6\linewidth}
        \centering
        \begin{tikzpicture}
\begin{axis}[
    width=8.5cm, 
    height=6.2cm,
    tick label style={font=\footnotesize},
    xlabel={\footnotesize Rate (bits/latent)},
    ylabel={\footnotesize MSE $(\times 10^{-3})$},
    xmax=8,
    grid=major,
    max space between ticks=25,
    axis lines=left,
    axis line style={-},
    enlarge x limits=false,
    legend pos=north east,
    legend cell align={left},
    legend style={font=\scriptsize, inner sep=1.5pt, row sep=0pt, column sep=3pt, legend image post style={scale=0.8}, cells={anchor=west}}
]

\addplot[color=c0, thick] table[x=rate_mean,y expr=\thisrow{mse_dec_mean}*1000,col sep=comma] {plots/polar_data.csv};
\addlegendentry{Polar-MLC}

\addplot[draw=none, name path=lower, forget plot] table[x=rate_mean,y expr=\thisrow{mse_dec_p05}*1000,col sep=comma] {plots/polar_data.csv};

\addplot[draw=none, name path=upper, forget plot] table[x=rate_mean,y expr=\thisrow{mse_dec_p95}*1000,col sep=comma] {plots/polar_data.csv};

\addplot[fill=c0!30, opacity=0.7, forget plot] fill between[of=lower and upper];

\addplot[color=c4, thick] table[x=rate_bnd_mean,y expr=\thisrow{mse_mean}*1000,col sep=comma] {plots/is4_data.csv};
\addlegendentry{IS (chunk size 4)};

\addplot[draw=none, name path=lower, forget plot] table[x=rate_bnd_mean,y expr=\thisrow{mse_p05}*1000,col sep=comma] {plots/is4_data.csv};

\addplot[draw=none, name path=upper, forget plot] table[x=rate_bnd_mean,y expr=\thisrow{mse_p95}*1000,col sep=comma] {plots/is4_data.csv};

\addplot[fill=c4!30, opacity=0.7, forget plot] fill between[of=lower and upper];

\addplot[color=c2, thick] table[x=rate_bnd_mean,y expr=\thisrow{mse_mean}*1000,col sep=comma] {plots/is2_data.csv};
\addlegendentry{IS (chunk size 2)};

\addplot[draw=none, name path=lower, forget plot] table[x=rate_bnd_mean,y expr=\thisrow{mse_p05}*1000,col sep=comma] {plots/is2_data.csv};

\addplot[draw=none, name path=upper, forget plot] table[x=rate_bnd_mean,y expr=\thisrow{mse_p95}*1000,col sep=comma] {plots/is2_data.csv};

\addplot[fill=c2!30, opacity=0.7, forget plot] fill between[of=lower and upper];

\addplot[color=c3, thick] table[x=rate_bnd_mean,y expr=\thisrow{mse_mean}*1000,col sep=comma] {plots/is1_data.csv};
\addlegendentry{IS (chunk size 1)};

\addplot[draw=none, name path=lower, forget plot] table[x=rate_bnd_mean,y expr=\thisrow{mse_p05}*1000,col sep=comma] {plots/is1_data.csv};

\addplot[draw=none, name path=upper, forget plot] table[x=rate_bnd_mean,y expr=\thisrow{mse_p95}*1000,col sep=comma] {plots/is1_data.csv};

\addplot[fill=c3!30, opacity=0.7, forget plot] fill between[of=lower and upper];





\addplot[color=c1, only marks] table[x=rate_mean,y expr=\thisrow{mse_mean}*1000,col sep=comma] {plots/baseline_data.csv};
\addlegendentry{VQ-VAE (7 models)}

\end{axis}
\end{tikzpicture}
    \end{subfigure}
    \hfill
    \begin{subfigure}[b]{0.39\linewidth}
        \centering
        \raisebox{15pt}{
            \includestandalone{diagrams/image_grid}
        }
    \end{subfigure}
    \caption{Rate-distortion performance. \emph{Left:} Using discrete-to-continuous channel simulation, a single stochastic model with the rate chosen at inference time can approach the performance of individually trained deterministic models with fixed codebooks, especially at high rates. P05--P95 distortions are shown. \emph{Right:} CIFAR-10 images compressed using fixed-rate VQ-VAEs and our variable-rate stochastic model.}
    \label{fig:rd_plot}
    \vspace{-1em}
\end{figure}
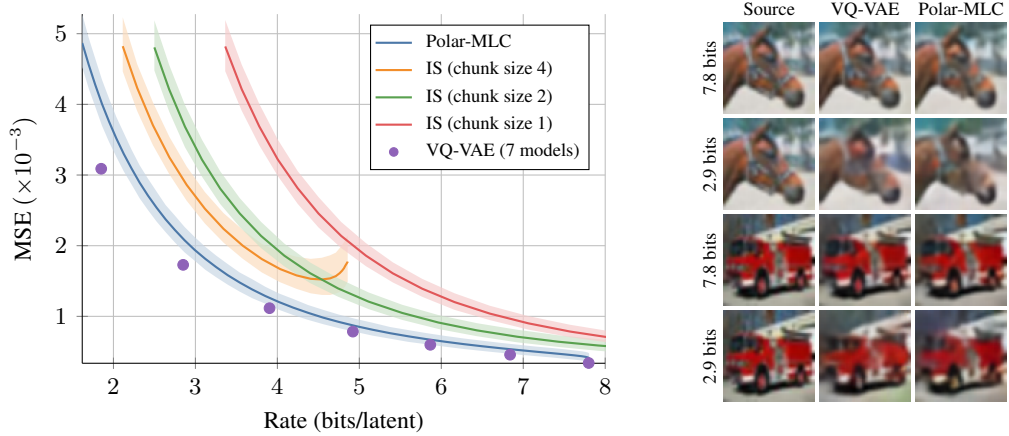

\subsection{Distributed Mean Estimation with Central Differential Privacy}\label{sec:experiment_cdp}
DME is a fundamental primitive in private distributed learning and optimization \citep{abadi2016,mcmahan17}.
In a general DME setting under DP, $C$ clients each hold data $X_i \in \mc{X}$ and an aggregator aims to estimate the mean $\mu = \sum_{i=1}^{C} X_i / C$ subject to privacy constraints.
A DP mechanism $P_{Y \mid X}$ is applied identically to each $X_i$ to give a randomized output $Y_i$, yielding the estimate $\hat{\mu} = \sum_{i=1}^{C} Y_i / C$.
Given the desired privacy $\varepsilon$, the aim is to minimize the expected mean squared error (MSE) $\E[\norm{\hat{\mu} - \mu}_2^2]$.
We consider CDP, where the aggregator is trusted and releases $\hat{\mu}$ to an untrusted party. 
Although this setting is weaker than LDP, where the aggregator itself is also untrusted, in practice it allows for significantly more accurate estimators \citep{chen2023}.

Our experiments adopt a similar setting to \citet{chen2023} and \citet{liu2024}, applying the Gaussian mechanism for $(\varepsilon, \delta)$-CDP with $\delta = 10^{-6}$.
We use $C=10^5$ clients and the dimension of a local vector $X_i$ is $D=512$.
Each coordinate $X_{ij}$ is uniform on a set of 16 values evenly spaced in $[-1,1]$.
Compared to the binary vectors used in \citet{liu2024}, this larger alphabet size has relevance to emerging applications in ultra-low-precision, e.g.\ INT4, federated learning \citep{ding2025}.
To reduce communication overhead, we assume $B=8$ instances of the DME problem can be batched and encoded jointly, giving a coding blocklength up to $BD = 4096$. 
In \cref{fig:cdp}, we plot the bits per sample at privacy $\varepsilon \in [0.05,30]$, with the required Gaussian noise calculated using the R\'{e}nyi DP accountant \citep{mironov2017}.
We compare with PFR as a standard channel simulation baseline; note that the Poisson private representation (PPR, \citealp{liu2024}) incurs a higher expected coding cost than PFR and is unnecessary for CDP. Both algorithms tested perform exact simulation and therefore provably achieve the same expected MSE as the uncompressed Gaussian mechanism at a given $\varepsilon$.
Our polar-MLC scheme achieves lower rates across the operating range, largely thanks to its ability to jointly encode all 4096 samples in each client vector at practical speeds.
More details are in \cref{sec:dme_extra}, including a comparison with dithered quantization which, while not a general-purpose channel simulation method, has also been used for compressed DME \citep{hasircioglu2024}.

\begin{figure}[tb]
    \centering
    \begin{subfigure}[c]{0.58\linewidth}
        \centering
        \begin{tikzpicture}
\begin{axis}[
    width=8cm, 
    height=5.5cm,
    tick label style={font=\footnotesize},
    xlabel={\footnotesize Privacy $\varepsilon$},
    ylabel={\footnotesize Rate (bits/sample)},
    xmax=30,
    clip=false,
    grid=major,
    max space between ticks=25,
    axis lines=left,
    axis line style={-},
    enlarge x limits=false,
    legend pos=south east,
    legend cell align={left},
    legend style={font=\scriptsize, inner sep=1.5pt, row sep=0pt, column sep=3pt, legend image post style={scale=0.8}, cells={anchor=west}, reverse legend}
]

\addplot[color=c3, thick] table[x=eps,y=I_X_Y_U,col sep=comma] {plots/theoretical_results_scalar.csv};
\addlegendentry{$I(X;Y \mid \bar{U}^{16})$}

\addplot[color=c2, thick] table[x=eps,y=I_X_Y,col sep=comma] {plots/theoretical_results_scalar.csv};
\addlegendentry{$I(X;Y)$}

\addplot[color=c4, thick] table[x=eps,y=dat_mean,col sep=comma] {plots/pfr_results_zipf.csv};
\addlegendentry{PFR}

\addplot[draw=none, name path=lower, forget plot] table[x=eps,y=dat_p05,col sep=comma] {plots/pfr_results_zipf.csv};

\addplot[draw=none, name path=upper, forget plot] table[x=eps,y=dat_p95,col sep=comma] {plots/pfr_results_zipf.csv};

\addplot[fill=c4!30, opacity=0.7, forget plot] fill between[of=lower and upper];

\addplot[color=c0, thick] table[x=eps,y=dat_mean,col sep=comma] {plots/polar_results_scalar.csv};
\addlegendentry{Polar-MLC}

\addplot[draw=none, name path=lower, forget plot] table[x=eps,y=dat_p05,col sep=comma] {plots/polar_results_scalar.csv};

\addplot[draw=none, name path=upper, forget plot] table[x=eps,y=dat_p95,col sep=comma] {plots/polar_results_scalar.csv};

\addplot[fill=c0!30, opacity=0.7, forget plot] fill between[of=lower and upper];

\end{axis}
\end{tikzpicture}
    \end{subfigure}
    \hfill
    \begin{subfigure}[c]{0.41\linewidth}
        \centering
        \setlength\tabcolsep{5 pt}
        \newcolumntype{T}{S[
            table-format=1.2(1),
            uncertainty-mode=separate
        ]}
        \raisebox{50pt}{
        \begin{tabular}{@{}l
            S[table-format=1.2(1), uncertainty-mode=separate, table-text-alignment=right]
            S[table-format=5.1(5), uncertainty-mode=separate, table-text-alignment=right]
        @{}}
            \toprule
            $\varepsilon$ & {Polar-MLC} & {PFR} \\
            \midrule
            0.05 & 3.06(32) & 58.7(78) \\
            2.19 & 2.96(12) & 182.2(400) \\
            4.33 & 2.93(03) & 1852.8(755) \\
            6.47 & 3.03(24) & 13148.9(2990) \\
            8.61 & 2.93(03) & 59890.6(1801.4) \\
            \bottomrule
        \end{tabular}}
    \end{subfigure}
    \caption{DME using the Gaussian mechanism. \emph{Left:} Communication rate as a function of the privacy $\varepsilon$. P05--P95 rates are shown. \emph{Right:} Encoding time in seconds and standard deviation for 128 blocks of 4096 samples; PFR splits the client vector into chunks of 8 samples. For comparison, the final arithmetic/entropy coding stage is not included for either method.}
    \label{fig:cdp}
    \vspace{-1em}
\end{figure}
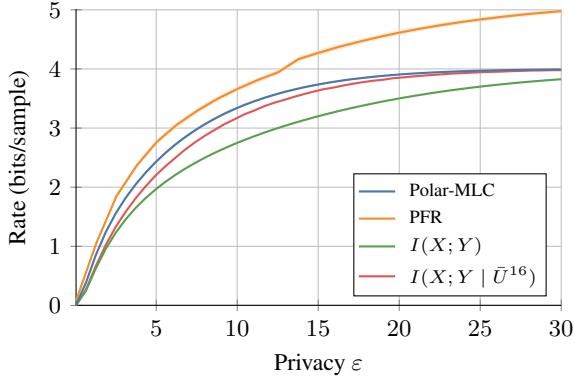

\section{Conclusion}
In this work, we have considered exact and approximate channel simulation with a discrete input alphabet and introduced the permuted scheme to solve both variants while using a fixed number of shared random samples.
Following this, we applied techniques from coding theory, namely polar and multi-level codes, to scale the scheme efficiently to long blocklengths.
Our polar-coded discrete-to-continuous simulation scheme was demonstrated in two relevant machine learning applications: variable-rate neural image compression with VQ-VAEs and communication-efficient DME with privacy guarantees via the Gaussian mechanism.
While the general channel simulation problem remains computationally difficult, our work shows that searching for scalable constructions on narrower yet practically relevant classes of channels may constitute a promising direction for continued research.

\bibliography{iclr2027_conference}

@article{li2024a,
    author={Cheuk Ting Li},
    title={Channel simulation: theory and applications to lossy compression and differential privacy},
    journal={Foundations and Trends in Communications and Information Theory},
    volume={21},
    number={6},
    pages={847--1106},
    year={2024}
}

@article{li2018,
    author={Cheuk Ting Li and Abbas El Gamal},
    journal={IEEE Transactions on Information Theory}, 
    title={Strong functional representation lemma and applications to coding theorems}, 
    year={2018},
    volume={64},
    number={11},
    pages={6967--6978}
}

@inproceedings{flamich2023a,
    author={Gergely Flamich},
    pages={37089--37127},
    title={Greedy {P}oisson rejection sampling},
    booktitle={Advances in Neural Information Processing Systems},
    volume={36},
    year={2023}
}

@inproceedings{mena2020,
    title={{S}inkhorn permutation variational marginal inference},
    author={Gonzalo Mena and Erdem Varol and Amin Nejatbakhsh and Eviatar Yemini and Liam Paninski},
    booktitle={Proceedings of the 2nd Symposium on Advances in Approximate Bayesian Inference},
    pages={1--9},
    year={2020}
}

@inproceedings{mena2018,
    author={Gonzalo Mena and David Belanger and Scott Linderman and Jasper Snoek},
    title={Learning latent permutations with {G}umbel-{S}inkhorn networks},
    booktitle={International Conference on Learning Representations},
    year={2018}
}

@inproceedings{sriramu2024,
    author={Sharang M. Sriramu and Rochelle Barsz and Elizabeth Polito and Aaron B. Wagner},
    booktitle={Advances in Neural Information Processing Systems},
    title={Fast channel simulation via error-correcting codes},
    pages={107932--107959},
    volume={37},
    year={2024}
}

@article{arikan2009,
    author={Erdal Ar{\i}kan},
    journal={{IEEE} Transactions on Information Theory}, 
    title={Channel polarization: a method for constructing capacity-achieving codes for symmetric binary-input memoryless channels}, 
    year={2009},
    volume={55},
    number={7},
    pages={3051--3073}
}

@inproceedings{arikan2010,
    author={Erdal Ar{\i}kan},
    booktitle={{IEEE} International Symposium on Information Theory}, 
    title={Source polarization}, 
    year={2010},
    pages={899--903}
}

@inproceedings{ozyilkan2026,
    title={{SoftBinary} coding: a new information-theoretic neural compression paradigm}, 
    author={Ezgi Ozyilkan and Sharang M. Sriramu and Elza Erkip and Aaron B. Wagner and Jona Ball{\'e}},
    booktitle={43rd International Conference on Machine Learning},
    year={2026}
}

@article{rissanen1976,
    author={Jorma Rissanen},
    title={Generalized {K}raft inequality and arithmetic coding},
    journal={{IBM} Journal of Research and Development},
    year={1976},
    pages={198--203},
    volume={20},
    number={3}
}

@article{wachsmann1999,
    author={Udo Wachsmann and  Robert F. H. Fischer and Johannes B. Huber},
    journal={{IEEE} Transactions on Information Theory}, 
    title={Multilevel codes: theoretical concepts and practical design rules}, 
    year={1999},
    volume={45},
    number={5},
    pages={1361--1391}
}

@inproceedings{flamich2023b,
    title={Adaptive greedy rejection sampling},
    author={Gergely Flamich and Lucas Theis},
    booktitle={{IEEE} International Symposium on Information Theory},
    pages={454--459},
    year={2023},
    organization={IEEE}
}

@article{harsha2010,
    author={Prahladh Harsha and Rahul Jain and David McAllester and Jaikumar Radhakrishnan},
    title={The communication complexity of correlation},
    journal={{IEEE} Transactions on Information Theory},
    year={2010},
    volume={56},
    number={1},
    pages={438--449}
}

@inproceedings{phan2024,
    title={Importance matching lemma for lossy compression with side information},
    author={Buu Phan and Ashish Khisti and Christos Louizos},
    booktitle={Proceedings of the 27th International Conference on Artificial Intelligence and Statistics},
    year={2024},
    pages={1387--1395}
}

@inproceedings{havasi2019,
    author={Marton Havasi and Robert Peharz and Jos{\'e} Miguel Hern{\'a}ndez-Lobato},
    title={Minimal random code learning: getting bits back from compressed model parameters},
    booktitle={International Conference on Learning Representations},
    year={2019}
}

@article{bennett2002,
    author={Charles H. Bennett and Peter W. Shor and John A. Smolin and Ashish V. Thapliyal},
    journal={{IEEE} Transactions on Information Theory}, 
    title={Entanglement-assisted capacity of a quantum channel and the reverse {S}hannon theorem},
    year={2002},
    volume={48},
    number={10},
    pages={2637--2655},
}

@inproceedings{zhao2026,
    author={Jianguo Zhao and Cheuk Ting Li},
    booktitle={{IEEE} International Symposium on Information Theory}, 
    title={Rejection-sampled linear codes for channel simulation}, 
    year={2026},
    pages={1--6},
}

@article{tang2023,
    author={Wenpin Tang and Fengmin Tang},
    title={The {P}oisson binomial distribution -- old and new},
    journal={Statistical Science},
    year={2023},
    volume={38},
    number={1},
    pages={108--119}
}

@inproceedings{hasircioglu2024,
    author={Burak Has{\i}rc{\i}o{\u{g}}lu and Deniz G{\"u}nd{\"u}z},
    title={Communication efficient private federated learning using dithering},
    booktitle={{IEEE} International Conference on Acoustics, Speech and Signal Processing},
    year={2024},
    pages={7575--7579}
}

@inproceedings{vandenoord2017,
    author={Aaron van den Oord and Oriol Vinyals and Koray Kavukcuoglu},
    booktitle={Advances in Neural Information Processing Systems},
    title={Neural discrete representation learning},
    volume={30},
    year={2017}
}

@article{knight2008,
    title={The {S}inkhorn--{K}nopp algorithm: convergence and applications},
    author={Philip A. Knight},
    journal={{SIAM} Journal on Matrix Analysis and Applications},
    volume={30},
    number={1},
    pages={261--275},
    year={2008}
}

@article{sinkhorn1964,
    title={A relationship between arbitrary positive matrices and doubly stochastic matrices},
    author={Richard Sinkhorn},
    journal={The Annals of Mathematical Statistics},
    volume={35},
    number={2},
    pages={876--879},
    year={1964}
}

@misc{winter2002,
    title={Compression of sources of probability distributions and density operators}, 
    author={Andreas Winter},
    year={2002},
    eprint={quant-ph/0208131},
    archivePrefix={arXiv},
    primaryClass={quant-ph},
    url={https://arxiv.org/abs/quant-ph/0208131}, 
}

@article{flamich2023c,
    title={Faster relative entropy coding with greedy rejection coding},
    author={Gergely Flamich and Stratis Markou and Jos{\'e} Miguel Hern{\'a}ndez-Lobato},
    journal={Advances in Neural Information Processing Systems},
    volume={36},
    pages={50558--50569},
    year={2023}
}

@inproceedings{phan2025,
    author={Buu Phan and Ashish Khisti},
    booktitle={Advances in Neural Information Processing Systems},
    pages={43066--43123},
    title={Channel simulation and distributed compression with ensemble rejection sampling},
    volume={38},
    year={2025}
}

@misc{hill2026,
    title={Rejection sampling is optimal for relative entropy coding}, 
    author={Spencer Hill and Fady Alajaji and Tam{\'a}s Linder and Gergely Flamich},
    year={2026},
    eprint={2604.23076},
    archivePrefix={arXiv},
    primaryClass={cs.IT},
    url={https://arxiv.org/abs/2604.23076}, 
}

@inproceedings{flamich2022,
    title={Fast relative entropy coding with {A*} coding},
    author={Gergely Flamich and Stratis Markou and Jos{\'e} Miguel Hern{\'a}ndez-Lobato},
    booktitle={Proceedings of the 39th International Conference on Machine Learning},
    pages={6548--6577},
    year={2022}
}

@inproceedings{hegazy2022,
    author={Mahmoud Hegazy and Cheuk Ting Li},
    booktitle={2022 IEEE Information Theory Workshop}, 
    title={Randomized quantization with exact error distribution}, 
    year={2022},
    pages={350--355}
}

@inproceedings{flamich2020,
    author={Gergely Flamich and Marton Havasi and Jos{\'e} Miguel Hern{\'a}ndez-Lobato},
    booktitle={Advances in Neural Information Processing Systems},
    pages={16131--16141},
    title={Compressing images by encoding their latent representations with relative entropy coding},
    volume={33},
    year={2020}
}

@misc{theis2022a,
    title={Lossy compression with {G}aussian diffusion}, 
    author={Lucas Theis and Tim Salimans and Matthew D. Hoffman and Fabian Mentzer},
    year={2022},
    eprint={2206.08889},
    archivePrefix={arXiv},
    primaryClass={stat.ML},
    url={https://arxiv.org/abs/2206.08889}, 
}

@inproceedings{vonderfecht2025,
    author={Jeremy Vonderfecht and Liu Feng},
    booktitle={International Conference on Learning Representations},
    pages={687--702},
    title={Lossy compression with pretrained diffusion models},
    year={2025}
}

@inproceedings{dwork2006,
    author={Cynthia Dwork and Frank McSherry and Kobbi Nissim and Adam Smith},
    title={Calibrating noise to sensitivity in private data analysis},
    booktitle={Theory of Cryptography},
    year={2006},
    pages={265--284},
}

@inproceedings{feldman2021,
    title={Lossless compression of efficient private local randomizers},
    author={Vitaly Feldman and Kunal Talwar},
    booktitle={Proceedings of the 38th International Conference on Machine Learning},
    pages={3208--3219},
    year={2021}
}

@inproceedings{shah2022,
    title={Optimal compression of locally differentially private mechanisms},
    author={Abhin Shah and Wei-Ning Chen and Johannes Ball{\'e} and Peter Kairouz and Lucas Theis},
    booktitle={Proceedings of the 25th International Conference on Artificial Intelligence and Statistics},
    pages={7680--7723},
    year={2022}
}

@inproceedings{hegazy2024,
    title={Compression with exact error distribution for federated learning},
    author={Mahmoud Hegazy and R{\'e}mi Leluc and Cheuk Ting Li and Aymeric Dieuleveut},
    booktitle={Proceedings of the 27th International Conference on Artificial Intelligence and Statistics},
    pages={613--621},
    year={2024}
}

@inproceedings{liu2024,
    author={Yanxiao Liu and Wei-Ning Chen and Ayfer {\"O}zg{\"u}r and Cheuk Ting Li},
    booktitle={Advances in Neural Information Processing Systems},
    title={Universal exact compression of differentially private mechanisms},
    pages={91492--91531},
    volume={37},
    year={2024}
}

@article{dwork2014,
    author={Cynthia Dwork and Aaron Roth},
    title={The algorithmic foundations of differential privacy},
    year={2014},
    volume={9},
    number={3--4},
    journal={Foundations and Trends in Theoretical Computer Science},
    pages={211--407}
}

@inproceedings{theis2022b,
    title={Algorithms for the communication of samples},
    author={Lucas Theis and Noureldin Y. Ahmed},
    booktitle={Proceedings of the 39th International Conference on Machine Learning},
    pages={21308--21328},
    year={2022},
}

@inproceedings{dosovitskiy2020,
    title={You only train once: loss-conditional training of deep networks},
    author={Alexey Dosovitskiy and Josip Djolonga},
    booktitle={International Conference on Learning Representations},
    year={2020}
}

@inproceedings{higgins2017,
    title={$\beta$-{VAE}: learning basic visual concepts with a constrained variational framework},
    author={Irina Higgins and Loic Matthey and Arka Pal and Christopher Burgess and Xavier Glorot and Matthew Botvinick and Shakir Mohamed and Alexander Lerchner},
    booktitle={International Conference on Learning Representations},
    year={2017}
}

@inproceedings{xu2026,
    title={Training-free vector quantization via {G}aussian {VAE}s}, 
    author={Tongda Xu and Wendi Zheng and Jiajun He and Jos{\'e} Miguel Hern{\'a}ndez-Lobato and Yan Wang and Ya-Qin Zhang and Jie Tang},
    year={2026},
    booktitle={43rd International Conference on Machine Learning}
}

@inproceedings{takida2022,
    title={{SQ-VAE}: variational {B}ayes on discrete representation with self-annealed stochastic quantization},
    author={Yuhta Takida and Takashi Shibuya and Weihsiang Liao and Chieh-Hsin Lai and Junki Ohmura and Toshimitsu Uesaka and Naoki Murata and Shusuke Takahashi and Toshiyuki Kumakura and Yuki Mitsufuji},
    booktitle={Proceedings of the 39th International Conference on Machine Learning},
    pages={20987--21012},
    year={2022},
}

@inproceedings{chen2023,
    author={Wei-Ning Chen and Dan Song and Ayfer {\"O}zg{\"u}r and Peter Kairouz},
    booktitle={Advances in Neural Information Processing Systems},
    pages={69202--69227},
    title={Privacy amplification via compression: achieving the optimal privacy-accuracy-communication trade-off in distributed mean estimation},
    volume={36},
    year={2023}
}

@inproceedings{ding2025,
    title={{LBI-FL}: low-bit integerized federated learning with temporally dynamic bit-width allocation},
    author={Li Ding and Hao Zhang and Wenrui Dai and Chenglin Li and Weijia Lu and Zhifei Yang and Xiaodong Zhang and Xiaofeng Ma and Junni Zou and Hongkai Xiong},
    booktitle={Proceedings of the 42nd International Conference on Machine Learning},
    pages={13885--13899},
    year={2025},
}

@book{polyanskiy2025, 
    address={Cambridge},
    title={Information theory: from coding to learning}, publisher={Cambridge University Press},
    author={Polyanskiy, Yury and Wu, Yihong}, 
    year={2025}
}

@article{mamatov1965,
    author={M. M. Mamatov},
    title={Estimation of the remainder term of the generalized {M}oivre-{L}aplace local limit theorem},
    journal={Fergan. Gos. Ped. Inst. U\v cen. Zap. Ser. Mat.},
    year={1965},
    volume={1},
    pages={63--66},
    note={MR0200962 (34 847)}
}

@inproceedings{altschuler2017,
    author={Jason Altschuler and Jonathan Niles-Weed and Philippe Rigollet},
    booktitle={Advances in Neural Information Processing Systems},
    title = {Near-linear time approximation algorithms for optimal transport via {S}inkhorn iteration},
    volume={30},
    year={2017}
}

@inproceedings{mironov2017,
    author={Ilya Mironov},
    booktitle={{IEEE} 30th Computer Security Foundations Symposium}, 
    title={{R}{\'e}nyi differential privacy}, 
    year={2017},
    pages={263--275},
}

@inproceedings{perez2018,
    author={Ethan Perez and Florian Strub and Harm de Vries and Vincent Dumoulin and Aaron Courville},
    title={{FiLM}: visual reasoning with a general conditioning layer},
    year = {2018},
    booktitle = {Proceedings of the 32nd {AAAI} Conference on Artificial Intelligence},
}

@article{roberts1962,
    author={Lawrence Roberts},
    journal={IRE Transactions on Information Theory}, 
    title={Picture coding using pseudo-random noise}, 
    year={1962},
    volume={8},
    number={2},
    pages={145--154}
}

@article{gray1993,
    author={Robert M. Gray and Thomas G. Stockham},
    journal={IEEE Transactions on Information Theory}, 
    title={Dithered quantizers}, 
    year={1993},
    volume={39},
    number={3},
    pages={805--812},
}

@inproceedings{kobus2024,
    author={Szymon Kobus and Lucas Theis and Deniz G{\"u}nd{\"u}z},
    booktitle={{IEEE} International Symposium on Information Theory}, 
    title={{G}aussian channel simulation with rotated dithered quantization}, 
    year={2024},
    pages={1907--1912},
}

@inproceedings{shahmiri2024,
    author={Ali Moradi Shahmiri and Chih Wei Ling and Cheuk Ting Li},
    booktitle={IEEE International Conference on Acoustics, Speech and Signal Processing}, 
    title={Communication-efficient {L}aplace mechanism for differential privacy via random quantization}, 
    year={2024},
    pages={4550--4554},
}

@misc{ling2025,
    title={Communication-efficient and privacy-adaptable mechanism for federated learning}, 
    author={Chih Wei Ling and Chun Hei Michael Shiu and Youqi Wu and Jiande Sun and Cheuk Ting Li and Linqi Song and Weitao Xu},
    year={2025},
    eprint={2501.12046},
    archivePrefix={arXiv},
    primaryClass={cs.LG},
    url={https://arxiv.org/abs/2501.12046}, 
}

@inproceedings{abadi2016,
    author={Martin Abadi and Andy Chu and Ian Goodfellow and H. Brendan McMahan and Ilya Mironov and Kunal Talwar and Zhang Li},
    title={Deep learning with differential privacy},
    year={2016},
    booktitle={Proceedings of the 2016 {ACM} {SIGSAC} Conference on Computer and Communications Security},
    pages={308--318},
}

@inproceedings{mcmahan17,
    title={Communication-efficient learning of deep networks from decentralized data},
    author={H. Brendan McMahan and Eider Moore and Daniel Ramage and Seth Hampson and  Blaise Ag{\"u}era y Arcas},
    booktitle={Proceedings of the 20th International Conference on Artificial Intelligence and Statistics},
    pages={1273--1282},
    year={2017}
}

@inproceedings{esser2021,
    author={Patrick Esser and Robin Rombach and Bj{\"o}rn Ommer},
    booktitle={Proceedings of the {IEEE/CVF} Conference on Computer Vision and Pattern Recognition},
    title={Taming transformers for high-resolution image synthesis},
    year={2021},
    pages={12873--12883},
}

@inproceedings{yu2024,
    title={Language model beats diffusion --- tokenizer is key to visual generation},
    author={Lijun Yu and Jos{\'e} Lezama and Nitesh Bharadwaj Gundavarapu and Luca Versari and Kihyuk Sohn and David Minnen and Yong Cheng and Agrim Gupta and Xiuye Gu and Alexander G Hauptmann and Boqing Gong and Ming-Hsuan Yang and Irfan Essa and David A Ross and Lu Jiang},
    booktitle={International Conference on Learning Representations},
    year={2024}
}

@article{flamich2026,
    author={Gergely Flamich and Deniz G{\"u}nd{\"u}z},
    journal={{IEEE} {BITS} the Information Theory Magazine}, 
    title={Data compression with stochastic codes}, 
    year={2026},
    pages={1--14},
}

@book{conover1999,
    author={W. J. Conover},
    title={Practical nonparametric statistics},
    publisher={Wiley},
    address={New York},
    year={1999},
    edition={3}
}
\bibliographystyle{iclr2027_conference}

\newpage
\appendix
\section{BI-AWGN Example}\label[appendix]{sec:biawgn_example}
\begin{figure}[b]
    \centering
    \begin{tikzpicture}[semithick, font=\footnotesize]

\node [draw, circle, minimum size=0.5cm, fill=blue!25] (adder) {$+$};
\node [left=1.2cm of adder] (input) {$W \sim \operatorname{Ber}(0.5)$};
\node [below=0.8cm of adder] (noise) {$Z \sim \mathcal{N}(0, \sigma^2)$};
\node [right=1.2cm of adder] (output) {$Y$};

\draw [->, >=latex] (input) -- (adder);
\draw [->, >=latex] (noise) -- (adder);
\draw [->, >=latex] (adder) -- (output);

\end{tikzpicture}
    \caption{System model for the BI-AWGN channel.}
    \label{fig:biawgn}
\end{figure}
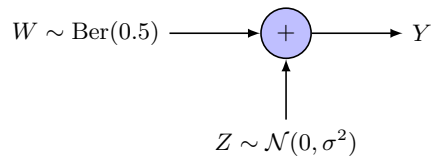

Here, we consider a concrete worked example focusing on simulating a BI-AWGN channel, as pictured in \cref{fig:biawgn}, to motivate the need for the permuted scheme to reduce the channel simulation rate while maintaining the correct output distribution using only $N = 2$ shared Gaussian samples.
In this example, the physical channel input $W \in \{ -1, 1 \}$ is uniformly distributed, and the channel noise $Z$ is a 1D Gaussian with variance $\sigma^2$.
To harmonize the problem with the indexing convention used in \cref{sec:one_shot_simulation}, we map $W$ to a canonical channel input $X \in \{ 1, 2 \}$, with $X = (W + 1) / 2 + 1$.
We have $H(X) = 1$, and the channel capacity can be shown to be, see e.g.\ \citet[Example~3.4]{polyanskiy2025},
\begin{IEEEeqnarray}{c}
    I(X;Y) = 1 - \int_{-\infty}^{\infty} \frac{1}{\sqrt{2\pi}} e^{-z^2 / 2} \log(1 + e^{2z / \sigma - 2 / \sigma^2}) \, dz.
\end{IEEEeqnarray}
We consider the one-shot channel simulation methods proposed in \cref{sec:one_shot_simulation} in turn, beginning with the naive scheme, continuing to the index randomization idea, and concluding by demonstrating how introducing a random permutation finally succeeds in driving the rate below $H(X)$.

\paragraph{Naive Scheme.}
In this particular example, the naive scheme consists of sampling 
\begin{IEEEeqnarray}{c}
    \bar{U}_1 \sim P_{Y \mid X}(\cdot \mid 1) = \mc{N}(-1, \sigma^2) \text{ and } \bar{U}_2 \sim P_{Y \mid X}(\cdot \mid 2) = \mc{N}(1, \sigma^2).
\end{IEEEeqnarray}
Upon observing the index $X = x$, the encoder would send $K = x$ to the decoder, which would then output $\bar{U}_K$.
Clearly, if $X = 1$ the decoder's output correctly follows the density $p_{Y \mid X}(y \mid 1) = \mc{N}(y; -1, \sigma^2)$, and if $X = 2$ it follows $p_{Y \mid X}(y \mid 2) = \mc{N}(y; 1, \sigma^2)$.
However, since $K$ is always equal to $X$, which in turn is independent of $\bar{U}^2$, it is not possible to transmit $K$ to the decoder using less than $H(X) = 1$ bits.
Thus, while the naive scheme produces correctly distributed output samples, it does not achieve any compression whatsoever compared to directly sending $X$ and later adding noise at the decoder.

\paragraph{Randomizing the Index Only.}
As briefly mentioned in \cref{sec:one_shot_simulation}, a first idea is to continue generating $\bar{U}_1 \sim \mc{N}(-1, \sigma^2)$ and $\bar{U}_2 \sim \mc{N}(1, \sigma^2)$, but additionally randomize the index $K$ in an attempt to reduce the communication cost.
To formalize this, we allow the encoder to sample $K$ from an arbitrary conditional distribution $P_{K \mid X, \bar{U}^2}(\cdot \mid x, u^2)$.
Without loss of generality, assume $X = 1$ for the sake of the analysis.
Applying Bayes' rule,
\begin{IEEEeqnarray}{rCl}
    P_{K \mid X, \bar{U}^2}(k \mid 1, u^2) & = & \frac{p_{\bar{U}^2 \mid K, X}(u^2 \mid k, 1) P_{K \mid X}(k \mid 1)}{p_{\bar{U}^2 \mid X}(u^2 \mid 1)} \label{eqn:naive_posterior} \\
    & = & \frac{\mc{N}(u_1; -1, \sigma^2) \mc{N}(u_2; 1, \sigma^2) P_{K \mid X}(k \mid 1)}{\mc{N}(u_1; -1, \sigma^2) \mc{N}(u_2; 1, \sigma^2)} \\
    & = & P_{K \mid X}(k \mid 1).
\end{IEEEeqnarray}
Let $\alpha = P_{K \mid X}(1 \mid 1)$, then the density of the output, which we now write as $\tilde{Y}$ to avoid confusion with the target distribution $P_{Y \mid X}$, is
\begin{IEEEeqnarray}{rCl}
    p_{\tilde{Y} \mid X}(y \mid 1) & = & \iint ( \alpha \delta(y - u_1) + (1 - \alpha) \delta(y - u_2) ) p_{\bar{U}^2}(u^2) \, du_1 du_2 \\
    & = & \alpha \int \delta(y - u_1) \mc{N}(u_1; -1, \sigma^2) \, du_1 \int \mc{N}(u_2; 1, \sigma^2) \, du_2 \\
    &&\< + (1 - \alpha) \int \delta(y - u_2) \mc{N}(u_2; 1, \sigma^2) \, du_2 \int \mc{N}(u_1; -1, \sigma^2) \, du_1 \nonumber \\
    & = & \alpha \mc{N}(y; -1, \sigma^2) + (1 - \alpha) \mc{N}(y; 1, \sigma^2).
\end{IEEEeqnarray}
The only solution yielding the desired output density $p_{\tilde{Y} \mid X}(y \mid 1) = \mc{N}(y; -1, \sigma^2)$ is the one-hot solution with $\alpha = 1$.
The same argument can be repeated for $X = 2$, showing that an exact simulation requires $P_{K \mid X, \bar{U}^2}(k \mid x, u^2) = \delta(k - x)$, i.e.\ $K = X$ just as in the naive scheme.
As a result, the best possible rate is again $H(X) = 1$, and the index randomization gives no advantage over directly sending $X$ losslessly and adding noise after the fact.

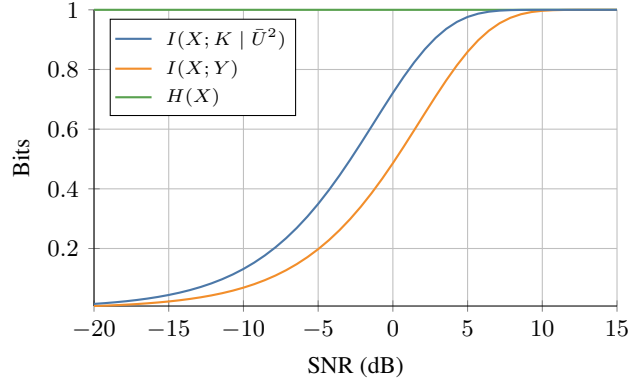
\begin{figure}[tb]
    \centering
    \begin{tikzpicture}
\begin{axis}[
    width=8.5cm, 
    height=5.5cm,
    tick label style={font=\footnotesize},
    xlabel={\footnotesize SNR (dB)},
    ylabel={\footnotesize Bits},
    xmin=-20, xmax=15,
    ymax=1,
    clip=false,
    grid=major,
    max space between ticks=25,
    axis lines=left,
    axis line style={-},
    enlarge x limits=false,
    legend pos=north west,
    legend cell align={left},
    legend style={font=\scriptsize, inner sep=1.5pt, row sep=0pt, column sep=3pt, legend image post style={scale=0.8}, cells={anchor=west}, reverse legend},
]

\addplot[color=c2, thick] coordinates {
    (-20,1) (15,1)
};
\addlegendentry{$H(X)$}

\addplot[color=c4, thick] table[x=snr,y=cap,col sep=comma] {plots/biawgn_data.csv};
\addlegendentry{$I(X; Y)$}

\addplot[color=c0, thick] table[x=snr,y=bnd,col sep=comma] {plots/biawgn_data.csv};
\addlegendentry{$I(X;K \mid \bar{U}^2)$}

\end{axis}
\end{tikzpicture}
    \caption{Conditional mutual information, capacity and source entropy for a BI-AWGN channel. The conditional mutual information $I(X; K \mid \bar{U}^2)$ is a lower bound on the permuted scheme's rate.}
    \label{fig:bi_awgn}
    \vspace{-1em}
\end{figure}

\paragraph{Introducing a Random Permutation.}
We now show how introducing a random shuffling of $\bar{U}^2$ allows us to achieve a communication cost less than the entropy.
The key failure of the index randomization scheme from the previous section can be observed in \eqref{eqn:naive_posterior}.
Because $\bar{U}_1$ and $\bar{U}_2$ are sampled from fixed distributions, $p_{\bar{U}^2 \mid K, X}(u^2 \mid k, x)$ and $p_{\bar{U}^2 \mid X}(u^2 \mid x)$ are equal for any $u^2$, $x$ and $k$.
This leads to the conclusion that $K$ must be independent of $\bar{U}^2$ given $X$.
The encoder and decoder therefore cannot take advantage of the shared Gaussian samples to reduce the rate.

As discussed in \cref{sec:oneshot_permuted}, we break this this dependency using a random permutation of the samples.
The probability of $K$ given $X$ before seeing $\bar{U}^2$ then becomes uniform as, after the random permutation, $\bar{U}_1$ is equally likely to have been sampled from mode 1 or 2, and likewise for $\bar{U}_2$.
In the BI-AWGN case, $\Pi$ is a uniform permutation of $\{ 1, 2\}$, meaning:

\begin{enumerate}
    \item With probability $1/2$, $\bar{U}_1 \sim \mc{N}(-1, \sigma^2)$ and $\bar{U}_2 \sim \mc{N}(1, \sigma^2)$.
    \item With probability $1/2$, $\bar{U}_1 \sim \mc{N}(1, \sigma^2)$ and $\bar{U}_2 \sim \mc{N}(-1, \sigma^2)$.
\end{enumerate}

\begin{figure}[b]
    \centering
    \begin{tikzpicture}[font=\footnotesize, semithick, >=latex]

\node [draw, circle, fill=blue!25] (samp1) {$\bar{U}_2$};

\node [draw, circle, fill=blue!25, left=4cm of samp1] (samp2) {$\bar{U}_1$};

\node [draw, fill=green!25, below=2cm of samp1] (mode1) {$P_{Y \mid X}(\cdot \mid 2)$};

\node [draw, fill=green!25, below=2cm of samp2] (mode2) {$P_{Y \mid X}(\cdot \mid 1)$};

\coordinate (M) at ($(mode2.north)!0.5!(samp1.south)$);
\node [draw, minimum width=6.2cm, minimum height=0.7cm, fill=gray!8] (perm) at (M) {};
\node [xshift=-2.8cm] at (perm) {$\Pi$};

\draw[->] (mode1.north) -- (samp1.south) node[xshift=0.2cm, yshift=-0.25cm] {$\frac{1}{2}$};
\draw[->] (mode1.north) -- (samp2.south) node[xshift=-0.2cm, yshift=-0.25cm] {$\frac{1}{2}$};
\draw[->] (mode2.north) -- (samp1.south) node[xshift=-1.35cm, yshift=-0.25cm] {$\frac{1}{2}$};
\draw[->] (mode2.north) -- (samp2.south) node[xshift=1.35cm, yshift=-0.25cm] {$\frac{1}{2}$};

\end{tikzpicture}
    \caption{Illustration of the random permutation $\Pi$ used by our scheme on the BI-AWGN channel.}
    \label{fig:permutation}
\end{figure}
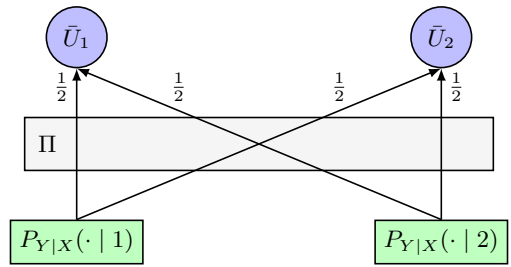

We visualize this procedure in \cref{fig:permutation}.
To sample $K$, the natural choice which we pursued in \cref{sec:oneshot_permuted} is to use the Bayesian posterior, the probability of $K = k$ being set equal to the probability given $\bar{U}^2$ and $X$ that $\bar{U}_k$ was the sample taken from the conditional distribution $P_{Y \mid X}(\cdot \mid X)$.
For our BI-AWGN simulation, this is
\begin{IEEEeqnarray}{c}
    P_{K \mid X, \bar{U}^2}(k \mid x, u^2) = \Pr[\Pi(x) = k \mid \bar{U}^2 = u^2] = \frac{p_{\bar{U}^2}(u^2 \mid \Pi(x) = k) \Pr[\Pi(x) = k]}{p_{\bar{U}^2}(u^2)}.
\end{IEEEeqnarray}
Since $\Pi$ is uniform, $\Pr[\Pi(x) = k] = 1/2$, and the unconditional distribution of $\bar{U}^2$ is the same as that induced by sampling without replacement from the two modes.
As before, assuming $X = 1$ we evaluate
\begin{IEEEeqnarray}{rCl}
    P_{K \mid X, \bar{U}^2}(1 \mid 1, u^2) & = & \frac{\mc{N}(u_1; -1, \sigma^2) \mc{N}(u_2; 1, \sigma^2)}{\mc{N}(u_1; -1, \sigma^2) \mc{N}(u_2; 1, \sigma^2) + \mc{N}(u_1; 1, \sigma^2) \mc{N}(u_2; -1, \sigma^2)} \label{eqn:biawgn_posterior_1} \\
    P_{K \mid X, \bar{U}^2}(2 \mid 1, u^2) & = & \frac{\mc{N}(u_1; 1, \sigma^2) \mc{N}(u_2; -1, \sigma^2)}{\mc{N}(u_1; -1, \sigma^2) \mc{N}(u_2; 1, \sigma^2) + \mc{N}(u_1; 1, \sigma^2) \mc{N}(u_2; -1, \sigma^2)}. \label{eqn:biawgn_posterior_2}
\end{IEEEeqnarray}
Substituting \eqref{eqn:biawgn_posterior_1}, \eqref{eqn:biawgn_posterior_2} into the expression for the algorithm's output distribution,
\begin{IEEEeqnarray}{rCl}
    p_{\tilde{Y} \mid X}(y \mid 1) & = & \iint ( \delta(y - u_1) P_{K \mid X, \bar{U}^2}(1 \mid 1, u^2) \nonumber \\
    &&\< + \delta(y - u_2) P_{K \mid X, \bar{U}^2}(2 \mid 1, u^2) ) p_{\bar{U}^2}(u^2) \, du_1 du_2
\end{IEEEeqnarray}
confirms that the permuted approach indeed generates an output following $\mc{N}(-1,\sigma^2)$ when $X = 1$, and by repeating the same steps, $\mc{N}(1,\sigma^2)$ when $X = 2$.
Moreover, as $K$ is no longer uniquely determined by $X$ and also depends on $\bar{U}^2$, the lower bound on the rate is the conditional mutual information $I(X;K \mid \bar{U}^2)$, and this may now be lower than $H(X)$.
We prove this lower bound later in \cref{prop:cost_lower_bound}.
For the BI-AWGN channel in particular, using the posterior distributions in \eqref{eqn:biawgn_posterior_1}, \eqref{eqn:biawgn_posterior_2} and computing the mutual information gives
\begin{IEEEeqnarray}{c}
    I(X; K \mid \bar{U}^2) = 1 - \int_{-\infty}^{\infty} \frac{1}{\sqrt{2\pi}} e^{-z^2 / 2} \log(1 + e^{2\sqrt{2} z / \sigma - 4 / \sigma^2}) \, dz.
\end{IEEEeqnarray}
To give an idea of the relative size of these quantities, we plot $I(X; K \mid \bar{U}^2)$, $I(X; Y)$ and $H(X)$ in \cref{fig:bi_awgn}, for signal-to-noise ratios (SNRs) between $-20$ and $15$, where $\mathrm{SNR} = -10 \log_{10}(\sigma^2)$.
At both low and high SNRs, $I(X; K \mid \bar{U}^2)$ becomes very close to the channel capacity.

\section{Proof of Results in the Main Paper}\label[appendix]{sec:proofs}
\subsection{Proof of Proposition~\ref{prop:correctness}}\label[appendix]{sec:correctness_proof}
Suppose $X = 1$ without loss of generality and let the output be denoted $\tilde{Y}$ throughout this section so as not to confuse its distribution with $P_{Y \mid X}$.
The density function of $\tilde{Y}$ conditioned on $X = 1$ and the shared samples being $\bar{U}^N = u^N$ is
\begin{IEEEeqnarray}{c}
    p_{\tilde{Y} \mid X, \bar{U}^N}(y \mid 1, u^N) = \sum_{k=1}^{N} \delta(y - u_k) \Pr[\Pi(1) = k \mid \bar{U}^N = u^N].
\end{IEEEeqnarray}
Integrating out the joint density of the samples,
\begin{IEEEeqnarray}{c}
    p_{\tilde{Y} \mid X}(y \mid 1) = \idotsint \sum_{k=1}^{N} \delta(y - u_k) \Pr[\Pi(1) = k \mid \bar{U}^N = u^N] p_{\bar{U}^N}(u^N) \, du_1 \cdots du_N. \IEEEeqnarraynumspace
\end{IEEEeqnarray}
Applying Bayes' rule,
\begin{IEEEeqnarray}{c}
    \Pr[\Pi(1) = k \mid \bar{U}^N = u^N] = \frac{p_{\bar{U}^N}(u^N \mid \Pi(1) = k) \Pr[\Pi(1) = k]}{p_{\bar{U}^N}(u^N)} = \frac{p_{\bar{U}^N}(u^N \mid \Pi(1) = k)}{N p_{\bar{U}^N}(u^N)}. \IEEEeqnarraynumspace
\end{IEEEeqnarray}
From this,
\begin{IEEEeqnarray}{rCl}
    p_{\tilde{Y} \mid X}(y \mid 1) & = & \frac{1}{N} \int \cdots \int \sum_{k=1}^{N} \delta(y - u_k) p_{\bar{U}^N}(u^N \mid \Pi(1) = k) \, du_1 \cdots du_N \\
    & = & \frac{1}{N} \sum_{k=1}^{N} \int \cdots \int \delta(y - u_k) p_{\bar{U}^N}(u^N \mid \Pi(1) = k) \, du_1 \cdots du_N.
\end{IEEEeqnarray}
Now, as the analysis of the inner nested integral will be the same for any $k$, let us fix $k = 1$ without loss of generality.
We find
\begin{IEEEeqnarray}{rCl}
    \IEEEeqnarraymulticol{3}{l}{
        \idotsint \delta(y - u_1) p_{\bar{U}^N}(u^N \mid \Pi(1) = 1) \, du_1 \cdots du_N
    }\nonumber\\* \quad
    & = & \int \delta(y - u_1) p_{\bar{U}_1}(u_1 \mid \Pi(1) = 1) \, du_1 \idotsint p_{\bar{U}_2^N}(u_2^N \mid \Pi(1) = 1) \, du_2 \cdots du_N \\
    & = & p_{\bar{U}_1}(y \mid \Pi(1) = 1) \\
    & = & p_{Y \mid X}(y \mid 1).
\end{IEEEeqnarray}
Thus, since the same simplification works for any $k$,
\begin{IEEEeqnarray}{c}
    p_{\tilde{Y} \mid X}(y \mid 1) = \frac{1}{N} \sum_{k=1}^{N} p_{Y \mid X}(y \mid 1) = p_{Y \mid X}(y \mid 1).
\end{IEEEeqnarray}
As the same argument holds for any $X \in [N]$, the output follows the required distribution. \hfill \IEEEQED{}

\subsection{Proof of Theorem~\ref{thm:cost_upper_bound}}\label[appendix]{sec:upper_bound_proof}
We begin by proving the first part of the theorem by bounding the expected size of the index $\bar{K}$ into the sorted list in \cref{alg:permuted_scheme}.
Applying EFR to communicate the sampled index with fixed target and prior distributions $P$ and $Q$ respectively yields the bound $\E[\log \bar{K}] \leq D_{\mathrm{KL}}(P \Vert Q) + 1$ \cite[Lemma~12]{li2024a}.
In our case, given $X = x$ and $\bar{U}^N = u^N$, we have $P(k) = P_{K \mid X, \bar{U}^N}(k \mid x, u^N)$ and $Q(k) = P_{K \mid \bar{U}^N}(k \mid u^N)$.
This gives
\begin{IEEEeqnarray}{c}
    \E[\log \bar{K} \mid X, \bar{U}^N] \leq D_{\mathrm{KL}}(P_{K \mid X, \bar{U}^N}(\cdot \mid X, \bar{U}^N) \Vert P_{K \mid \bar{U}^N}(\cdot \mid \bar{U}^N)) + 1. \label{eqn:li_kl_bound}
\end{IEEEeqnarray}
It then only remains to take the expectation over $X$ and $\bar{U}^N$ in order to bound the unconditional expected size.
In particular,
\begin{IEEEeqnarray}{c}
    \E_{X, \bar{U}^N}[ D_{\mathrm{KL}}(P_{K \mid X, \bar{U}^N}(\cdot \mid X, \bar{U}^N) \Vert P_{K \mid \bar{U}^N}(\cdot \mid \bar{U}^N)) ] =  I(X; K \mid \bar{U}^N).
\end{IEEEeqnarray}
Thus,
\begin{IEEEeqnarray}{c}
    \E[\log \bar{K}] = \E_{X, \bar{U}^N}[\E[\log \bar{K} \mid X, \bar{U}^N]] \leq I(X; K \mid \bar{U}^N) + 1 = I(X;Y \mid \bar{U}^N) + 1
\end{IEEEeqnarray}
where the last equality follows as given the set of samples $\bar{U}^N$, the mapping from $K$ to $Y$ is a bijection with probability one; as our setting assumes a continuous output distribution, ties within $\bar{U}^N$ occur with probability zero.

Concerning the second part of the theorem which deals with the expected message length, we make use of the following result which bounds the cross-entropy with a carefully-chosen Zipf distribution.

\begin{proposition}[{\citealp[Proposition~4]{li2018}}]
For a random variable $Z$ following the distribution $P_Z$ and satisfying $\E[\log Z] \leq \ell$, its cross-entropy with $\operatorname{Zipf}(1 + 1/\ell)$ satisfies
\begin{IEEEeqnarray}{c}
    H(Z) \leq H(P_Z, \operatorname{Zipf}(1 + 1 / \ell)) \leq \ell + \log(\ell + 1) + 1.
\end{IEEEeqnarray}
\end{proposition}

As a result, if we use a Shannon code designed for $\operatorname{Zipf}(1 + 1 / (I(X;Y \mid \bar{U}^N) + 1))$ as the $\mathtt{enc}$, $\mathtt{dec}$ pair in \cref{alg:permuted_scheme}, it satisfies the upper bound on the encoding length in \eqref{eqn:coding_length_bound}. \hfill \IEEEQED{}

\subsection{Proof of Proposition~\ref{prop:multi_sample_correctness}}\label[appendix]{sec:multi_sample_correctness_proof}
In this section, we prove that the multi-sample permuted scheme performs exact channel simulation.
The approach is essentially the same as that used in \cref{sec:correctness_proof} apart from the slight modification to the posterior distribution over $K_m$, which accounts for the presence of several samples per mode.
Supposing $X = 1$ without loss of generality and letting the sample output by the algorithm be $\tilde{Y}$, the conditional density function of $\tilde{Y}$ is
\begin{IEEEeqnarray}{rCl}
    p_{\tilde{Y} \mid X, \bar{U}^{Nm}}(y \mid 1, u^{Nm}) & = & \sum_{k=1}^{Nm} \delta(y - u_k) P_{K_m \mid X, \bar{U}^{Nm}}(k \mid 1, u^{Nm}) \\
    & = & \frac{1}{m} \sum_{k=1}^{Nm} \delta(y - u_k) \Pr[k \in \Pi(1) \mid \bar{U}^{Nm} = u^{Nm}].
\end{IEEEeqnarray}
Then, as in the proof of \cref{prop:correctness}, we integrate out the joint density of the $Nm$ samples to get
\begin{IEEEeqnarray}{c}
    p_{\tilde{Y} \mid X}(y \mid 1) = \idotsint \frac{1}{m} \sum_{k=1}^{Nm} \delta(y - u_k) \Pr[k \in \Pi(1) \mid \bar{U}^{Nm} = u^{Nm}] \, du_1 \cdots du_{Nm} \IEEEeqnarraynumspace
\end{IEEEeqnarray}
and apply Bayes' rule to write
\begin{IEEEeqnarray}{c}
    \Pr[k \in \Pi(1) \mid \bar{U}^{Nm} = u^{Nm}] = \frac{p_{\bar{U}^{Nm}}(u^{Nm} \mid k \in \Pi(1))}{N p_{\bar{U}^{Nm}}(u^{Nm})}
\end{IEEEeqnarray}
in order to make the substitution
\begin{IEEEeqnarray}{c}
    p_{\tilde{Y} \mid X}(y \mid 1) = \frac{1}{Nm} \sum_{k=1}^{Nm} \idotsint \delta(y - u_k) p_{\bar{U}^{Nm}}(u^{Nm} \mid k \in \Pi(1)) \, du_1 \cdots du_{Nm}.
\end{IEEEeqnarray}
Analyzing the inner integral when $k = 1$ without loss of generality,
\begin{IEEEeqnarray}{rCl}
    \IEEEeqnarraymulticol{3}{l}{
        \idotsint \delta(y - u_1) p_{\bar{U}^{Nm}}(u^{Nm} \mid 1 \in \Pi(1)) \, du_1 \cdots du_{Nm}
    }\nonumber\\* \quad
    & = & \int \delta(y - u_1) p_{\bar{U}_1}(u_1 \mid 1 \in \Pi(1)) \, du_1 \nonumber \\
    &&\< \times \idotsint p_{\bar{U}^{Nm \setminus \{ 1 \}}}(u^{Nm \setminus \{ 1 \}} \mid 1 \in \Pi(1)) \, du_2 \cdots du_{Nm} \\
    & = & p_{\bar{U}_1}(y \mid 1 \in \Pi(1)) \\
    & = & p_{Y \mid X}(y \mid 1)
\end{IEEEeqnarray}
where $\bar{U}^{Nm \setminus \{ 1 \}}$ denotes the set of all samples excluding $\bar{U}_1$.
As the same simplification holds for any $k$, we find
\begin{IEEEeqnarray}{c}
    p_{\tilde{Y} \mid X}(y \mid 1) = \frac{1}{Nm} \sum_{k=1}^{Nm} p_{Y \mid X}(y \mid 1) = p_{Y \mid X}(y \mid 1).
\end{IEEEeqnarray}
The same argument holds for any input $X \in [N]$. \hfill \IEEEQED{}

\subsection{Proof of Theorem~\ref{thm:asymptotic}}\label[appendix]{sec:asymptotic_proof}
As mentioned in the theorem statement, we limit our attention to the case of $N = 2$.
Following the steps used to prove \cref{thm:cost_upper_bound} in \cref{sec:upper_bound_proof}, with the target distribution given $X = x$ and $\bar{U}^{2m} = u^{2m}$ being $P(k) = P_{K_m \mid X, \bar{U}^{2m}}(k \mid x, u^{2m})$ as in \eqref{eqn:multisample_target_distribution}, and the prior being $Q(k) = P_{K_m \mid \bar{U}^{2m}}(k \mid u^{2m})$, the expected message length when there are $m$ samples per mode satisfies
\begin{IEEEeqnarray}{c}
    \E[\ell_m(M)] \leq I(X; K_m \mid \bar{U}^{2m}) + \log(I(X; K_m \mid \bar{U}^{2m}) + 2) + 3. \label{eqn:M_length_ineq}
\end{IEEEeqnarray}
Expanding, we note that
\begin{IEEEeqnarray}{c}
    I(X; K_m \mid \bar{U}^{2m}) = H(X) - H(X \mid K_m, \bar{U}^{2m}) \label{eqn:mut_inf_expansion}
\end{IEEEeqnarray}
as $X$ is independent of $\bar{U}^{2m}$.
Then, the term depending on $m$ is
\begin{IEEEeqnarray}{c}
    H(X \mid K_m, \bar{U}^{2m}) = \E[ h_B(p_m(K_m, \bar{U}^{2m})) ]
\end{IEEEeqnarray}
where $h_B(p) = -p \log p - (1 - p) \log(1 - p)$ is the binary entropy function on $p \in [0,1]$ and $p_m(k, u^{2m}) = P_{X \mid K_m, \bar{U}^{2m}}(1 \mid k, u^{2m})$.
Also define $q_m(k, u^{2m}) = P_{X \mid Y}(1 \mid u_k)$.
To continue, we make use of the following lemma, the proof of which is deferred to \cref{sec:posterior_limit_proof}.

\begin{lemma}\label{lem:posterior_limit}
We have that
\begin{IEEEeqnarray}{c}
    \lim_{m \to \infty} \abs{p_m(K_m, \bar{U}^{2m}) - q_m(K_m, \bar{U}^{2m})} = 0 \text{ a.s.}
\end{IEEEeqnarray}
\end{lemma}

As $h_B(p)$ is uniformly continuous on $[0, 1]$, \cref{lem:posterior_limit} implies that
\begin{IEEEeqnarray}{c}
    \lim_{m \to \infty} \abs{h_B(p_m(K_m, \bar{U}^{2m})) - h_B(q_m(K_m, \bar{U}^{2m}))} = 0 \text{ a.s.}
\end{IEEEeqnarray}
Furthermore, as $h_B(p)$ is bounded between 0 and 1, the difference is also bounded.
Applying the bounded convergence theorem then gives
\begin{IEEEeqnarray}{rCl}
    \IEEEeqnarraymulticol{3}{l}{
        \lim_{m \to \infty} \E[h_B(p_m(K_m, \bar{U}^{2m})) - h_B(q_m(K_m, \bar{U}^{2m}))]
    }\nonumber\\* \quad
    & = & \E\Bigl[ \lim_{m \to \infty} ( h_B(p_m(K_m, \bar{U}^{2m})) - h_B(q_m(K_m, \bar{U}^{2m})) ) \Bigr] \\
    & = & 0.
\end{IEEEeqnarray}
But, for arbitrary $m$,
\begin{IEEEeqnarray}{rCl}
    \E[h_B(q_m(K_m, \bar{U}^{2m}))] & = & \E_{K_m, \bar{U}^{2m}}[h_B(P_{X \mid Y}(1 \mid \bar{U}_{K_m}))] \\
    & = & \E_{\bar{U}_{K_m}}[h_B(P_{X \mid Y}(1 \mid \bar{U}_{K_m}))] \\
    & = & \E_{Y \sim P_Y}[h_B(P_{X \mid Y}(1 \mid Y))] \label{eqn:uk_to_y_swap} \\ 
    & = & H(X \mid Y)
\end{IEEEeqnarray}
where in \eqref{eqn:uk_to_y_swap} we have used the fact that, as $\bar{U}_{K_m}$ is the output selected by the multi-sample permuted scheme when there are $m$ samples per mode, the guarantee of exact channel simulation from \cref{prop:multi_sample_correctness} ensures that $\bar{U}_{K_m} \sim P_Y$ for any $m$.
It follows that
\begin{IEEEeqnarray}{c}
    \lim_{m \to \infty} \E[h_B(p_m(K_m, \bar{U}^{2m}))] = H(X \mid Y).
\end{IEEEeqnarray}
Finally, substituting into \eqref{eqn:mut_inf_expansion}, then \eqref{eqn:M_length_ineq}, and using continuity, we arrive at the desired result. \hfill \IEEEQED{}

\subsection{Proof of Lemma~\ref{lem:posterior_limit}}\label[appendix]{sec:posterior_limit_proof}
We start by dispensing with a trivial case.
If the intersection of the supports of the conditional densities at the two modes, $p_{Y \mid X}(\cdot \mid 1)$ and $p_{Y \mid X}(\cdot \mid 2)$, has zero probability, then the output $Y = \bar{U}_{K_m}$ immediately determines the input $X$.
In this case,
\begin{IEEEeqnarray}{c}
    P_{X \mid K_m, \bar{U}^{2m}}(1 \mid K_m, \bar{U}^{2m}) = P_{X \mid Y}(1 \mid \bar{U}_{K_m}) = \1\{ \bar{U}_{K_m} \in \operatorname{supp} p_{Y \mid X}(\cdot \mid 1) \}
\end{IEEEeqnarray}
and the required equality holds trivially for all $m$.

Otherwise, we start by evaluating the pointwise posterior probability for a given $k \in [2m]$, and shared randomness realization $\bar{U}^{2m} = u^{2m}$.
Using Bayes' rule,
\begin{IEEEeqnarray}{c}
    P_{X \mid K_m, \bar{U}^{2m}}(1 \mid k, u^{2m}) = \frac{P_{K_m \mid X, \bar{U}^{2m}}(k \mid 1, u^{2m}) P_X(1)}{\sum_{x=1}^{2} P_{K_m \mid X, \bar{U}^{2m}}(k \mid x, u^{2m}) P_X(x)}.
\end{IEEEeqnarray}
Substituting the posterior distribution as defined in \eqref{eqn:multisample_target_distribution}, we get
\begin{IEEEeqnarray}{c}
    P_{X \mid K_m, \bar{U}^{2m}}(1 \mid k, u^{2m}) = \frac{\Pr[k \in \Pi(1) \mid \bar{U}^{2m} = u^{2m}] P_X(1)}{\sum_{x=1}^{2} \Pr[k \in \Pi(x) \mid \bar{U}^{2m} = u^{2m}] P_X(x)}. \label{eqn:expanded_posterior_1}
\end{IEEEeqnarray}
We next turn our attention to computing $\Pr[k \in \Pi(x) \mid \bar{U}^{2m} = u^{2m}]$ which, as detailed in \cref{sec:permuted_scheme_impl} for the single-sample permuted scheme, involves the computation of likelihood matrix permanents.
To make this precise, by applying Bayes' rule a second time after noting that $\Pr[k \in \Pi(x)] = 1/2$ for any $k$ and $x$, we see that
\begin{IEEEeqnarray}{c}
    \Pr[k \in \Pi(x) \mid \bar{U}^{2m} = u^{2m}] = \frac{p_{\bar{U}^{2m}}(u^{2m} \mid k \in \Pi(x))}{\sum_{x'=1}^{2} p_{\bar{U}^{2m}}(u^{2m} \mid k \in \Pi(x'))}.
\end{IEEEeqnarray}
As in \cref{sec:multi_sample_correctness_proof}, let $\bar{U}^{2m \setminus \{ k \}}$ denote the set of all samples excluding $\bar{U}_k$.
We can then express the conditional density function as
\begin{IEEEeqnarray}{c}
    p_{\bar{U}^{2m}}(u^{2m} \mid k \in \Pi(x)) = p_{Y \mid X}(u_k \mid x) p_{\bar{U}^{2m \setminus \{ k \}}}(u^{2m \setminus \{ k \}} \mid k \in \Pi(x)).
\end{IEEEeqnarray}
Going back to \eqref{eqn:expanded_posterior_1}, we get
\begin{IEEEeqnarray}{rCl}
    P_{X \mid K_m, \bar{U}^{2m}}(1 \mid k, u^{2m}) & = & \frac{p_{\bar{U}^{2m}}(u^{2m} \mid k \in \Pi(1)) P_X(1)}{\sum_{x=1}^{2} p_{\bar{U}^{2m}}(u^{2m} \mid k \in \Pi(x)) P_X(x)} \\
    & = & \frac{p_{Y \mid X}(u_k \mid 1) p_{\bar{U}^{2m \setminus \{ k \}}}(u^{2m \setminus \{ k \}} \mid k \in \Pi(1)) P_X(1)}{\sum_{x=1}^{2} p_{Y \mid X}(u_k \mid x) p_{\bar{U}^{2m \setminus \{ k \}}}(u^{2m \setminus \{ k \}} \mid k \in \Pi(x)) P_X(x)}. \IEEEeqnarraynumspace \label{eqn:expanded_posterior}
\end{IEEEeqnarray}
Next define the likelihood ratio of the rejected samples as
\begin{IEEEeqnarray}{c}
    \Lambda_m(k, u^{2m}) = \frac{p_{\bar{U}^{2m \setminus \{ k \}}}(u^{2m \setminus \{ k \}} \mid k \in \Pi(1))}{p_{\bar{U}^{2m \setminus \{ k \}}}(u^{2m \setminus \{ k \}} \mid k \in \Pi(2))}.
\end{IEEEeqnarray}
To analyze this ratio, we introduce the $2m \times 2m$ likelihood matrix $L(u^{2m})$ where $L_{ij} = p_{Y \mid X}(u_i \mid 1)$ for $1 \leq j \leq m$, and $p_{Y \mid X}(u_i \mid 2)$ for $m+1 \leq j \leq 2m$.
This matrix contains only two unique columns, let these be $z^{(1)}$ and $z^{(2)}$, with 
\begin{IEEEeqnarray}{c}
    z^{(1)} = \begin{bmatrix} p_{Y \mid X}(u_1 \mid 1) \\ \vdots \\ p_{Y \mid X}(u_{2m} \mid 1) \end{bmatrix} \text{ and } z^{(2)} = \begin{bmatrix} p_{Y \mid X}(u_1 \mid 2) \\ \vdots \\ p_{Y \mid X}(u_{2m} \mid 2) \end{bmatrix}. \label{eqn:prob_matrices}
\end{IEEEeqnarray}
With this established, note that
\begin{enumerate}
    \item $p_{\bar{U}^{2m \setminus \{ k \}}}(u^{2m \setminus \{ k \}} \mid k \in \Pi(1))$ is proportional to the permanent of the sub-matrix with row $k$ and one of the columns equal to $z^{(1)}$ removed.
    \item $p_{\bar{U}^{2m \setminus \{ k \}}}(u^{2m \setminus \{ k \}} \mid k \in \Pi(2))$ is proportional to the permanent of the sub-matrix with row $k$ and one of the columns equal to $z^{(2)}$ removed.
\end{enumerate}
Since the matrix permanent is a polynomial of the elements of the matrix, we can also express the desired probabilities in terms of a polynomial.
Define
\begin{IEEEeqnarray}{c}
    P_k(w) = \prod_{i \in [2m] \setminus \{ k \}} (w z_i^{(1)} + z_i^{(2)}) = \sum_{j=0}^{2m-1} C_j^{(k)} w^j.
\end{IEEEeqnarray}
Each term in $P_k(w)$ represents the joint likelihood of a given assignment of all $2m-1$ remaining samples between the two modes; $p_{\bar{U}^{2m \setminus \{ k \}}}(u^{2m \setminus \{ k \}} \mid k \in \Pi(1))$ is proportional to the sum of all those which correspond to selecting $m-1$ likelihoods from $z^{(1)}$ and $m$ from $z^{(2)}$.
This sum is the coefficient $C_{m-1}^{(k)}$ in the expanded polynomial.
Similarly, $p_{\bar{U}^{2m \setminus \{ k \}}}(u^{2m \setminus \{ k \}} \mid k \in \Pi(2))$ is proportional to $C_m^{(k)}$, with the same constant of proportionality.
Thus, the probability ratio becomes $\Lambda_m(k, u^{2m}) = C_{m-1}^{(k)} / C_{m}^{(k)}$.

To continue, we normalize each term of the polynomial to get
\begin{IEEEeqnarray}{c}
    \hat{P}_k(w) = \prod_{i \in [2m] \setminus \{ k \}} (w p_i + (1 - p_i)) = \sum_{j=0}^{2m-1} \hat{C}_j^{(k)} w^j
\end{IEEEeqnarray}
where
\begin{IEEEeqnarray}{c}
    p_i = \frac{z_i^{(1)}}{z_i^{(1)} + z_i^{(2)}} \text{ and } 1 - p_i = \frac{z_i^{(2)}}{z_i^{(1)} + z_i^{(2)}}. \label{eqn:p_def}
\end{IEEEeqnarray}
We note that when the $p_i$'s are derived from our random sample pool $\bar{U}^{2m}$, the event $\{ z_i^{(1)} = z_i^{(2)} = 0 \}$ has probability zero, since each sample is sourced from one of the modes and has a strictly positive density evaluated under the mode that produced it.
Furthermore, as the intersection of the supports of $p_{Y \mid X}(\cdot \mid 1)$ and $p_{Y \mid X}(\cdot \mid 2)$ has nonzero probability outside of the trivial case already dealt with, we will have $\Pr[0 < p_i < 1] > 0$.

Now, we can analyze $\hat{P}_k(w)$ by recognizing that it is the probability generating function of a Poisson binomial random variable with $n = 2m-1$ trials \citep{tang2023}.
In particular, for $1 \leq i \leq 2m$, let $X_i$ be a Bernoulli random variable with success probability $p_i$, and let
\begin{IEEEeqnarray}{c}
    S_k(u^{2m}) = \sum_{i \in [2m] \setminus \{ k \}} X_i. \label{eqn:sk_def}
\end{IEEEeqnarray}
Then $S_k(u^{2m})$ is a Poisson binomial random variable for each $k$, and its probability generating function is $\hat{P}_k(w)$.
For $0 \leq j \leq n$, we have that $\hat{C}_j^{(k)} = \Pr[S_k(u^{2m}) = j]$ and so
\begin{IEEEeqnarray}{c}
    \Lambda_m(k, u^{2m}) = \frac{\Pr[S_k(u^{2m}) = m-1]}{\Pr[S_k(u^{2m}) = m]}. \label{eqn:R_2M_indiv}
\end{IEEEeqnarray}

While \eqref{eqn:R_2M_indiv} holds for a particular realization of $(K_m, \bar{U}^{2m})$, we are interested in the ensemble behavior when the full pool of $2m$ shared random samples is generated following $P_{\bar{U}^{2m}}$ and the index $K_m$ is also random.
The following lemma establishes the asymptotic properties that will allow us to adapt our results to this context.

\begin{lemma}\label{lem:prob_ratio_convergence}
Assume $\Pr[0 < p_i < 1] > 0$. Then,
\begin{IEEEeqnarray}{c}
    \adjustlimits \lim_{m \to \infty} \sup_{1 \leq k \leq 2m} \abs*{\frac{\Pr[S_k(\bar{U}^{2m}) = m - 1]}{\Pr[S_k(\bar{U}^{2m}) = m]} - 1} = 0 \text{ a.s.}
\end{IEEEeqnarray}
\end{lemma}
The uniform convergence result from \cref{lem:prob_ratio_convergence} implies that $\lim_{m \to \infty} \Lambda_m(K_m, \bar{U}^{2m}) = 1$ almost surely, and we are interested in computing $P_{X \mid K_m, \bar{U}^{2m}}(1 \mid K_m, \bar{U}^{2m})$ in the limit.
Consider an arbitrary sample path $\{ u^{2m}, k_m \}_{m \geq 1}$ for which $\lim_{m \to \infty} \Lambda_m(k_m, u^{2m}) = 1$, and define
\begin{IEEEeqnarray}{rCl}
    p_m & = & p_m(k_m, u^{2m}) = P_{X \mid K_m, \bar{U}^{2m}}(1 \mid k_m, u^{2m}) \\
    q_m & = & q_m(k_m, u^{2m}) = P_{X \mid Y}(1 \mid u_{k_m}) \\
    a_m & = & p_{Y \mid X}(u_{k_m} \mid 1) P_X(1) \\
    b_m & = & p_{Y \mid X}(u_{k_m} \mid 2) P_X(2).
\end{IEEEeqnarray}
We also drop the arguments of $\Lambda_m(k_m, u^{2m})$.
Using \eqref{eqn:expanded_posterior} and applying Bayes' rule to $q_m$, we get
\begin{IEEEeqnarray}{rCl}
    0 & \leq & \lim_{m \to \infty} \abs{p_m - q_m} \\
    & = & \lim_{m \to \infty} \abs*{ \frac{a_m \Lambda_m}{a_m \Lambda_m + b_m} - \frac{a_m}{a_m + b_m} } \\
    & = & \lim_{m \to \infty} \abs*{ \frac{q_m \Lambda_m}{q_m \Lambda_m + 1 - q_m} - q_m } \\
    & = & \lim_{m \to \infty} \frac{q_m (1 - q_m) \abs{\Lambda_m - 1}}{1 + q_m(\Lambda_m - 1)} \\
    & \leq & \lim_{m \to \infty} \frac{1}{2} \abs{\Lambda_m - 1} \label{eqn:final_limit_simplification}
\end{IEEEeqnarray}
where in \eqref{eqn:final_limit_simplification} we have used the fact that as $\Lambda_m \to 1$, there exists some $M_0$ for which $\Lambda_m \geq 1/2$ for all $m > M_0$, and therefore for which the denominator $1 + q_m(\Lambda_m - 1) \geq 1/2$.
Also, $q_m (1 - q_m) \leq 1/4$.
Taking the limit of $\Lambda_m$ and noting that the argument holds, by \cref{lem:prob_ratio_convergence}, on a set of sample paths of probability one, we complete the proof. \hfill \IEEEQED{}

\subsection{Proof of Lemma~\ref{lem:prob_ratio_convergence}}\label[appendix]{sec:prob_ratio_convergence_proof}
The main difficulty in proving \cref{lem:prob_ratio_convergence} is that the success probabilities $p_i$ making up the Poisson binomial distribution in question are not fixed, but are instead functions of the random variables $\bar{U}^{2m}$, which themselves are not independent.
Consequently, there are two distinct layers of randomness to deal with. 
Our first step is to recognize that, from \eqref{eqn:p_def}, $p_i$ is a function of $\bar{U}_i$ only, and we can group $\bar{U}^{2m}$ into $m$ blocks $\bar{U}^2_1,\ldots,\bar{U}_{m}^2$, each of which is indeed iid from $P_{\bar{U}^2}$.
Here, $P_{\bar{U}^2}$ is the distribution of two samples where exactly one is from $P_{Y \mid X}(\cdot \mid 1)$ and the other is from $P_{Y \mid X}(\cdot \mid 2)$.
Let $\mc{S}_j$ be the set of two indices corresponding to locations of the samples forming block $\bar{U}^2_j$, where $j \in [m]$.
With this structure established, we present some results downstream of the $p_i$'s which will be useful later.

\begin{lemma}\label{lem:outer_randomness}
Define
\begin{IEEEeqnarray}{c}
    Y_j = \sum_{i \in \mc{S}_j} \E[X_i \mid \bar{U}_i] \text{ and } V_j = \sum_{i \in \mc{S}_j} \Var(X_i \mid \bar{U}_i) \label{eqn:Y_V_def}
\end{IEEEeqnarray}
and note that $\E[X_i \mid \bar{U}_i] = p_i$ and $\Var(X_i \mid \bar{U}_i) = p_i (1 - p_i)$.
Then, the following hold almost surely:
\begin{enumerate}
    \item $\displaystyle \limsup_{m \to \infty} \frac{\bigl| \sum_{j=1}^{m} Y_j - m \bigr|}{\sqrt{2m \ln\ln m}} = \sigma_Y$, where $\sigma_Y^2 = \Var(Y_j)$.
    \item $\displaystyle \lim_{m \to \infty} \frac{\sum_{j=1}^{m} V_j}{m} = \mu_V$, where $\mu_V = \E[V_j]$.
\end{enumerate}
\end{lemma}
\begin{IEEEproof}
Note that $Y_j$ and $V_j$ are both functions of $\bar{U}_j^2$ only, and as mentioned previously, $\bar{U}_1^2, \ldots, \bar{U}_m^2$ are iid draws from $P_{\bar{U}^2}$.
Moreover, the $X_i$'s are Bernoulli random variables which can only take values in $\{ 0, 1 \}$. 
Consequently, from the definitions in \eqref{eqn:Y_V_def}, $Y^m$ and $V^m$ are iid sequences of bounded random variables.

Expanding $\mc{S}_j = \{ j_1, j_2\}$, we find the mean of $Y_j$ to be
\begin{IEEEeqnarray}{c}
    \E[Y_j] = \E \biggl[ \frac{p_{Y \mid X}(\bar{U}_{j_1} \mid 1)}{p_{Y \mid X}(\bar{U}_{j_1} \mid 1) + p_{Y \mid X}(\bar{U}_{j_1} \mid 2)} + \frac{p_{Y \mid X}(\bar{U}_{j_2} \mid 1)}{p_{Y \mid X}(\bar{U}_{j_2} \mid 1) + p_{Y \mid X}(\bar{U}_{j_2} \mid 2)} \biggr].
\end{IEEEeqnarray}
As dictated by the common randomness structure, exactly one of $\bar{U}_{j_1}$ and $\bar{U}_{j_2}$ is sampled from $P_{Y \mid X}(\cdot \mid 1)$ and the other is sampled from $P_{Y \mid X}(\cdot \mid 2)$.
Assume without loss of generality that $\bar{U}_{j_1} \sim P_{Y \mid X}(\cdot \mid 1)$ and $\bar{U}_{j_2} \sim P_{Y \mid X}(\cdot \mid 2)$.
We have
\begin{IEEEeqnarray}{rCl}
    \E[Y_j] & = & \int \frac{p_{Y \mid X}(u \mid 1)^2}{p_{Y \mid X}(u \mid 1) + p_{Y \mid X}(u \mid 2)} \, du + \int \frac{p_{Y \mid X}(u \mid 1) p_{Y \mid X}(u \mid 2)}{p_{Y \mid X}(u \mid 1) + p_{Y \mid X}(u \mid 2)} \, du \\
    & = & \int \frac{p_{Y \mid X}(u \mid 1) (p_{Y \mid X}(u \mid 1) + p_{Y \mid X}(u \mid 2))}{p_{Y \mid X}(u \mid 1) + p_{Y \mid X}(u \mid 2)} \, du \\
    & = & \int p_{Y \mid X}(u \mid 1) \, du \\
    & = & 1.
\end{IEEEeqnarray}
Thus $\E[\sum_{j=1}^{m} Y_j] = m$, and applying the law of the iterated logarithm and the strong law of large numbers yields the desired results.
\end{IEEEproof}

Now we turn to analyzing the Poisson binomial random variable $S_k(\bar{U}^{2m})$ itself.
Define
\begin{IEEEeqnarray}{rCl}
    \mu_m^{(k)} & = & \sum_{i \in [2m] \setminus \{ k \}} \E[X_i \mid \bar{U}_i] = \sum_{j=1}^{m} Y_j - p_k \label{eqn:poisson_binom_def_1} \\
    B_m^{(k)} & = & \sum_{i \in [2m] \setminus \{ k \}} \Var(X_i \mid \bar{U}_i) = \sum_{j=1}^{m} V_j - p_k (1 - p_k). \label{eqn:poisson_binom_def_2}
\end{IEEEeqnarray}
As a corollary to \cref{lem:outer_randomness}, we have the following.

\begin{corollary}\label{cor:outer_randomness_extra}
Define
\begin{IEEEeqnarray}{c}
    R_m^\ast = \sup_{1 \leq k \leq 2m} \abs{\mu_m^{(k)} - m} \text{ and } B_m^\ast = \inf_{1 \leq k \leq 2m} B_m^{(k)}. \label{eqn:inf_sup_def}
\end{IEEEeqnarray}
Then, the following hold almost surely:
\begin{enumerate}
    \item $\displaystyle \limsup_{m \to \infty} \frac{R_m^\ast}{\sqrt{2m \ln\ln m}} = \sigma_Y$.
    \item $\displaystyle \lim_{m \to \infty} \frac{B_m^\ast}{m} = \mu_V$.
\end{enumerate}
\end{corollary}
\begin{IEEEproof}
For the first part, using \eqref{eqn:poisson_binom_def_1} and the triangle inequality, for all $k$,
\begin{IEEEeqnarray}{rCl}
    \IEEEeqnarraymulticol{3}{l}{
        \biggl| \sum_{j=1}^{m} Y_j - m \biggr| - p_k \leq \abs{\mu_m^{(k)} - m} \leq \biggl| \sum_{j=1}^{m} Y_j - m \biggr| + p_k
    }\nonumber\\* \quad
    & \implies & \biggl| \sum_{j=1}^{m} Y_j - m \biggr| - 1 \leq R_m^\ast \leq \biggl| \sum_{j=1}^{m} Y_j - m \biggr| + 1
\end{IEEEeqnarray}
as $0 \leq p_k \leq 1$.
Therefore,
\begin{IEEEeqnarray}{c}
    \limsup_{m \to \infty} \frac{\bigl| \sum_{j=1}^{m} Y_j - m \bigr| - 1}{\sqrt{2m \ln\ln m}} \leq \limsup_{m \to \infty} \frac{R_m^\ast}{\sqrt{2m \ln\ln m}} \leq \limsup_{m \to \infty} \frac{\bigl| \sum_{j=1}^{m} Y_j - m \bigr| + 1}{\sqrt{2m \ln\ln m}} \IEEEeqnarraynumspace
\end{IEEEeqnarray}
after which the desired result follows from \cref{lem:outer_randomness}.

Similarly for the second part, using \eqref{eqn:poisson_binom_def_2} and as $0 \leq p_k(1 - p_k) \leq 1/4$ for all $k$,
\begin{IEEEeqnarray}{c}
    \sum_{j=1}^{m} V_j - \frac{1}{4} \leq B_m^\ast \leq \sum_{j=1}^{m} V_j.
\end{IEEEeqnarray}
From this,
\begin{IEEEeqnarray}{c}
    \lim_{m \to \infty} \frac{\sum_{j=1}^{m} V_j - 1/4}{m} \leq \lim_{m \to \infty} \frac{B_m^\ast}{m} \leq \lim_{m \to \infty} \frac{\sum_{j=1}^{m} V_j}{m}.
\end{IEEEeqnarray}
As before, the result is then obtained by applying \cref{lem:outer_randomness}.
\end{IEEEproof}

\begin{remark}
The condition $\Pr[0 < p_i < 1] > 0$ ensures that the probability of each Bernoulli trial being deterministic is strictly less than one.
Concretely, this gives us $\mu_V > 0$, and \cref{cor:outer_randomness_extra} then implies that $B_m^\ast \to \infty$ almost surely.
\end{remark}

In what follows, we focus on analyzing the convergence of the Poisson binomial probability ratio for a generic sample path of the outer randomness $\{ u^{2m} \}_{m \geq 1}$ on which the properties claimed in \cref{lem:outer_randomness,cor:outer_randomness_extra} hold.
As stated there, the set of such sample paths has probability one under the law of $\{ \bar{U}^{2m} \}_{m \geq 1}$.
Establishing the desired result for $\{ u^{2m} \}_{m \geq 1}$ is then enough to show almost sure convergence.
Continuing requires the following local limit theorem for sums of independent binary random variables.

\begin{theorem}[{\citealp{mamatov1965}}]\label{thm:mamatov}
Let $W^n$ be a sequence of independent random variables taking on values 0 and 1.
Define
\begin{IEEEeqnarray}{c}
    q_i = \Pr[W_i = 1], \ T_n = \sum_{i=1}^{n} W_i, \  a_n = \sum_{i=1}^{n} \E[W_i], \ b_n = \sum_{i=1}^{n} \Var(W_i), \IEEEeqnarraynumspace \\
    L_n = b_n^{-3/2} \sum_{i=1}^{n} q_i (1 - q_i) (q_i^2 + (1 - q_i)^2).
\end{IEEEeqnarray}
Suppose $0 < q_i < 1$, then there exists an absolute constant $C$ such that
\begin{IEEEeqnarray}{c}
    \abs*{\sqrt{b_n} \Pr[T_n = j] - \frac{1}{\sqrt{2 \pi}} e^{-(j - a_n)^2 / 2 b_n}} < CL_n.
\end{IEEEeqnarray}
\end{theorem}

Given our realization $\bar{U}^{2m} = u^{2m}$, let $\mc{I}_m \subseteq [2m]$ be the set of indices for which $0 < p_i < 1$.
The zero-variance trials at a given $k$ and $m$, which are $\{ X_i \}_{i \in [2m] \setminus (\mc{I}_m \cup \{ k \})}$, are all deterministic and contribute a constant total of
\begin{IEEEeqnarray}{c}
    G_m^{(k)} = \sum_{i \in [2m] \setminus (\mc{I}_m \cup \{ k \})} p_i
\end{IEEEeqnarray}
to $S_k(u^{2m})$.
To use \cref{thm:mamatov}, we therefore define the restricted sum
\begin{IEEEeqnarray}{c}{}
    \tilde{S}_k(u^{2m}) = \sum_{i \in \mc{I}_m \setminus \{ k \}} X_i = S_k(u^{2m}) - G_m^{(k)}.
\end{IEEEeqnarray}
The sums of the means and variances for the new sequence, with reference to \eqref{eqn:poisson_binom_def_1} and \eqref{eqn:poisson_binom_def_2}, are respectively
\begin{IEEEeqnarray}{rCl}
    \tilde{\mu}_m^{(k)} & = & \sum_{i \in \mc{I}_m \setminus \{ k \}} \E[X_i \mid \bar{U}_i = u_i] = \mu_m^{(k)} - G_m^{(k)} \\
    \tilde{B}_m^{(k)} & = & \sum_{i \in \mc{I}_m \setminus \{ k \}} \Var(X_i \mid \bar{U}_i = u_i) = B_m^{(k)}.
\end{IEEEeqnarray}
The support of $\tilde{S}_k(u^{2m})$ is $\{ 0, 1, \ldots, r_m^{(k)} \}$ where $r_m^{(k)} = \abs{\mc{I}_m \setminus \{ k \}}$.
As we wish to evaluate the \cref{thm:mamatov} bound for $\tilde{S}_k(u^{2m})$ at the points $j_1 = m - G_m^{(k)}$ and $j_2 = m - 1 - G_m^{(k)}$, we first check that for large enough $m$, these points fall inside the support, clearing the way for a straightforward application of the theorem.
First, since $B_m^{(k)} = \sum_{i \in \mc{I}_m \setminus \{ k \}} p_i(1 - p_i)$ we have
\begin{IEEEeqnarray}{rCl}
    B_m^{(k)} & \leq & \sum_{i \in \mc{I}_m \setminus \{ k \}} p_i = \tilde{\mu}_m^{(k)} \\
    B_m^{(k)} & \leq & \sum_{i \in \mc{I}_m \setminus \{ k \}} (1 - p_i) = r_m^{(k)} - \tilde{\mu}_m^{(k)}.
\end{IEEEeqnarray}
Therefore $B_m^{(k)} \leq \min\{ \tilde{\mu}_m^{(k)}, r_m^{(k)} - \tilde{\mu}_m^{(k)} \}$.
In other words, the distance of the restricted sum's mean $\tilde{\mu}_m^{(k)}$ from the endpoints of its support is at least $B_m^{(k)}$, and this holds for any $m$.
Concerning our specific evaluation points,
\begin{IEEEeqnarray}{c}
    j_1 - \tilde{\mu}_m^{(k)} = m - \mu_m^{(k)} \text{ and } j_2 - \tilde{\mu}_m^{(k)} = m - 1 - \mu_m^{(k)}
\end{IEEEeqnarray}
we have that
\begin{IEEEeqnarray}{c}
    \frac{1}{B_m^\ast} \max\{ \abs{j_1 - \tilde{\mu}_m^{(k)}}, \abs{j_2 - \tilde{\mu}_m^{(k)}} \} \leq \frac{R_m^\ast + 1}{B_m^\ast}.
\end{IEEEeqnarray}
From \cref{cor:outer_randomness_extra}, for any $\varepsilon_1 > 0$, there exists some $M_1$ such that for all $m > M_1$,
\begin{IEEEeqnarray}{c}
    \frac{R_m^\ast}{\sqrt{2 m \ln\ln m}} \leq \sigma_Y + \varepsilon_1 \label{eqn:epsilon-delta_1}
\end{IEEEeqnarray}
and also for any $\varepsilon_2 > 0$ there exists some $M_2$ such that for all $m > M_2$,
\begin{IEEEeqnarray}{c}
    \frac{m}{B_m^\ast} \leq \frac{1}{\mu_V} + \varepsilon_2. \label{eqn:epsilon-delta_2}
\end{IEEEeqnarray}
Thus for any $\varepsilon_1, \varepsilon_2 > 0$ there exists some $M_3$ such that for all $m > M_3$,
\begin{IEEEeqnarray}{rCl}
    \frac{R_m^\ast}{B_m^\ast} = \frac{R_m^\ast}{\sqrt{2m \ln\ln m}} \frac{\sqrt{2 \ln\ln m / m}}{B_m^\ast / m} \leq D(\varepsilon_1, \varepsilon_2) \sqrt{\frac{2 \ln\ln m}{m}} \label{eqn:Rm_over_Bm}
\end{IEEEeqnarray}
where $D(\varepsilon_1, \varepsilon_2) = (\sigma_Y + \varepsilon_1)(\mu_V^{-1} + \varepsilon_2) > 0$.
Thus, fixing $\varepsilon_1$ and $\varepsilon_2$,
\begin{IEEEeqnarray}{c}
    0 \leq \lim_{m \to \infty} \frac{R_m^\ast}{B_m^\ast} \leq \lim_{m \to \infty} D(\varepsilon_1, \varepsilon_2) \sqrt{\frac{2 \ln\ln m}{m}} = 0. \label{eqn:Rm_over_Bm_limit}
\end{IEEEeqnarray}
Combining, and as $B_m^\ast \to \infty$ by \cref{cor:outer_randomness_extra}, we get
\begin{IEEEeqnarray}{c}
    \lim_{m \to \infty} \frac{1}{B_m^\ast} \sup_{1 \leq k \leq 2m} \max\{ \abs{j_1 - \tilde{\mu}_m^{(k)}}, \abs{j_2 - \tilde{\mu}_m^{(k)}} \} = 0.
\end{IEEEeqnarray}
Therefore there exists some $M_4$ such that for all $m > M_4$, the distance of both evaluation points $j_1$ and $j_2$ from the restricted sum's mean is less than $B_m^{(k)}$, and consequently within its support.
In what follows, we assume $m > M_4$ throughout so that \cref{thm:mamatov} can safely be used to analyze the restricted sequence of Bernoulli trials, which is $\{ X_i \}_{i \in \mc{I}_m \setminus \{ k \}}$.

Applying \cref{thm:mamatov} to this restricted sequence gives us for each $m$, $k$ and at our evaluation points $j \in \{ j_1, j_2 \}$,
\begin{IEEEeqnarray}{c}
    \sqrt{B_m^{(k)}} \Pr[\tilde{S}_k(u^{2m}) = j] = A_m^{(k)}(j) + \varepsilon_m^{(k)}(j)
\end{IEEEeqnarray}
where $A_m^{(k)}(j) = (2\pi)^{-1/2} e^{-(j - \mu_m^{(k)} + G_m^{(k)})^2 / 2 B_m^{(k)}}$ and $\abs{\varepsilon_m^{(k)}(j)} \leq C / \sqrt{B_m^{(k)}}$, using the fact that, as $p_i^2 + (1 - p_i)^2 \leq 1$,
\begin{IEEEeqnarray}{c}
    (B_m^{(k)})^{-3/2} \sum_{i \in \mc{I}_m \setminus \{ k \}} p_i (1 - p_i) (p_i^2 + (1 - p_i)^2) \leq (B_m^{(k)})^{-1/2}.
\end{IEEEeqnarray}
The probability ratio of interest can be written as
\begin{IEEEeqnarray}{rCl}
    \frac{\Pr[S_k(u^{2m}) = m-1]}{\Pr[S_k(u^{2m}) = m]} & = & \frac{\Pr[\tilde{S}_k(u^{2m}) = m - 1 - G_m^{(k)}]}{\Pr[\tilde{S}_k(u^{2m}) = m - G_m^{(k)}]} \\
    & = & \frac{A_m^{(k)}(m - 1 - G_m^{(k)}) + \varepsilon_m^{(k)}(m - 1 - G)}{A_M^{(k)}(m - G_m^{(k)}) + \varepsilon_m^{(k)}(m - G_m^{(k)})} \\
    & = & \frac{A_m^{(k)}(m - 1 - G_m^{(k)})}{A_m^{(k)}(m - G_m^{(k)})} \frac{1 + \Delta_m^{(k)}(m - 1 - G_m^{(k)})}{1 + \Delta_m^{(k)}(m - G_m^{(k)})}. \label{eqn:prob_ratio}
\end{IEEEeqnarray}
where we have defined $\Delta_m^{(k)}(j) = \varepsilon_m^{(k)}(j) / A_m^{(k)}(j)$.

To analyze the probability ratio, we begin with the error sequences.
Using \cref{thm:mamatov}, for any $m$ and $k \in [2m]$,
\begin{IEEEeqnarray}{rCl}
    \abs{\Delta_m^{(k)}(m - G_m^{(k)})} & \leq & \frac{C \sqrt{2 \pi}}{(B_m^{(k)})^{1/2}} \exp\biggl( \frac{(\mu_m^{(k)} - m)^2}{2 B_m^{(k)}} \biggr) \\
    & \leq & \frac{C \sqrt{2 \pi}}{(B_m^\ast)^{1/2}} \exp\biggl( \frac{(R_m^\ast)^2}{2 B_m^\ast} \biggr) \\
    & \coloneq & d_m(m) \\
    \abs{\Delta_m^{(k)}(m - 1 - G_m^{(k)})} & \leq & \frac{C \sqrt{2\pi}}{(B_m^{(k)})^{1/2}} \exp\biggl( \frac{(\mu_m^{(k)} - m + 1)^2}{2B_m^{(k)}} \biggr) \\
    & = & \frac{C \sqrt{2\pi}}{(B_m^{(k)})^{1/2}} \exp\biggl( \frac{(\mu_m^{(k)} - m)^2}{2 B_m^{(k)}} \biggr) \exp\biggl( \frac{\mu_m^{(k)} - m}{B_m^{(k)}} \biggr) \exp\biggl( \frac{1}{2 B_m^{(k)}} \biggr) \IEEEeqnarraynumspace \\
    & \leq & \frac{C \sqrt{2\pi}}{(B_m^\ast)^{1/2}} \exp\biggl( \frac{(R_m^\ast)^2}{2 B_m^\ast} \biggr) \exp\biggl( \frac{R_m^\ast}{B_m^\ast} \biggr) \exp\biggl( \frac{1}{2 B_m^\ast} \biggr) \\
    & \coloneq & d_m(m-1). \label{eqn:d_M_minus_1}
\end{IEEEeqnarray}
As a result,
\begin{IEEEeqnarray}{c}
    \max\{ \abs{\Delta_m^{(k)}(m - G_m^{(k)})}, \abs{\Delta_m^{(k)}(m - 1 - G_m^{(k)})} \} \leq \max\{ d_m(m), d_m(m-1) \} \coloneq d_m. \IEEEeqnarraynumspace
\end{IEEEeqnarray}

We now proceed to analyze the sequence $d_m$ by looking first at $d_m(m)$ and then at $d_m(m-1)$.
First, for all $m \geq 3$,
\begin{IEEEeqnarray}{c}
    d_m(m) = \frac{C m^{-1/2} \sqrt{2 \pi}}{(B_m^\ast / m)^{1/2}} \exp\biggl( \biggl( \frac{R_m^\ast}{\sqrt{2 m \ln\ln m}} \biggr)^2 \frac{m \ln\ln m}{B_m^\ast} \biggr).
\end{IEEEeqnarray}
By \eqref{eqn:epsilon-delta_1} and \eqref{eqn:epsilon-delta_2}, for any $\varepsilon_1, \varepsilon_2 > 0$, there exists some $M_5$ such that for any $m > M_5$,
\begin{IEEEeqnarray}{c}
    \biggl( \frac{R_m^\ast}{\sqrt{2m \ln\ln m}} \biggr)^2 \frac{m}{B_m^\ast} \leq E(\varepsilon_1, \varepsilon_2) \label{eqn:D_expansion}
\end{IEEEeqnarray}
where $E(\varepsilon_1, \varepsilon_2) = (\sigma_Y + \varepsilon_1)^2(\mu_V^{-1} + \varepsilon_2) > 0$.
With the required preliminaries out of the way, and fixing $\varepsilon_1$ and $\varepsilon_2$, we evaluate the limit as follows.
\begin{IEEEeqnarray}{rCl}
    0 & \leq & \lim_{m \to \infty} d_m(m) \\
    & \leq & \lim_{m \to \infty} \frac{C m^{-1/2} \sqrt{2 \pi}}{(B_m^\ast / m)^{1/2}} (\ln m)^{E(\varepsilon_1, \varepsilon_2)} \\
    & = & \frac{C \sqrt{2\pi}}{\sqrt{\mu_V}} \lim_{m \to \infty} \frac{(\ln m)^{E(\varepsilon_1, \varepsilon_2)}}{\sqrt{m}}
\end{IEEEeqnarray}
where we have also applied \cref{cor:outer_randomness_extra} to evaluate $(B_m^\ast / m)^{1/2}$ in the limit.
Since for any positive constant $\rho$, $\lim_{m \to \infty} m^{-1/2} (\ln m)^\rho = 0$, it follows that $\lim_{m \to \infty} d_m(m) = 0$.

Moving on to $d_m(m-1)$, we reuse \eqref{eqn:Rm_over_Bm} and \eqref{eqn:D_expansion}, fix $\varepsilon_1$ and $\varepsilon_2$, and account for the final exponential term in \eqref{eqn:d_M_minus_1} as $B_m^\ast \to \infty$ by \cref{cor:outer_randomness_extra}.
We get
\begin{IEEEeqnarray}{rCl}
    0 & \leq & \lim_{m \to \infty} d_m(m-1) \nonumber \\
    & \leq & \lim_{m \to \infty} \frac{C m^{-1/2} \sqrt{2\pi}}{(B_m^\ast / m)^{1/2}} (\ln m)^{E(\varepsilon_1, \varepsilon_2)} \exp\biggl( D(\varepsilon_1, \varepsilon_2) \sqrt{\frac{2 \ln\ln m}{m}} \biggr) \IEEEeqnarraynumspace \\
    & = & \frac{C \sqrt{2 \pi}}{\sqrt{\mu_V}} \lim_{m \to \infty} \frac{(\ln m)^{E(\varepsilon_1, \varepsilon_2)}}{\sqrt{m}} \exp\biggl( D(\varepsilon_1, \varepsilon_2) \sqrt{\frac{2 \ln\ln m}{m}} \biggr). \IEEEeqnarraynumspace \label{eqn:exponent_limit}
\end{IEEEeqnarray}
As the limit of the exponential term evaluates to one, we again get $\lim_{m \to \infty} d_m(m-1) = 0$.
Consequently, $\lim_{m \to \infty} d_m = 0$.

Next, we examine the Gaussian ratio.
Expanding gives
\begin{IEEEeqnarray}{rCl}
    \IEEEeqnarraymulticol{3}{l}{
        \ln \frac{A_m^{(k)}(m - 1 - G_m^{(k)})}{A_m^{(k)}(m - G_m^{(k)})} = \frac{m - \mu_m^{(k)}}{B_m^{(k)}} - \frac{1}{2 B_m^{(k)}}
    }\\* \quad
    & \implies & -\frac{\abs{\mu_m^{(k)} - m}}{B_m^{(k)}} - \frac{1}{2 B_m^{(k)}} \leq \ln \frac{A_m^{(k)}(m - 1 - G_m^{(k)})}{A_m^{(k)}(m - G_m^{(k)})} \leq \frac{\abs{\mu_m^{(k)} - m}}{B_m^{(k)}} + \frac{1}{2 B_m^{(k)}} \\
    & \implies & -\frac{R_m^\ast}{B_m^\ast} - \frac{1}{2 B_m^\ast} \leq \ln \frac{A_m^{(k)}(m - 1 - G_m^{(k)})}{A_m^{(k)}(m - G_m^{(k)})} \leq \frac{R_m^\ast}{B_m^\ast} + \frac{1}{2 B_m^\ast}.
\end{IEEEeqnarray}
Let $c_m = (R_m^\ast / B_m^\ast) + 1 / (2 B_m^\ast)$, we then have that for any $m$ and $k \in [2m]$,
\begin{IEEEeqnarray}{c}
    e^{-c_m} \leq \frac{A_m^{(k)}(m - 1 - G_m^{(k)})}{A_m^{(k)}(m - G_m^{(k)})} \leq e^{c_m}. \label{eqn:gaussian_ratio_bound}
\end{IEEEeqnarray}
Turning our attention to the sequence $c_m$, using the result in \eqref{eqn:Rm_over_Bm_limit} and the fact that $B_m^\ast \to \infty$ by \cref{cor:outer_randomness_extra}, we obtain $\lim_{m \to \infty} c_m = 0$.

Using these results on $d_m$ and $c_m$, we can proceed to examine \eqref{eqn:prob_ratio}.
As $d_m \to 0$, there exists some $M_6$ such that for all $m > M_6$, $d_m < 1$ and hence for any $k \in [2m]$,
\begin{IEEEeqnarray}{c}
    \frac{1 - d_m}{1 + d_m} \leq \frac{1 + \Delta_m^{(k)}(m - 1 - G_m^{(k)})}{1 + \Delta_m^{(k)}(m - G_m^{(k)})} \leq \frac{1 + d_m}{1 - d_m}.
\end{IEEEeqnarray}
Then, using \eqref{eqn:gaussian_ratio_bound},
\begin{IEEEeqnarray}{c}
    e^{-c_m} \frac{1 - d_m}{1 + d_m} \leq \frac{A_m^{(k)}(m - 1 - G_m^{(k)})}{A_m^{(k)}(m - G_m^{(k)})} \frac{1 + \Delta_m^{(k)}(m - 1 - G_m^{(k)})}{1 + \Delta_m^{(k)}(m - G_m^{(k)})} \leq e^{c_m} \frac{1 + d_m}{1 - d_m}.
\end{IEEEeqnarray}
As the bounding sequences do not depend on $k$, we can move to the supremum and get
\begin{IEEEeqnarray}{rCl}
    0 & \leq & \sup_{1 \leq k \leq 2m} \biggl| \frac{A_m^{(k)}(m - 1 - G_m^{(k)})}{A_m^{(k)}(m - G_m^{(k)})} \frac{1 + \Delta_m^{(k)}(m - 1 - G_m^{(k)})}{1 + \Delta_m^{(k)}(m - G_m^{(k)})} - 1 \biggr| \\
    & \leq & \max\biggl\{ e^{c_m} \frac{1 + d_m}{1 - d_m} - 1, 1 - e^{-c_m} \frac{1 - d_m}{1 + d_m} \biggr\}.
\end{IEEEeqnarray}
As we have already established that $\lim_{m \to \infty} c_m = \lim_{m \to \infty} d_m = 0$, we obtain the desired result for the sample path $\{ u^{2m} \}_{m \geq 1}$, from which almost sure convergence follows.  \hfill \IEEEQED{}

\section{Theoretical Comparison and Lower Bound}\label[appendix]{sec:lower_bound}
Recall that from \cref{thm:cost_upper_bound}, the one-shot permuted scheme achieves
\begin{IEEEeqnarray}{c}
    \E[\ell(M)] \leq I(X; Y \mid \bar{U}^N) + \log(I(X;Y \mid \bar{U}^N) + 2) + 3.\label{eqn:achievability_bound_apx}
\end{IEEEeqnarray}
We can also establish the equality
\begin{IEEEeqnarray}{rCl}
    I(X;Y) + I(X;\bar{U}^N \mid Y) & = & H(X) - H(X \mid Y) + H(X \mid Y) - H(X \mid Y, \bar{U}^N) \\
    & \stackrel{\text{(a)}}{=} & H(X \mid \bar{U}^N) - H(X \mid Y, \bar{U}^N) \\
    & = & I(X; Y \mid \bar{U}^N) \label{eqn:mutual_info_equality} \\
    & \stackrel{\text{(b)}}{=} & I(X; K \mid \bar{U}^N) \label{eqn:mutual_info_k_to_y}
\end{IEEEeqnarray}
where in (a) we have used the fact that $\bar{U}^N$ is generated independently of $X$ and (b) comes from the bijection argument used in the proof of \cref{thm:cost_upper_bound}.
Therefore, the leading term in the achievability bound in \eqref{eqn:achievability_bound_apx} contains a penalty, namely $I(X;\bar{U}^N \mid Y)$, above the mutual information $I(X;Y)$.

This may be surprising when compared against other well-known channel simulation schemes including PFR \citep{li2018}, which achieves
\begin{IEEEeqnarray}{c}
    \E[\ell(M)] \leq I(X;Y) + \log(I(X;Y) + 2) + 3
\end{IEEEeqnarray}
or GRS \citep{harsha2010,flamich2023b}, which satisfies
\begin{IEEEeqnarray}{c}
    \E[\ell(M)] \leq I(X;Y) + \log(I(X;Y) + \log(4e)) + \log(8e).
\end{IEEEeqnarray}
Neither of these expressions contain such a penalty; extending the PFR and GRS bounds to $n$ independent uses of the channel with a source $X^n$ leaves $I(X;Y)$ as the only term which does not vanish as $n$ becomes large, when considering the number of bits per sample $\E[\ell(M)] / n$.

We now present a lower bound on the communication cost, which shows that the penalty is inherent to the structure of the finite sample pool $\bar{U}^N$ used as part of the shared randomness in the permuted scheme, rather than being an artifact of the implementation or a loose upper bound in \cref{thm:cost_upper_bound}.

\begin{proposition}\label{prop:cost_lower_bound}
The communication cost of the permuted scheme is bound below by
\begin{IEEEeqnarray}{c}
    \E[\ell(M)] \geq I(X;Y \mid \bar{U}^N). \label{eqn:coding_cost_lower_bound}
\end{IEEEeqnarray}
\end{proposition}
\begin{IEEEproof}
The cost of transmitting any message $M$ with shared randomness $W$ is bound below by
\begin{IEEEeqnarray}{c}
    H(M \mid W) \geq H(M \mid W) - H(M \mid X, W) = I(X; M \mid W).
\end{IEEEeqnarray}
Since $W$ is generated independent of the input $X$,
\begin{IEEEeqnarray}{c}
    I(X; M \mid W) = I(X; M \mid W) + I(X; W) = I(X; M, W).
\end{IEEEeqnarray}
Since the decoder outputs $Y$ as a function of $M$ and $W$, adding $Y$ does not change the mutual information.
We get
\begin{IEEEeqnarray}{rCl}
    I(X; M, W) & = & I(X; M, W, Y) \\
    & = & I(X;Y) + I(X; W \mid Y) + I(X; M \mid Y, W) \\
    & \geq & I(X;Y) + I(X; W \mid Y).
\end{IEEEeqnarray}
Finally,
\begin{IEEEeqnarray}{c}
    \E[\ell(M)] \geq H(M \mid W) \geq I(X; Y) + I(X; W \mid Y). \label{eqn:generic_lower_bound}
\end{IEEEeqnarray}
For the permuted scheme in particular, the shared randomness comprises $\bar{U}^N$, $S^N$ and $\Pi$.
Substituting for $W$ gives 
\begin{IEEEeqnarray}{c}
    \E[\ell(M)] \geq I(X;Y) + I(X; \bar{U}^N, S^N, \Pi \mid Y) \geq I(X;Y) + I(X; \bar{U}^N \mid Y).
\end{IEEEeqnarray}
which yields the desired lower bound after using \eqref{eqn:mutual_info_equality}.
\end{IEEEproof}

\section{Non-Uniform PolarSim with Side Information}\label[appendix]{sec:polar_sim}
In this section, we present a modified version of PolarSim \citep{sriramu2024} that allows for non-uniformly distributed inputs as well as the use of side information available at both the encoder and decoder.
While the generalized PolarSim algorithm presented in \citet{ozyilkan2026} covers this setting and further allows for nonstationary channels, it relies on interleaving additional random permutations inside the polar transform, and its theoretical analysis is complicated by the fact that the number of recursion levels may not equal $\log n$ for $n$ channel uses.
In contrast, our method makes use of the standard, unmodified polar transform \citep{arikan2009,arikan2010}, considerably simplifying the theory and implementation.
We emphasize that our modified PolarSim is still limited to binary-output channels, and is paired with MLC in \cref{sec:polar} of the main paper to handle larger output alphabets.

The complete algorithm, comprising an encoder and a decoder sharing common randomness via the PRNG $\mathfrak{G}$, is presented in \cref{alg:polar_sim}, where $G_n$ is the polar transform matrix as defined in \citet{arikan2010}.
Note that $G_n^{-1} = G_n$, and the polar transform can in practice be implemented recursively in $O(n \log n)$ time.
The goal of \cref{alg:polar_sim} is to simulate $n$ iid copies of a binary-output channel $P_{Y \mid X, W}$ when provided input sequences $x^n$ and $w^n$.
Here, $x^n$ is the channel input and $w^n$ is the side information, with $(x^n, w^n)$ comprising $n$ iid draws from a joint distribution $P_{X, W}$.
As $w^n$ is known to both the encoder and the decoder, the conditional mutual information $I(X;Y \mid W)$ is a lower bound on the communication cost per channel use.

\begin{algorithm}[!b]
\caption{Modified PolarSim \citep{sriramu2024,ozyilkan2026}.}\label{alg:polar_sim}
\begin{algorithmic}\small
\Function{Encode}{$x^n, w^n, \mathfrak{G}, \mathtt{enc}$}
    \State \algorithmicrequire{} Channel input $x^n$, side information $w^n$, PRNG $\mathfrak{G}$, lossless encoder $\mathtt{enc}$.
    \State \algorithmicensure{} Message $M$.
    \State Generate $S^n \sim \operatorname{Unif}[0, 1]$ using $\mathfrak{G}$.
    \For{$i = 1, \ldots, n$}
        \State $p_i \gets \Pr[Z_i = 0 \mid Z^{i-1} = z^{i-1}, X^n = x^n, W^n = w^n]$
        \State $q_i \gets \Pr[Z_i = 0 \mid Z^{i-1} = z^{i-1}, W^n = w^n]$
        \State $z_i \gets \1\{ S_i > p_i \}$
        \State $\tilde{z}_i \gets \1\{ S_i > q_i \}$
        \State $\Delta_i \gets z_i \oplus \tilde{z}_i$
    \EndFor
    \State $M \gets \mathtt{enc}(\Delta^n)$
    \State \Return $M$
\EndFunction
\vspace{1ex}
\Function{Decode}{$M, w^n, \mathfrak{G}, \mathtt{dec}$}
    \State \algorithmicrequire{} Message $M$, side information $w^n$, PRNG $\mathfrak{G}$, lossless decoder $\mathtt{dec}$.
    \State \algorithmicensure{} Channel output $y^n \in \{ 0, 1 \}^n$.
    \State Generate $S^n \sim \operatorname{Unif}[0, 1]$ using $\mathfrak{G}$.
    \State $\Delta^n \gets \mathtt{dec}(M)$
    \For{$i = 1, \ldots, n$}
        \State $q_i \gets \Pr[Z_i = 0 \mid Z^{i-1} = z^{i-1}, W^n = w^n]$
        \State $\tilde{z}_i \gets \1\{ S_i > q_i \}$
        \State $z_i \gets \Delta_i \oplus \tilde{z}_i$
        \EndFor
    \State $y^n \gets z^n G_n$
    \State \Return $y^n$
\EndFunction
\end{algorithmic}
\end{algorithm}

We now present two results on our variant of the PolarSim algorithm, first dealing with its correctness in terms of the output distribution and second with the rate performance in the limit of large blocklengths.
Note that in practice, we implement $\mathtt{enc}$, $\mathtt{dec}$ using arithmetic coding \citep{rissanen1976} with empirically estimated probabilities, therefore leading to some rate overhead above the theoretical bounds, which imply knowledge of the exact probabilities.

\begin{theorem}\label{thm:polar_sim}
\Cref{alg:polar_sim} simulates the channel $P_{Y \mid X, W}$ exactly, in the sense that given inputs $x^n$ and $w^n$, the output $y^n \in \{ 0, 1 \}^n$ is a draw from $P_{Y^n \mid X^n, W^n}(\cdot \mid w^n, x^n) = \prod_{i=1}^{n} P_{Y \mid X, W}(\cdot \mid x_i, w_i)$.
Moreover, it is rate-optimal in the sense that there exists a choice of $\mathtt{enc}$, $\mathtt{dec}$ such that
\begin{IEEEeqnarray}{c}
    \lim_{n \to \infty} \frac{1}{n} \E[\ell(M)] = I(X; Y \mid W). 
\end{IEEEeqnarray}
\end{theorem}
\begin{IEEEproof}
Beginning with the claim of correctness and following the argument in \citet{ozyilkan2026}, the polar transform $y^n = z^n G_n$ is a bijection, and by construction $z^n$ is a draw from $P_{Z^n \mid X^n, W^n}(\cdot \mid w^n, x^n)$.
As the decoder reconstructs $z^n$ exactly, it follows that $y^n$ is a draw from $P_{Y^n \mid X^n, W^n}(\cdot \mid x^n, w^n)$, which is the desired output distribution.

Moving on to consider the rate, we again follow a slight modification of the proof given in \citet{ozyilkan2026}; the full argument is presented here for completeness using the same notation as that in \cref{alg:polar_sim}.
We have
\begin{IEEEeqnarray}{rCl}
    \IEEEeqnarraymulticol{3}{l}{
        \Pr[\Delta_i = 1]
    }\nonumber\\* \
    & = & \E_{Z^{i-1}, X^n, W^n}[\Pr[\Delta_i = 1 \mid Z^{i-1}, X^n, W^n]] \\
    & = & \E_{Z^{i-1}, X^n, W^n}[\abs{\Pr[Z_i = 1 \mid Z^{i-1}, X^n, W^n] - \Pr[Z_i = 1 \mid Z^{i-1}, W^n]}] \\
    & \leq & ( \E_{Z^{i-1}, X^n, W^n}[(\Pr[Z_i = 1 \mid Z^{i-1}, X^n, W^n] - \Pr[Z_i = 1 \mid Z^{i-1}, W^n])^2] )^{1/2} \\
    & = & (\E_{Z^{i-1}, W^n}[\E[(\Pr[Z_i = 1 \mid Z^{i-1}, X^n, W^n] \nonumber \\
    &&\< - \E[\Pr[Z_i = 1 \mid Z^{i-1}, X^n, W^n] \mid Z^{i-1}, W^n])^2 \mid Z^{i-1}, W^n]])^{1/2} \\
    & = & ( \E_{Z^{i-1}, W^n}[\Var(\Pr[Z_i = 1 \mid Z^{i-1}, X^n, W^n] \mid Z^{i-1}, W^n)] )^{1/2} \\
    & = & ( \Var(\Pr[Z_i = 1 \mid Z^{i-1}, X^n, W^n ]) \nonumber \\
    &&\< - \Var(\E[\Pr[Z_i = 1 \mid Z^{i-1}, X^n, W^n] \mid Z^{i-1}, W^n]) )^{1/2} \\
    & = & ( \Var(\Pr[Z_i = 1 \mid Z^{i-1}, X^n, W^n]) - \Var(\Pr[Z_i = 1 \mid Z^{i-1}, W^n]) )^{1/2}
\end{IEEEeqnarray}
where we have followed the steps in \citet[Appendix~C]{ozyilkan2026}, while removing the permutation randomness present in their construction and inserting the shared side information sequence $W^n$.
Then, following Lemma~C.1 in the same document, we get
\begin{IEEEeqnarray}{c}
    \Pr[\Delta_i = 1] \leq \frac{1}{2} (H(Z_i \mid Z^{i-1}, W^n) - (H(Z_i \mid Z^{i-1}, X^n, W^n))^2)^{1/2} \leq \frac{1}{2}.
\end{IEEEeqnarray}
If $\mathtt{enc}$, $\mathtt{dec}$ are chosen to be an arithmetic encoder-decoder pair furnished with the appropriate probabilities, the expected message length satisfies \citep{rissanen1976}
\begin{IEEEeqnarray}{rCl}
    \frac{\E[\ell(M)]}{n} & \leq & \frac{1}{n} \sum_{i=1}^{n} h_B(\Pr[\Delta_i = 1]) + \frac{2}{n} \\
    & \leq & \frac{1}{n} \sum_{i=1}^{n} h_B\biggl( \frac{1}{2} (H(Z_i \mid Z^{i-1}, W^n) - (H(Z_i \mid Z^{i-1}, X^n, W^n))^2)^{1/2} \biggr) + \frac{2}{n}. \IEEEeqnarraynumspace
\end{IEEEeqnarray}
For any $\delta \in (0, 1/2)$, create three partitions,
\begin{IEEEeqnarray}{rCl}
    A_{L} & = & \{ i \in [n] : H(Z_i \mid Z^{i-1}, W^n) \in [0, \delta) \} \\
    A_{M} & = & \{ i \in [n] : H(Z_i \mid Z^{i-1}, W^n) \in [\delta, 1 - \delta] \} \\
    A_{H} & = & \{ i \in [n] : H(Z_i \mid Z^{i-1}, W^n) \in (1 - \delta, 1] \}
\end{IEEEeqnarray}
which form the low, medium and high-entropy sets respectively. 
Then,
\begin{IEEEeqnarray}{rCl}
    \frac{\E[\ell(M)]}{n} & \leq & \frac{1}{n} \sum_{i \in A_H} h_B\biggl( \frac{1}{2} (H(Z_i \mid Z^{i-1}, W^n) - (H(Z_i \mid Z^{i-1}, X^n, W^n))^2)^{1/2} \biggr) \nonumber \\
    &&\< + h_B\biggl( \frac{\sqrt{\delta}}{2} \biggr) + \frac{\abs{A_M}}{n} + \frac{2}{n}.
\end{IEEEeqnarray}

To continue with the proof, which will involve bounding the sizes of $A_L$, $A_M$ and $A_H$, we recall the standard source polarization theorem.

\begin{lemma}[{\citealp[Theorem~1]{arikan2010}}]\label{lem:polarization}
Let $(X,Y) \sim P_{X,Y}$ be an arbitrary pair of random variables over $\mc{X} \times \mc{Y}$, where $\mc{X} = \{ 0, 1\}$.
Consider $\{ (X_i, Y_i) \}_{i\geq 1}$ a sequence of independent draws from $(X, Y)$ and write $(X^n, Y^n)$ to denote the first $n$ elements of this sequence.
For any $n = 2^m$, $m \geq 1$, let $Z^n = X^n G_n$.
Then, for any $\delta \in (0, 1)$,
\begin{IEEEeqnarray}{c}
    \lim_{n \to \infty} \frac{1}{n} \abs{\{ i \in [n] : H(Z_i \mid Z^{i-1}, Y^n) \in (1 - \delta, 1] \}}  = H(X \mid Y) \\
    \lim_{n \to \infty} \frac{1}{n} \abs{\{ i \in [n] : H(Z_i \mid Z^{i-1}, Y^n) \in [0, \delta) \}} = 1 - H(X \mid Y). 
\end{IEEEeqnarray}
\end{lemma}
\begin{remark}
Note that while the result in \citet{arikan2010} is stated for countable $\mc{Y}$ for simplicity, the same arguments, in particular regarding the convergence of the Bhattacharyya supermartingale and its connection to the entropy martingale, work unchanged with the side information spaces considered here if the source Bhattacharyya parameters are defined appropriately as
\begin{IEEEeqnarray}{c}
    Z(X \mid Y) = 2 \E_Y\Bigl[ \sqrt{P_{X \mid Y}(0 \mid Y) P_{X \mid Y}(1 \mid Y)} \Bigr].
\end{IEEEeqnarray}
More specifically, our MLC construction in \cref{sec:polar_mlc} uses the side information $(\bar{U}^N, B^{l-1})$ at level $l$, where $\bar{U}^N$ is continuous and $B^{l-1}$ is discrete.
Thus we state \cref{lem:polarization} in this setting.
See also Theorem~B.2 of \citet{ozyilkan2026} for a related analysis.
In what follows we assume that $n$ is a valid polar blocklength throughout, i.e.\ $n = 2^m$ for some integer $m \geq 1$ as required by the source polarization theorem.
\end{remark}

Since $(X^n, W^n) \sim P_{X,W}$ iid and the channel $P_{Y \mid X, W}$ is memoryless, we note
\begin{IEEEeqnarray}{rCl}
    P_{Y^n \mid W^n}(y^n \mid w^n) & = & \E_{X^n \mid W^n=w^n}[P_{Y^n \mid X^n, W^n}(y^n \mid X^n, w^n)] \label{eqn:memoryless_channel_start} \\
    & = & \E_{X^n \mid W^n=w^n} \biggl[ \prod_{i=1}^{n} P_{Y \mid X, W}(y_i \mid X_i, w_i) \biggr] \\
    & = & \E_{X_1 \mid W_1=w_1} \cdots \E_{X_n \mid W_n=w_n} \biggl[ \prod_{i=1}^{n} P_{Y \mid X, W}(y_i \mid X_i, w_i) \biggr] \\
    & = & \prod_{i=1}^{n} \E_{X_i \mid W_i = w_i}[P_{Y \mid X, W}(y_i \mid X_i, w_i)] \\
    & = & \prod_{i=1}^{n} P_{Y \mid W}(y_i \mid w_i). \label{eqn:memoryless_channel_end}
\end{IEEEeqnarray}
Therefore, $(Y^n, W^n) \sim P_{Y \mid W} P_W$ iid.
Applying \cref{lem:polarization} with $X^n \gets Y^n$ and $Y^n \gets W^n$, we then have that $\lim_{n \to \infty} \abs{A_M} / n = 0$.
From this, there exists some $N_A$ such that for all $n > N_A$, $\abs{A_M} / n < \delta$.
Further using the bound $h_B(\sqrt{\delta} / 2) \leq 2 \delta^{1/4}$ \citep{ozyilkan2026}, we have for any $n > N_A$,
\begin{IEEEeqnarray}{rCl}
    \frac{\E[\ell(M)]}{n} & \leq & \frac{1}{n} \sum_{i \in A_H} h_B\biggl( \frac{1}{2} (H(Z_i \mid Z^{i-1}, W^n) - (H(Z_i \mid Z^{i-1}, X^n, W^n))^2)^{1/2} \biggr) \nonumber \\
    &&\< + 2 \delta^{1/4} + \delta + \frac{2}{n}. \label{eqn:length_ineq}
\end{IEEEeqnarray}
Under the same condition on $n$, and again following \citet{ozyilkan2026} with the necessary modifications, 
\begin{IEEEeqnarray}{rCl}
    \IEEEeqnarraymulticol{3}{l}{
        (1 - \delta) \abs{A_H} \leq \sum_{i=1}^{n} H(Z_i \mid Z^{i-1}, W^n) \leq \abs{A_H} + (1 - \delta) \abs{A_M} + \delta \abs{A_L}
    } \label{eqn:partition_bound_start}\\* \
    & \implies & (1 - \delta) \abs{A_H} \leq n H(Y \mid W) \leq \abs{A_H} + (1 - \delta) \abs{A_M} + \delta \abs{A_L} \\
    & \implies & (1 - \delta) \frac{\abs{A_H}}{n} \leq H(Y \mid W) \leq \frac{\abs{A_H}}{n} + 2 \delta \\
    & \implies & -\delta \leq H(Y \mid W) - \frac{\abs{A_H}}{n} \leq 2 \delta \\
    & \implies & \frac{\abs{A_H}}{n} \leq H(Y \mid W) + \delta. \label{eqn:partition_bound_end}
\end{IEEEeqnarray}
Here we have used the fact that, as the polar transform is a bijection, $H(Z^n \mid W^n) = H(Y^n \mid W^n)$.
Also, as we previously established that $(Y^n, W^n)$ is iid,
\begin{IEEEeqnarray}{c}
    \sum_{i=1}^{n} H(Z_i \mid Z^{i-1}, W^n) = H(Y^n \mid W^n) = n H(Y \mid W). \label{eqn:iid_bijection_simplification}
\end{IEEEeqnarray}
Next defining
\begin{IEEEeqnarray}{c}
    \tilde{R} = \frac{1}{n} \sum_{i \in A_H} h_B\biggl( \frac{1}{2} (H(Z_i \mid Z^{i-1}, W^n) - (H(Z_i \mid Z^{i-1}, X^n, W^n))^2)^{1/2} \biggr) \label{eqn:r_tilde_def}
\end{IEEEeqnarray}
and the partitions
\begin{IEEEeqnarray}{rCl}
    B_{L} & = & \{ i \in [n] : H(Z_i \mid Z^{i-1}, X^n, W^n) \in [0, \delta) \} \\
    B_{M} & = & \{ i \in [n] : H(Z_i \mid Z^{i-1}, X^n, W^n) \in [\delta, 1 - \delta] \} \\
    B_{H} & = & \{ i \in [n] : H(Z_i \mid Z^{i-1}, X^n, W^n) \in (1 - \delta, 1] \}
\end{IEEEeqnarray}
we see that
\begin{IEEEeqnarray}{rCl}
    \tilde{R} & \leq & \frac{\abs{A_H \cap B_L}}{n} + \frac{\abs{A_H \cap B_M}}{n} + h_B\biggl( \sqrt{\frac{\delta}{2}} \biggr) \\
    & \leq &  \frac{\abs{A_H \cap B_L}}{n} + \frac{\abs{B_M}}{n} + h_B\biggl( \sqrt{\frac{\delta}{2}} \biggr) \label{eqn:r_tilde_ineq}
\end{IEEEeqnarray}
where we have used the fact that for any $i \in A_H \cap B_H$,
\begin{IEEEeqnarray}{rCl}
    H(Z_i \mid Z^{i-1}, W^n) - (H(Z_i \mid Z^{i-1}, X^n, W^n))^2 \leq 1 - (1 - \delta)^2 \leq 2 \delta.
\end{IEEEeqnarray}
As $(Y^n, X^n, W^n) \sim P_{Y \mid X, W} P_{X,W}$ iid, we can apply \cref{lem:polarization} again, this time with $X^n \gets Y^n$ and $Y^n \gets (X^n, W^n)$.
We then have $\lim_{n \to \infty} \abs{B_M} / n = 0$, and so there exists some $N_B$ such that for all $n > N_B$, $\abs{B_M} / n < \delta$.
For such $n$, following a similar approach to \eqref{eqn:partition_bound_start}--\eqref{eqn:partition_bound_end},
\begin{IEEEeqnarray}{rCl}
    \IEEEeqnarraymulticol{3}{l}{
        (1 - \delta) \abs{B_H} \leq \sum_{i=1}^{n} H(Z_i \mid Z^{i-1}, X^n, W^n) \leq \abs{B_H} + (1 - \delta) \abs{B_M} + \delta \abs{B_L}
    }\\* \
    & \implies & (1 - \delta) \abs{B_H} \leq n H(Y \mid X, W) \leq \abs{B_H} + (1 - \delta) \abs{B_M} + \delta \abs{B_L} \\
    & \implies & (1 - \delta) \frac{\abs{B_H}}{n} \leq H(Y \mid X, W) \leq \frac{\abs{B_H}}{n} + 2 \delta \\
    & \implies & -\delta \leq H(Y \mid X, W) - \frac{\abs{B_H}}{n} \leq 2 \delta \\
    & \implies & \frac{\abs{B_H}}{n} \geq H(Y \mid X, W) - 2 \delta.
\end{IEEEeqnarray}
Again, similar to \eqref{eqn:iid_bijection_simplification}, we have made use of the simplification
\begin{IEEEeqnarray}{c}
    \sum_{i=1}^{n} H(Z_i \mid Z^{i-1}, X^n, W^n) = H(Y^n \mid X^n, W^n) = n H(Y \mid X, W).
\end{IEEEeqnarray}
Thus, for any $n > N_0 = \max\{ N_A, N_B \}$, and noting that $B_H \subseteq A_H$ we have
\begin{IEEEeqnarray}{rCl}
    \frac{\abs{A_H \cap B_L}}{n}  & \leq & \frac{\abs{A_H}}{n} - \frac{\abs{A_H \cap B_H}}{n} \\
    & = & \frac{\abs{A_H}}{n} - \frac{\abs{B_H}}{n} \label{eqn:ah_bh_ineq} \\
    & \leq & H(Y \mid W) - H(Y \mid X, W) + 3 \delta \\
    & = & I(X; Y \mid W) + 3 \delta.
\end{IEEEeqnarray}
Substituting into \eqref{eqn:r_tilde_ineq}, we obtain for all such $n$,
\begin{IEEEeqnarray}{c}
    \tilde{R} \leq I(X; Y \mid W) + 4 \delta + 2 \delta^{1/4}
\end{IEEEeqnarray}
where we have applied the inequality $h_B(\sqrt{\delta / 2}) \leq 2 \delta^{1/4}$ \citep{ozyilkan2026}.
Finally, let $N_1 = \max\{ N_0, 2 / \delta \}$.
For any $n > N_1$, after substituting into \eqref{eqn:length_ineq}, we get
\begin{IEEEeqnarray}{c}
    \frac{\E[\ell(M)]}{n} \leq I(X; Y \mid W) + 6 \delta + 4 \delta^{1/4}.
\end{IEEEeqnarray}
Since $\delta$ can be made arbitrarily small, after combining with the fact that $I(X;Y \mid W)$ is a lower bound on the communication cost per sample, we have shown that
\begin{IEEEeqnarray}{c}
    \lim_{n \to \infty} \frac{1}{n} \E[\ell(M)] = I(X; Y \mid W)
\end{IEEEeqnarray}
which completes the proof.
\end{IEEEproof}

\section{Additional Experimental Results and Details}\label[appendix]{sec:exp_extra}
\subsection{Approximating the Posterior Distribution}\label[appendix]{sec:approx_posterior}
As briefly mentioned in \cref{sec:permuted_scheme_impl}, computing the exact posterior distribution $P_{K \mid X, \bar{U}^N}(\cdot \mid x, u^N)$ rapidly becomes impractical at larger alphabet sizes as it involves calculating matrix permanents.
When the alphabet size grows beyond $N = 16$ in our image compression application, we therefore adopt an iterative approximation of the posterior via a Sinkhorn-Knopp-based algorithm introduced to machine learning in \citet{mena2018,mena2020}.
In this section, we first briefly describe how the approximation is implemented in the context of the permuted scheme as introduced in \cref{sec:one_shot_simulation}, and then provide empirical results on the suitability of its output distribution for applied channel simulation tasks.

\paragraph{Implementation Details.}
\citet{mena2020} describe an algorithm for approximating posterior distributions over a permutation matrix $P$, i.e. a random $N \times N$ binary matrix with a unique 1 in each row and column representing a permutation of $[N]$, where the distribution of $P$ is 
\begin{IEEEeqnarray}{c}
    P_{P \mid A}(P \mid A) = \frac{1}{Z_{A}} \exp(\langle \ln A, P \rangle_F).
\end{IEEEeqnarray}
Here, $\langle \ln A, P \rangle_F$ is the Frobenius matrix inner product, $A$ is a parameter matrix and $Z_{A}$ is a normalizing constant.
In \citet{mena2020} it is shown and theoretically justified that the matrix of expectations $\E[P]$ can be well-approximated by applying the Sinkhorn operator \citep{sinkhorn1964} to $A$.
The Sinkhorn operator is defined to be the infinite successive row-then-column normalization of $A$ and produces a doubly stochastic matrix, which can be viewed as the solution to an instance of optimal transport with entropic regularization \citep{altschuler2017}.
The basic iterative implementation of this successive normalization idea is sometimes referred to as the Sinkhorn-Knopp algorithm \citep{knight2008}.
In practice, a finite number of row and column normalization steps are performed, and the algorithm is carried out in log-space to increase numerical stability.

We briefly outline how this framework allows us to compute an approximation for $P_{K \mid X, \bar{U}^N}(\cdot \mid x, u^N)$.
The likelihood of observing $\bar{U}^N = u^N$ in the course of the permuted scheme given the underlying permutation is $\Pi = \pi$ is
\begin{IEEEeqnarray}{c}
    p_{\bar{U}^N \mid \Pi}(u^N \mid \pi) = \prod_{i=1}^{N} p_{Y \mid X}(u_i \mid \pi^{-1}(i)).
\end{IEEEeqnarray}
Applying Bayes' rule, the posterior induced on the uniformly distributed random permutation $\Pi$ is therefore 
\begin{IEEEeqnarray}{c}
    \Pr[\Pi = \pi \mid \bar{U}^N = u^N] = \frac{p_{\bar{U}^N \mid \Pi}(u^N \mid \pi)}{N! p_{\bar{U}^N}(u^N)} = \frac{p_{\bar{U}^N \mid \Pi}(u^N \mid \pi)}{\operatorname{perm}(L(u^N))}.
\end{IEEEeqnarray}
We can map $\Pi$ to a permutation matrix $P(\Pi) \in \{ 0, 1 \}^{N \times N}$; given a realization $\pi$, we have that $\pi(j) = i$ if and only if $P_{ij} = 1$ or, in the context of the permuted scheme, $u_i$ was drawn from $P_{Y \mid X}(\cdot \mid j)$.
Then, the posterior can be written as a Frobenius inner product involving the log-likelihood matrix $\ln L(u^N)$, where $L(u^N)$ is as defined in \cref{sec:permuted_scheme_impl},
\begin{IEEEeqnarray}{c}
    \Pr[\Pi = \pi \mid \bar{U}^N = u^N] = \frac{\exp(\langle \ln L(u^N), P(\pi) \rangle_{F})}{\operatorname{perm}(L(u^N))}.
\end{IEEEeqnarray}
Taking the expectation of each element of $P(\Pi)$ with respect to the posterior distribution on $\Pi$ yields
\begin{IEEEeqnarray}{rCl}
    \E[P_{ij}] & = & \sum_{\pi \in S_N} \Pr[\Pi = \pi \mid \bar{U}^N = u^N] \1\{ \pi(j) = i \} \\
    & = & \E[\1\{ \Pi(j) = i \} \mid \bar{U}^N = u^N] \\
    & = & \Pr[\Pi(j) = i \mid \bar{U}^N = u^N]. \label{eqn:approx_posterior_final}
\end{IEEEeqnarray}
Note that \eqref{eqn:approx_posterior_final} is exactly the posterior distribution required to run the permuted scheme as described in \cref{sec:oneshot_permuted}, that is $\Pr[\Pi(j) = i \mid \bar{U}^N = u^N] = P_{K \mid X, \bar{U}^N}(i \mid j, u^N)$.
Thus we can obtain the required posterior probabilities given $X = x$ and $\bar{U}^N = u^N$, namely $\Pr[\Pi(x) = k \mid \bar{U}^N = u^N]$ for all $1 \leq k \leq N$, by taking the expectation and extracting column $x$ of $\E[P(\Pi)]$.
Consequently, to get an approximation, we can follow the approach in \citet{mena2020} by computing $\ln L(u^N)$, performing Sinkhorn-Knopp iterations on this log-likelihood matrix, and taking column $x$ from the result.
We show a simple PyTorch implementation in \cref{lst:sinkhorn}.
In practice, for our experiments in \cref{sec:experiment_vqvae}, we use a custom Triton implementation which avoids instantiating the entire likelihood matrix in GPU memory to increase throughput.

\begin{listing}[tb]
\begin{minted}[frame=single,fontsize=\footnotesize]{python}
def sinkhorn(L, num_iters=20):
    for _ in range(num_iters):
        # Row normalization
        L = L - torch.logsumexp(L, dim=1, keepdim=True)
        # Column normalization
        L = L - torch.logsumexp(L, dim=0, keepdim=True)
    return L
\end{minted}
\caption{Sinkhorn-Knopp algorithm in PyTorch.}
\label{lst:sinkhorn}
\end{listing}

\paragraph{Performance of the Approximation.}
To measure the suitability of the approximate Sinkhorn-Knopp posterior, we run a small additional experiment on basic discrete-input $N$-ary PAM channels with additive Gaussian noise.
In particular, we first test the marginal output distribution produced by the channel simulation for $N \in \{ 32, 64, 128, 256, 512, 1024 \}$ by generating $10^5$ samples uniformly from the $N$-ary input $X$, then simulating $P_{Y \mid X}$ with the permuted scheme for each input sample, incorporating the Sinkhorn-Knopp algorithm to compute $P_{K \mid X, \bar{U}^N}(\cdot \mid x, u^N)$ as required.
The inputs come from a constellation $\mc{X} = \{ 2m - (N - 1) : m \in [N] \}$ evenly spaced around the origin, and we use $P_{Y \mid X}(\cdot \mid x) = \mc{N}(x, \sigma^2)$ with $\sigma^2 = 10^{-\mathrm{SNR} / 10}$ and low, moderate and high signal-to-noise ratios (SNRs) in $\{ -10.0, 1.0, 10.0 \}$.

We use the Kolmogorov-Smirnov (KS) statistic \citep{conover1999} to measure the closeness to the target marginal with
\begin{IEEEeqnarray}{c}
    p_{Y}(y) = \frac{1}{N} \sum_{x \in \mc{X}} \mc{N}(y; x, \sigma^2) \label{eqn:true_marginal}
\end{IEEEeqnarray}
and present the results in \cref{tab:marginal_ks}.
For scale, we compare against the same test run on samples generated directly from the true marginal in \eqref{eqn:true_marginal}.
The mean of the KS statistic is displayed over 10 independent runs with $10^5$ samples each, along with its standard deviation.
For all iteration counts, the mean KS statistic for our approximate sampler falls within the uncertainty region. 
We also plot the empirical marginal CDF for $N=32$ derived from our approximate sampler using the Sinkhorn-Knopp posterior alongside the analytical Gaussian mixture CDF in \cref{fig:empirical_cdf}.

\begin{table}[bt]
    \centering
    \setlength\tabcolsep{5 pt}
    \newcolumntype{T}{S[
        table-format=1.2(1),
        uncertainty-mode=separate,
        table-text-alignment=right
    ]}
    \begin{tabular}{@{}llTTTTTT@{}}
        \toprule
        & {SNR} & \multicolumn{2}{c}{$-10.0$} & \multicolumn{2}{c}{$1.0$} & \multicolumn{2}{c}{$10.0$} \\
        $N$ & Iters. & {Approx.} & {True} & {Approx.} & {True} & {Approx.} & {True} \\
        \midrule
        32 & 20 & 2.77(81) & 2.94(87) & 2.87(72) & 3.00(90) & 2.76(82) & 2.99(92) \\
        & 50  & 2.38(86) & & 2.49(79) & & 2.68(76) & \\
        & 100 & 2.69(85) & & 2.61(80) & & 2.67(89) & \\
        64 & 20 & 2.59(81) & 2.50(61) & 2.85(77) & 2.62(76) & 2.89(86) & 2.63(77) \\
        & 50 & 3.23(100) & & 2.94(97) & & 3.30(96) & \\
        & 100 & 2.67(8.0) & & 2.59(69) & & 2.58(74) & \\
        128 & 20 & 2.48(57) & 2.62(77) & 2.58(65) & 2.65(0.77) & 2.58(63) & 2.66(78) \\
        & 50 & 2.65(82) & & 2.63(84) & & 2.66(83) & \\
        & 100 & 2.87(88) & & 2.89(89) & & 2.85(87) & \\
        256 & 20 & 2.94(93) & 2.60(87) & 2.95(92) & 2.62(86) & 2.95(93) & 2.61(87) \\
        & 50 & 2.90(9.0) & & 2.92(89) & & 2.94(87) & \\
        & 100 & 2.71(6.8) & & 2.75(62) & & 2.75(60) & \\
        512 & 20 & 2.75(77) & 2.86(65) & 2.73(79) & 2.82(67) & 2.73(77) & 2.82(68) \\
        & 50 & 2.58(7.1) & & 2.55(71) & & 2.55(70) & \\
        & 100 & 2.67(7.3) & & 2.68(72) & & 2.69(72) & \\
        1024 & 20 & 2.80(84) & 2.54(63) & 2.82(85) & 2.53(64) & 2.82(84) & 2.52(65) \\
        & 50 & 2.65(7.2) & & 2.64(72) & & 2.63(72) & \\
        & 100 & 2.56(7.2) & & 2.52(70) & & 2.53(71) & \\
        \bottomrule
    \end{tabular}
    \caption{KS statistics for the marginal distributions produced for $N$-ary Gaussian channels using the Sinkhorn-Knopp approximation (Approx.) with varying iteration counts and ground-truth samples (True). All values have been scaled up by $10^3$ and one standard deviation is shown.}
    \label{tab:marginal_ks}
\end{table}

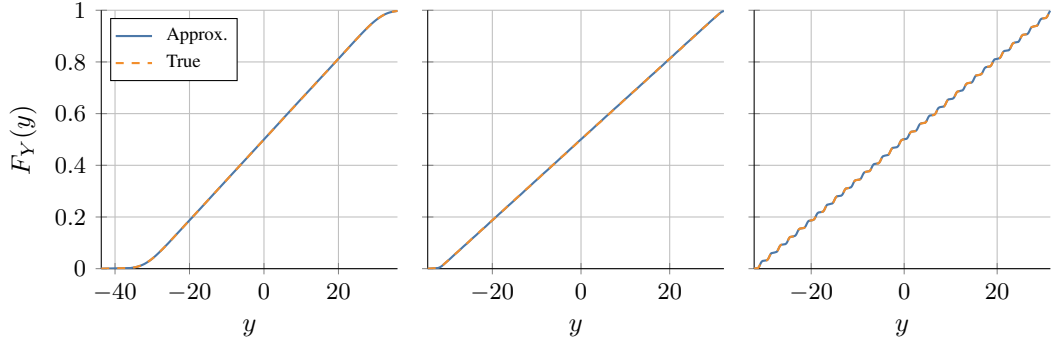
\begin{figure}[bt]
    \centering
    \begin{tikzpicture}
\begin{groupplot}[
    width=5.5cm, 
    height=5cm,
    tick label style={font=\footnotesize}, 
    ymin=0,
    ymax=1,
    grid=major,
    max space between ticks=25,
    axis lines=left,
    axis line style={-},
    enlarge x limits=false,
    group style={group name=my plots, group size=3 by 1, horizontal sep=0.4cm, yticklabels at=edge left, ylabels at=edge left},
    xlabel={$y$}, ylabel={$F_Y(y)$}
]
    \nextgroupplot[
        legend pos=north west,
        legend cell align={left},
        legend style={font=\scriptsize, inner sep=1.5pt, row sep=0pt, column sep=3pt, legend image post style={scale=0.8}, cells={anchor=west}}
    ]

    \addplot[color=c0, thick] table[x=x1_est,y=y1_est,col sep=comma] {plots/empirical_cdf_data.csv};
    \addlegendentry{Approx.}
    
    \addplot[color=c4, thick, dashed] table[x=x1_est,y=y1_true,col sep=comma] {plots/empirical_cdf_data.csv};
    \addlegendentry{True}

    \nextgroupplot[
    ]

    \addplot[color=c0, thick] table[x=x2_est,y=y2_est,col sep=comma] {plots/empirical_cdf_data.csv};
    
    \addplot[color=c4, thick, dashed] table[x=x2_est,y=y2_true,col sep=comma] {plots/empirical_cdf_data.csv};

    \nextgroupplot[
    ]
    
    \addplot[color=c0, thick] table[x=x3_est,y=y3_est,col sep=comma] {plots/empirical_cdf_data.csv};
    
    \addplot[color=c4, thick, dashed] table[x=x3_est,y=y3_true,col sep=comma] {plots/empirical_cdf_data.csv};
\end{groupplot}
\end{tikzpicture}
    \caption{Empirical marginal CDFs for the output of a 1D Gaussian channel with 32-ary input constructed using the Sinkhorn-Knopp approximation (Approx.) and the ground-truth CDF (True).}
    \label{fig:empirical_cdf}
\end{figure}

As a more challenging test, we also consider the quality of the conditional distributions, namely $P_{Y \mid X}(\cdot \mid x)$ for each $x \in [N]$, in the specific context of the VQ-VAE experiment where we apply the Sinkhorn-Knopp approximation.
As set out in \cref{sec:img_compression_extra}, the experiment uses $N=256$ and the target distribution for each $x$ is an isotropic Gaussian with $D=64$ dimensions.
We run each mode through our variable-rate model for 5 different rate parameters $\lambda$ to generate $10^5$ samples each from the learned posterior and record the closeness to the target conditional distribution
\begin{IEEEeqnarray}{c}
    P_{Y \mid X}(\cdot \mid x) = \mc{N}(\mu_x, \operatorname{diag}(\sigma^2_x))
\end{IEEEeqnarray}
marginally for each of the 64 dimensions over 10 runs using 20 Sinkhorn-Knopp iterations.
These results are recorded in \cref{tab:conditional_ks}, where we show the mean, 5th and 95th-percentile, minimum and maximum KS statistics across all modes and all dimensions.
While the approximation error, particularly in the worst case, is higher than that observed for the marginal KS statistics in \cref{tab:marginal_ks}, the downstream rate-distortion results reported in \cref{sec:experiment_vqvae} suggest that the stochastic decoder can readily adapt to discrepancies which may arise in some modes or dimensions.

\begin{table}[bt]
    \centering
    \newcolumntype{T}{S[
        table-format=1.2,
        table-text-alignment=right
    ]}
    \newcolumntype{Q}{S[
        table-format=2.1,
        table-text-alignment=right
    ]}
    \begin{tabular}{@{}lTTTTrTTTTT@{}}
        \toprule
        & \multicolumn{5}{c}{Approx.} & \multicolumn{5}{c}{True} \\
        $\lambda$ & {Mean} & {P05} & {P95} & {Min} & {Max} & {Mean} & {P05} & {P95} & {Min} & {Max} \\
        \midrule
        \num{1.0e-1} & 2.74 & 2.27 & 3.23 & 1.87 & $35.9$ & 2.66 & 2.26 & 3.08 & 1.86 & 3.76 \\
        \num{2.7e-2} & 2.95 & 2.29 & 3.96 & 1.86 & $58.6$ & 2.66 & 2.27 & 3.10 & 1.76 & 3.72 \\
        \num{7.1e-3} & 3.71 & 2.32 & 7.94 & 1.70 & $50.4$ & 2.66 & 2.26 & 3.10 & 1.78 & 3.87 \\
        \num{1.9e-3} & 3.21 & 2.31 & 4.62 & 1.82 & $35.2$ & 2.66 & 2.27 & 3.09 & 1.75 & 3.73 \\
        \num{5.0e-4} & 2.66 & 2.27 & 3.09 & 1.79 & $3.58$ & 2.66 & 2.27 & 3.09 & 1.72 & 3.81 \\
        \bottomrule
    \end{tabular}
    \caption{KS test statistics for the conditional output distributions produced from our trained stochastic VQ-VAE using the Sinkhorn-Knopp approximation (Approx.) and ground-truth samples (True). All values have been scaled up by  $10^3$.}
    \label{tab:conditional_ks}
\end{table}

\subsection{Variable-Rate Image Compression}\label[appendix]{sec:img_compression_extra}
In this section, we provide additional details concerning the implementation of the experiments in \cref{sec:experiment_vqvae}, dealing in turn with our proposed stochastic VQ-VAE, then the standard deterministic VQ-VAE baseline, and finally the IS scheme, thus covering all the methods presented in \cref{fig:rd_plot} in the main paper.
First, we consider some elements of the setting common to all approaches.
We evaluate all the models on the CIFAR-10 dataset upsampled to $64 \times 64$ using bilinear interpolation, and adopt the standard split of \num{50000} train and \num{10000} test samples.
The batch size during training is 64; for testing it is 32, as this corresponds to 8192 latent variables to be transmitted per batch using either our polar-MLC scheme or IS.

Our results report both the mean rate-distortion values as well as the 5th and 95th-percentile distortions for each rate operating point, where the distortion is the average MSE across a batch.
The deterministic models are trained on a single L40S GPU with \SI{48}{\giga\byte} of VRAM, while the variable-rate stochastic models are trained on two such GPUs.
Evaluation is carried out on one GPU.
In \cref{fig:extra_example_imgs}, we provide some additional example output images from our proposed variable-rate stochastic VQ-VAE and the deterministic baselines in low, medium and high-rate regimes.

\paragraph{Stochastic VQ-VAE with Polar-MLC.}
The stochastic VQ-VAE uses the network architecture proposed in \citet{vandenoord2017}, based on the publicly available reference implementation.\footnote{\url{github.com/deepmind/sonnet/blob/v2/examples/vqvae_example.ipynb}}
This architecture is also used for the deterministic VQ-VAE and IS baselines.
In particular, the model consists of an encoder, a vector quantizer and a decoder, all of which are trained jointly.
The encoder contains 2 convolutional downsampling layers with stride 2 and kernel size $4 \times 4$, with 2 residual blocks following after.
The downsampling layers all have 256 hidden units, while the residual layers contain 64.
Each residual block contains a $3 \times 3$ convolution followed by a $1 \times 1$ convolution, and ReLU activations are adopted throughout.
The decoder mirrors the encoder, with the downsampling layers replaced by the corresponding upsampling operations.

Our codebook for the stochastic models has size $N = 256$ and the dimension of each codeword is $D = 64$. 
The codebook means are trained using exponential moving average (EMA) updates with a decay of 0.99 where codewords are selected at training and inference time based on the minimum-distance rule in \eqref{eqn:vq_vae_assignment}; gradients flow back to the encoder through a straight-through estimator (STE).
Codebook variances are trained using standard gradient descent.

To achieve rate-conditional training, we condition the decoder network and the learned codebook variances on a scalar input $\lambda$ using the strategy from \citet{dosovitskiy2020}.
Each convolutional and residual layer is followed by a feature-wise linear modulation (FiLM, \citealp{perez2018}) block, implemented as a multi-layer perceptron (MLP) with one hidden layer, a hidden dimension of 256, and SiLU activation.
If a layer outputs a feature map $F$ with dimensions $W \times H \times C$, each FiLM block learns a modulation $(a(\lambda), b(\lambda))$ where both $a$ and $b$ are vectors of dimension $C$.
Then, the feature map is transformed element-wise as
\begin{IEEEeqnarray}{c}
    \hat{F}_{ijk} = a_k(\lambda) F_{ijk} + b_k(\lambda).
\end{IEEEeqnarray}
During training, $\lambda$ is sampled log-uniformly and independently for each sample within a batch between $10^{-5}$ and $10^{-1}$.
Models are initialized from a converged deterministic checkpoint.
Following this, all parameters apart from the EMA-updated weights are trained using the Adam optimizer with $\beta_1 = 0.9$, $\beta_2 = 0.999$ and a learning rate of $10^{-4}$ for 200 epochs, followed by a cosine decay down to $10^{-6}$ over a further 100 epochs.
The decoder is trained on both ground truth samples and approximate samples generated from the Sinkhorn-Knopp approximation introduced in \cref{sec:permuted_scheme_impl}.

To obtain the rate measurements in \cref{fig:rd_plot} under joint coding of the 8192 latents using the polar-MLC scheme from \cref{sec:polar}, we report the length of the bitstream returned by arithmetic coding \citep{rissanen1976} divided by the number of latents $n$.
The Sinkhorn-Knopp approximation is used to estimate the discrete posterior $P_{K \mid X, \bar{U}^N}(\cdot \mid x, u^N)$, and the prior on the codebook is treated as uniform during training but learned across the training set for inference.
We denote this learned prior as $\hat{P}_X$ in what follows.
Lastly, following the exposition of the modified PolarSim algorithm in \cref{sec:polar_sim}, the arithmetic coder requires an estimate of the probability $\Pr[\Delta_i = 1]$ for each $i \in [n]$ and, in the MLC case for a 256-ary alphabet, for each level $l \in [8]$.
These are computed using 2000 Monte Carlo iterations with the inputs being synthetic discrete latents sampled from the learned prior.
During inference, the latents are shuffled using additional shared randomness before being fed to the channel simulator to better satisfy the iid input assumption of the polar-MLC scheme.
The original order is recovered before passing the results to the decoder network.
We evaluate on 3 random seeds.

\paragraph{Deterministic VQ-VAE Baseline.}
The deterministic VQ-VAEs use the same neural architecture, except they do not use rate-conditional training and each codebook entry is simply a fixed vector $c_k \in \mathbb{R}^D$.
To achieve different rates, we separately train 7 independent VQ-VAEs with the codebook sizes being $N \in \{ 4, 8, 16, 32, 64, 128, 256 \}$.
Again, we use the Adam optimizer, this time with a base learning rate of $2 \times 10^{-4}$ over 300 epochs, before applying a cosine decay to $10^{-6}$ for an additional 100 epochs.
To estimate the communication cost for the deterministic models, we assume an ideal entropy code given access to the learned prior on the codebook obtained from the training set, giving a lower bound of $\ell(M) \geq \sum_{i=1}^{n} - \log \hat{P}_X(x_i)$ bits for a realization $x^n$ of the full latent vector.

\paragraph{Importance Sampling Baseline.}
Our implementation follows the communication-efficient IS scheme discussed in \cite{phan2024}, which is a variant of ordered random coding \citep{theis2022b}, and is evaluated on the same trained $N=256$ rate-conditional checkpoints used to test our polar-MLC scheme.
Thus, we treat IS as an alternative approximate channel simulator, which does not require modification to any other aspects of the stochastic VQ-VAE concept, network architecture, or training.

We refer the reader to \citet{phan2024} for full details regarding the IS algorithm and outline the specifics relevant to our experiment here.
For a single sample, our proposal distribution is the learned marginal with
\begin{IEEEeqnarray}{c}
    \hat{p}_Y(y) = \sum_{x=1}^{N} \hat{P}_X(x) \mc{N}(y; \mu_x, \operatorname{diag}(\sigma^2_x))
\end{IEEEeqnarray}
and the target distribution is the specific mode, $P_{Y \mid X}(\cdot \mid x) = \mc{N}(\mu_x, \operatorname{diag}(\sigma^2_x))$.
To keep the computational complexity of IS manageable, the 8192 latents are split into chunks of size $B$ and each chunk is accorded $2^{16}$ iid samples from the marginal.
This sample count allows one operating point to be evaluated in approximately 1.5 hours; our polar-MLC method can do the same in 2--3 minutes. 
In \cref{fig:rd_plot}, we evaluate chunk sizes $B \in \{ 1, 2, 4 \}$.
While larger chunk sizes reduce the bitrate via amortization, the number of samples becomes insufficient to faithfully reproduce the output distribution at higher rates if $B$ is too large, leading to increased distortion.
We report the bitrate using the cross-entropy with the specific Zipf distribution described in \citet{li2024a}, which provides a slightly tighter bound than that used in \citet{phan2024}.
If an entropy code is constructed using this distribution, a lower bound on the bitrate is then $\ell(M) \geq \sum_{i=1}^{n/B} -\log P_{\mathrm{Zipf(\alpha)}}(k_i)$ where $k_i$ is the sample index returned by importance sampling for block $i$, and $P_{\mathrm{Zipf(\alpha)}}$ is the PMF of a Zipf distribution with parameter $\alpha = 1 + 1 / (I(X;Y) + 1)$.
We evaluate on 3 random seeds.

\begin{figure}[p]
    \centering
    \begin{subfigure}[b]{0.49\linewidth}
        \centering
        \includegraphics[width=\linewidth]{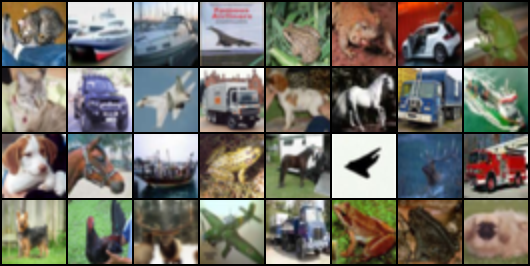}
        \caption{Ground truth}
    \end{subfigure}
    
    \begin{subfigure}[b]{0.49\linewidth}
        \centering
        \includegraphics[width=\linewidth]{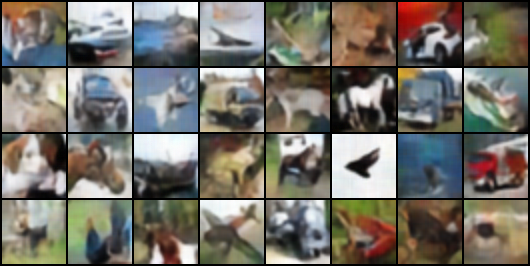}
        \caption{Polar-MLC, $\lambda=\num{7.2e-2}$, $R=1.9$}
    \end{subfigure}
    \hfill
    \begin{subfigure}[b]{0.49\linewidth}
        \centering
        \includegraphics[width=\linewidth]{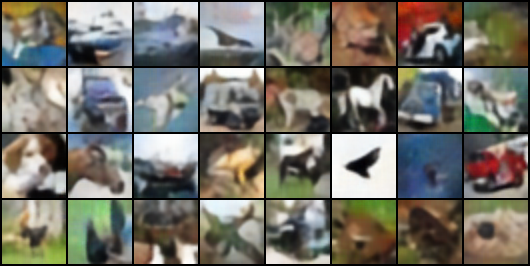}
        \caption{VQ-VAE, $N=4$, $R=1.9$}
    \end{subfigure}

    \begin{subfigure}[b]{0.49\linewidth}
        \centering
        \includegraphics[width=\linewidth]{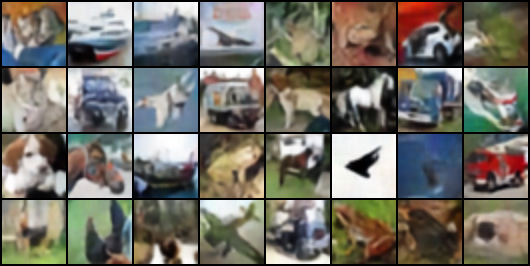}
        \caption{Polar-MLC, $\lambda=\num{2.5e-2}$, $R=2.9$}
    \end{subfigure}
    \hfill
    \begin{subfigure}[b]{0.49\linewidth}
        \centering
        \includegraphics[width=\linewidth]{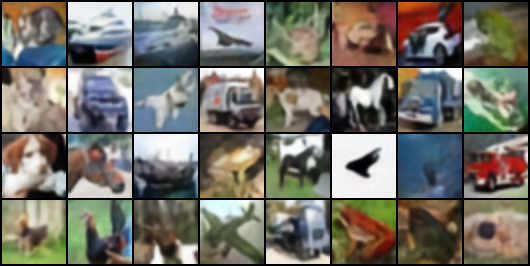}
        \caption{VQ-VAE, $N = 8$, $R=2.9$}
    \end{subfigure}

    \begin{subfigure}[b]{0.49\linewidth}
        \centering
        \includegraphics[width=\linewidth]{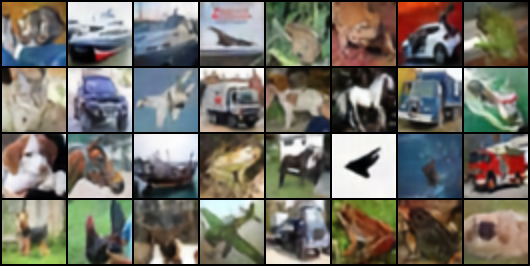}
        \caption{Polar-MLC, $\lambda = \num{6.0e-3}$, $R=5.0$}
    \end{subfigure}
    \hfill
    \begin{subfigure}[b]{0.49\linewidth}
        \centering
        \includegraphics[width=\linewidth]{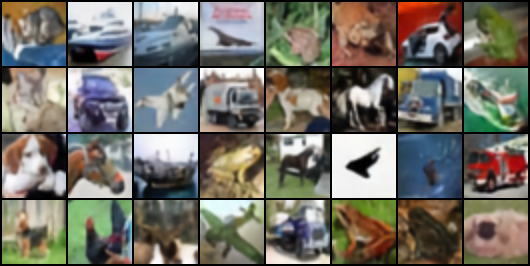}
        \caption{VQ-VAE, $N = 32$, $R=4.9$}
    \end{subfigure}

    \begin{subfigure}[b]{0.49\linewidth}
        \centering
        \includegraphics[width=\linewidth]{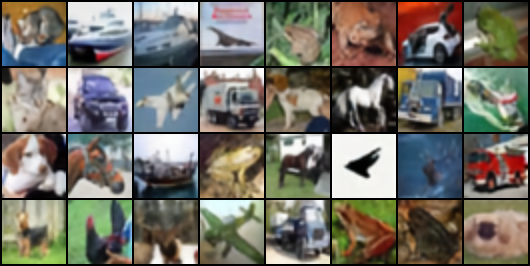}
        \caption{Polar-MLC, $\lambda = \num{5.0e-4}$, $R=7.8$}
    \end{subfigure}
    \hfill
    \begin{subfigure}[b]{0.49\linewidth}
        \centering
        \includegraphics[width=\linewidth]{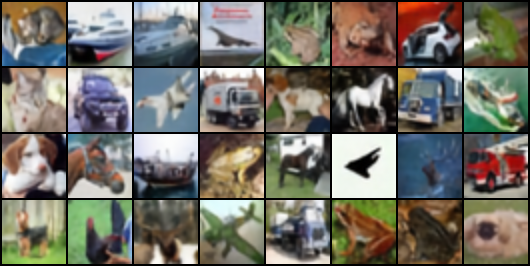}
        \caption{VQ-VAE, $N = 256$, $R=7.8$}
    \end{subfigure}
    \caption{Additional example outputs from the stochastic VQ-VAE experiment using the polar-MLC channel simulation scheme, along with reconstructions from deterministic VQ-VAE baselines achieving similar coding rates. The rate $R$ is in bits/latent.}
    \label{fig:extra_example_imgs}
\end{figure}

\subsection{Distributed Mean Estimation}\label[appendix]{sec:dme_extra}
Here, we provide some further information on the DME experiment in \cref{sec:experiment_cdp}, including privacy calibration, simulation with the polar-MLC scheme, the PFR implementation, and finally a comparison to dithered quantization.

\paragraph{Privacy Details and Calibration.}
First, we go into some further detail on the CDP setting and how privacy budgets are accounted for in order to determine the amount of Gaussian noise to be added.
Once the appropriate noise standard deviation has been determined by the privacy accountant, any exact channel simulator may be used to simulate such noise and faithfully reproduce the exact Gaussian mechanism with the associated privacy-utility guarantees and tradeoffs.

We start by giving some more specifics of our DME setup.
Recall that the vector $X_i$ held locally by client $i \in [C]$ has dimension $D=512$, and each coordinate $X_{ij}$, $j \in [D]$ is uniformly distributed on a set of $N=16$ points spaced evenly on $[-1,1]$.
More specifically, this input constellation is $\mc{X} = \{ 2 (m-1) / 15 - 1 : m \in [N] \}$.
We therefore have that the client vectors live in the $\ell_2$ ball $\mc{B}_D(L) = \{ v \in \mathbb{R}^D : \norm{v}_2 \leq L \}$ with $L = \sqrt{D}$.

Given the desired privacy $\varepsilon$, we apply the Gaussian mechanism by having each client add isotropic Gaussian noise $Z_i \sim \mc{N}(0, \sigma_\varepsilon^2 I_D)$ to its local vector.
Then, the noisy sample $Y_i = X_i + Z_i$ can be communicated to the server using a channel simulator.
The server in turn calculates and later releases the estimate $\hat{\mu} = \sum_{i=1}^{C} Y_i / C = \mu + Z$, where $Z = \sum_{i=1}^{C} Z_i / C$ is the aggregate added noise.

The standard definition of $(\varepsilon,\delta)$-CDP with respect to the released mean, see for example \citet{chen2023}, is provided below.

\begin{definition}[$(\varepsilon,\delta)$-CDP for DME]\label{def:cdp}
The mean $\hat{\mu}$ satisfies $(\varepsilon,\delta)$-central DP if for any neighboring datasets $x^C = (x_1, \ldots, x_i, \ldots, x_C)$ and $\tilde{x}^C = (x_1, \ldots, \tilde{x}_i, \ldots, x_C)$ and measurable $S$,
\begin{IEEEeqnarray}{c}
    \Pr[\hat{\mu} \in S \mid X^C = x^C] \leq e^{\varepsilon} \Pr[\hat{\mu} \in S \mid X^C = \tilde{x}^C] + \delta. \label{eqn:cdp_def}
\end{IEEEeqnarray}
\end{definition}

Provided $\sigma_\varepsilon$ is chosen appropriately given $\varepsilon$ and $\delta$, the Gaussian mechanism is known to satisfy \eqref{eqn:cdp_def} \citep{dwork2014}.
To numerically obtain the required $\sigma_\varepsilon$, following \citet{liu2024} we use a R\'{e}nyi DP accountant.
Our implementation uses Google's differential privacy library,\footnote{\url{github.com/google/differential-privacy/}} which provides a \verb|GaussianDpEvent| that can be calibrated against a desired $(\varepsilon,\delta)$ pair to return the ratio $\nu$ between the standard deviation and the $\ell_2$ sensitivity of the target function.
With reference to the notion of neighboring datasets in \cref{def:cdp}, the $\ell_2$ sensitivity of the mean is
\begin{IEEEeqnarray}{c}
    \Delta_2 = \sup_{x^C,\tilde{x}^C} \norm{\mu(x^C) - \mu(\tilde{x}^C)}_2 = \frac{1}{C} \sup_{x_i,\tilde{x}_i} \norm{x_i - \tilde{x}_i}_2 = \frac{2 \sqrt{D}}{C}.
\end{IEEEeqnarray}
Therefore, upon finding $\nu$ from the privacy accountant, the aggregate noise should follow $Z \sim \mc{N}(0, (4 \nu^2 D / C^2) I_D)$ and the noise added at each client should have $\sigma_\varepsilon = 2 \nu \sqrt{D/C}$.

\paragraph{Polar-MLC Coding Scheme for DME.}
Once we have obtained the required $\sigma_\varepsilon$ that each client should use, communicating the noisy data $Y_i$ to the server can be framed as a multi-shot channel simulation task with $P_{Y \mid X}(\cdot \mid x) = \mc{N}(x, \sigma^2_\varepsilon)$ and $D$ independent uses of this channel.
The alphabet size is $N=16$, and the exact polar-MLC scheme from \cref{sec:polar} using matrix permanents can be used to this end.
Because polar coding benefits from longer blocklengths, we jointly encode $B=8$ instances of the DME problem at a time by concatenating the noisy vectors to be transmitted, giving a coding blocklength of $n=4096$.
We emphasize that this does not affect the privacy or utility properties of the Gaussian mechanism under consideration, the only disadvantage of batched communication being greater latency at the server; privacy guarantees remain on the individual length-$D$ vectors rather than being jointly composed across a batch.

Concerning the implementation, as in \cref{sec:img_compression_extra}, probability tables for arithmetic coding are obtained from 2000 Monte Carlo iterations.
To align with the PFR baseline to be discussed next, we report a tight upper bound on the arithmetic coding cost, $\ell(M) \leq \sum_{i=1}^{n} -\log \hat{P}_{\Delta_i}(\Delta_i) + 2$, where $\Delta^n$ is the binary correction vector that appears in the modified PolarSim, \cref{alg:polar_sim}, and $\hat{P}_{\Delta_i}(\Delta_i)$ comes from the estimated probability table.
The number of bits per sample in \cref{fig:cdp} then refers directly to this upper bound applied to the normalized coding cost $\ell(M) / n$ and added across all 4 levels of the multi-level code.
We test 50 values of $\varepsilon$ and run 200 trials for each.

\paragraph{PFR Baseline.}
For PFR, we implement the basic scheme from \cite{li2018}.
We refer the reader to Algorithm~2 in \citet{li2024a} for detailed pseudocode outlining the method.
The proposal distribution is the marginal, which in our case is a standard mixture of Gaussians with
\begin{IEEEeqnarray}{c}
    p_Y(y) = \frac{1}{N} \sum_{x \in \mc{X}} \mc{N}(y; x, \sigma_\varepsilon^2) 
\end{IEEEeqnarray}
and the target distribution is the specific conditional Gaussian, $P_{Y \mid X}(\cdot \mid x) = \mc{N}(x, \sigma_\varepsilon^2)$.
To terminate, PFR requires an upper bound on the density ratio for each possible $x$, $g_x^\ast \geq \sup_y p_{Y \mid X}(y \mid x) / p_Y(y)$.
Fortunately, such a bound exists for our Gaussian mixture.
To enhance the speed of PFR, which becomes slower if a loose stopping bound is used, we precompute the tightest bound numerically for each $x \in \mc{X}$ using an optimization routine before entering the sampling loop.

While we have above described the algorithm in scalar terms for simplicity, to achieve a competitive coding cost, we apply PFR jointly on chunks of $B$ samples at a time.
As the computational complexity and expected number of samples to be generated increases exponentially in $B$ and also in $I(X;Y)$ \citep{li2024a}, $B$ cannot be made too large.
We therefore sweep $B$ in increments of 2 and choose the largest feasible value of $B$ for each $\varepsilon$, reported in \cref{tab:block_size_pfr}.
Our criteria is that for each $\varepsilon$, a block of 4096 samples can be communicated in less than \num{10} -- \SI{20}{\second} on an L40S GPU; with reference to the table in \cref{fig:cdp}, we emphasize that our polar-MLC algorithm can process 128 such blocks in approximately \SI{3}{\second} on a 32-core CPU.
We evaluate 25 values of $\varepsilon$ and run 200 trials for each.
Following Algorithm~2 in \cite{li2024a}, we assume the integer indices returned by PFR are to be encoded using an entropy code designed for a specific Zipf distribution with parameter $\alpha = 1 + 1 / (I(X;Y) + 1)$, and report the lower bound $\ell(M) \geq \sum_{i=1}^{n/B} -\log P_{\mathrm{Zipf(\alpha)}}(k_i)$ where $k_i$ is the sample index returned by PFR for block $i$ and $P_{\mathrm{Zipf}(\alpha)}$ is the PMF of the relevant Zipf distribution.

\begin{table}[tb]
    \centering
    \begin{tabular}{@{}lrrrrr@{}}
        \toprule
        Index & 0 & 1 & 2--3 & 4--10 & 11--24 \\
        \midrule
        $B$ & 12 & 10 & 8 & 6 & 4 \\
        \bottomrule
    \end{tabular}
    \caption{Chunk sizes used for the PFR experiment. The test index refers to evenly spaced values of $\varepsilon$ between 0.05 and 30. As $\varepsilon$ increases, so does the mutual information and the computational cost at a given chunked blocklength.}
    \label{tab:block_size_pfr}
\end{table}

\paragraph{Performance Benchmarks.}
We briefly mention the benchmarking setup used to produce the table in \cref{fig:cdp}.
For the polar-MLC scheme, 128 blocks of 4096 samples are transmitted, using one 32-core Intel Xeon Gold 6338 CPU running a C++ backend.
We run the benchmark 20 times and report the mean and standard deviation of the wall-clock time.
PFR runs on a single L40S GPU, and due to its higher computational cost, we run the benchmark 5 times on 8 blocks, then scale up to estimate the time for 128 blocks.
All tests, including the benchmark, use 3 random seeds.

\paragraph{Comparison with Dithered Quantization.}
Subtractively dithered quantization \citep{roberts1962,gray1993} is a widely-known technique which can exactly simulate additive noise channels, and as such has also been used to achieve DP with additive noise mechanisms, including the Gaussian mechanism which is of interest here \citep{hasircioglu2024,hegazy2024}.
However, as it can only realize additive noise channels, it is not considered a general-purpose channel simulation scheme in the same sense as PFR, GRS or our polar-MLC construction.

Nevertheless, for completeness we provide in this section a comparison of the polar-MLC scheme with the algorithm from \citet{hasircioglu2024}, which is specifically designed to solve the CDP DME problem using the Gaussian mechanism.
Given a scalar $X = x$, to transmit a sample $Y \sim \mc{N}(x, \sigma_\varepsilon^2)$, the scheme performs the following:
\begin{enumerate}
    \item Generate a sample $V \sim \Gamma(3/2, 1/2)$, where $\Gamma(3/2, 1/2)$ is the gamma distribution with shape and rate parameters $3/2$ and $1/2$ respectively.
    \item Given $V = v$, compute the quantization step size as $\Delta = 2 \sigma_\varepsilon \sqrt{v}$ and sample $U \sim \operatorname{Unif}[-\Delta / 2, \Delta / 2]$.
    \item Apply the quantization function $Q(w) =  \lfloor (x - \Delta / 2) / \Delta \rceil \Delta + \Delta / 2$ to $x + U$.
    \item Transmit the result using a fixed-length code consuming $b = \lceil \log_2(2 \lfloor 1 / \Delta + 1 \rceil) \rceil$ bits.
\end{enumerate}
The above description is for a scalar channel.
For a higher-dimensional channel such as ours, where 4096 samples must be transmitted, the scalar procedure is simply repeated and the resulting codewords concatenated.
In \cref{fig:dq_cdp}, we show the results obtained using this scheme and the fixed-length code described above in the same setting as that of the PFR to polar-MLC comparison in \cref{fig:cdp}.
We note that, unlike the polar-MLC or PFR schemes, the fixed-length construction does not approach zero bits/sample as $\varepsilon \to 0$, as it incurs at least 1 bit of communication for each coordinate.

As the scheme in \citet{hasircioglu2024} operates coordinate-wise, it may perform more poorly when the dimension of each sample is increased.
To illustrate this, we also consider a setting where the alphabet size is still $N=16$, but the constellation $\mc{X}$ is now the set of vertices of a 4D hypercube with side length $1$.
This maintains the same $\ell_2$ sensitivity of the mean despite the change in geometry, and the experimental results are shown in \cref{fig:dq_cdp} for $\varepsilon \in [0.05, 3]$.
As each dimension of the constellation must be simulated individually using the scalar protocol, the dithered quantization scheme incurs additional overhead while our more general-purpose polar-MLC scheme, which is agnostic to the geometry of the constellation, readily generalizes.
Again, we use 3 random seeds.

As an aside, we note that while variants on dithered quantizers directly simulating Gaussian noise in higher dimensions have recently been considered \citep{kobus2024}, they currently do not achieve exact simulation and therefore cannot guarantee privacy via the usual properties of the Gaussian mechanism.
Achieving exact simulation currently requires, for example, more complex schemes including an inner rejection sampling loop \citep{ling2025}, which incurs greater computational complexity and is beyond the scope of our experiment.
For scalar approaches, variable-length codes may also be used to encode the quantized index instead of a fixed-length code.
For example, \citet{hegazy2024} and \citet{shahmiri2024} adopt Elias gamma coding, which we find underperforms the fixed-length code in our problem setting.
Entropy codes conditioned on the shared randomness can theoretically obtain even better rate performance but generally require constructing a new code for each realization, which may be infeasible in practice, as noted in \citet{hegazy2024}.
This is outside the scope of the present comparison; we leave the consideration of such coding schemes and any potential simplifications in the discrete-to-continuous setting to future work.  

\begin{figure}
    \centering
    \begin{tikzpicture}
\begin{groupplot}[
    width=7.6cm, 
    height=6cm,
    tick label style={font=\footnotesize},
    xlabel={\footnotesize Privacy $\varepsilon$},
    ylabel={\footnotesize Rate (bits/sample)},
    clip=false,
    grid=major,
    max space between ticks=25,
    axis lines=left,
    axis line style={-},
    enlarge x limits=false,
    group style={group name=my plots, group size=2 by 1, horizontal sep=0.7cm, ylabels at=edge left},
]
\nextgroupplot[
    legend pos=south east,
    legend cell align={left},
    legend style={font=\scriptsize, inner sep=1.5pt, row sep=0pt, column sep=3pt, legend image post style={scale=0.8}, cells={anchor=west}, reverse legend},
    xmax=30
]

\addplot[color=c3, thick] table[x=eps,y=I_X_Y_U,col sep=comma] {plots/theoretical_results_scalar.csv};
\addlegendentry{$I(X;Y \mid \bar{U}^{16})$}

\addplot[color=c2, thick] table[x=eps,y=I_X_Y,col sep=comma] {plots/theoretical_results_scalar.csv};
\addlegendentry{$I(X;Y)$}

\addplot[color=c4, thick] table[x=eps,y=dat_mean,col sep=comma] {plots/dq_results_scalar_fixed.csv};
\addlegendentry{Has{\i}rc{\i}o{\u{g}}lu and G{\"u}nd{\"u}z (2024)}

\addplot[draw=none, name path=lower, forget plot] table[x=eps,y=dat_p05,col sep=comma] {plots/dq_results_scalar_fixed.csv};

\addplot[draw=none, name path=upper, forget plot] table[x=eps,y=dat_p95,col sep=comma] {plots/dq_results_scalar_fixed.csv};

\addplot[fill=c4!30, opacity=0.7, forget plot] fill between[of=lower and upper];

\addplot[color=c0, thick] table[x=eps,y=dat_mean,col sep=comma] {plots/polar_results_scalar.csv};
\addlegendentry{Polar-MLC}

\addplot[draw=none, name path=lower, forget plot] table[x=eps,y=dat_p05,col sep=comma] {plots/polar_results_scalar.csv};

\addplot[draw=none, name path=upper, forget plot] table[x=eps,y=dat_p95,col sep=comma] {plots/polar_results_scalar.csv};

\addplot[fill=c0!30, opacity=0.7, forget plot] fill between[of=lower and upper];

\nextgroupplot[
    xmax=3.0,
    ytick={1,2,3,4,5,6,7}
]

\addplot[color=c3, thick] table[x=eps,y=I_X_Y_U,col sep=comma] {plots/theoretical_results_vector.csv};

\addplot[color=c2, thick] table[x=eps,y=I_X_Y,col sep=comma] {plots/theoretical_results_vector.csv};

\addplot[color=c4, thick] table[x=eps,y=dat_mean,col sep=comma] {plots/dq_results_vector_fixed.csv};

\addplot[draw=none, name path=lower, forget plot] table[x=eps,y=dat_p05,col sep=comma] {plots/dq_results_vector_fixed.csv};

\addplot[draw=none, name path=upper, forget plot] table[x=eps,y=dat_p95,col sep=comma] {plots/dq_results_vector_fixed.csv};

\addplot[fill=c4!30, opacity=0.7, forget plot] fill between[of=lower and upper];

\addplot[color=c0, thick] table[x=eps,y=dat_mean,col sep=comma] {plots/polar_results_vector.csv};

\addplot[draw=none, name path=lower, forget plot] table[x=eps,y=dat_p05,col sep=comma] {plots/polar_results_vector.csv};

\addplot[draw=none, name path=upper, forget plot] table[x=eps,y=dat_p95,col sep=comma] {plots/polar_results_vector.csv};

\addplot[fill=c0!30, opacity=0.7, forget plot] fill between[of=lower and upper];

\end{groupplot}
\end{tikzpicture}
    \caption{Comparison with dithered quantization for distributed mean estimation under the Gaussian mechanism. \emph{Left:} Scalar 16-ary channel. \emph{Right:} 4D 16-ary channel. P05--P95 rates are shown.\vspace{-1em}}
    \label{fig:dq_cdp}
\end{figure}
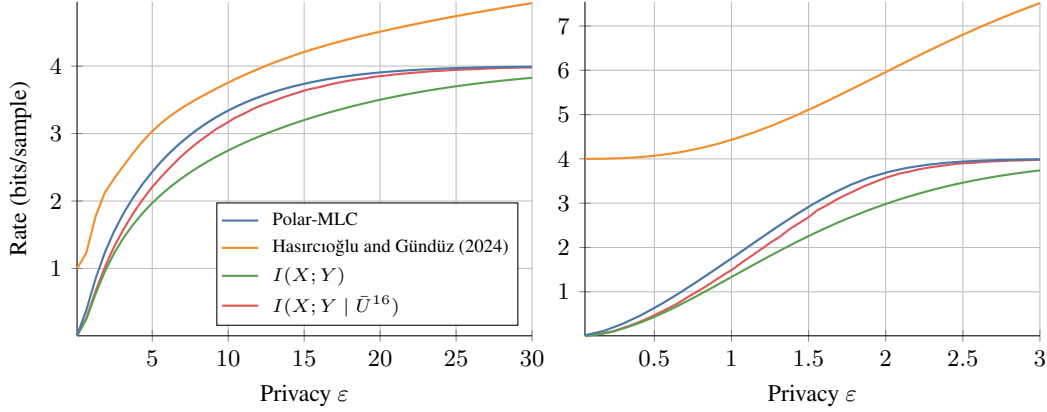

\end{document}